%% file: main.tex
\documentclass{article}

\usepackage[nonatbib,final]{neurips_2023}

\input{Macros.tex}

\title{Exact Generalization Guarantees for (Regularized) Wasserstein Distributionally Robust Models}

\author{%
  Waïss Azizian,~~ Franck Iutzeler,~~  Jérôme Malick \\
  Univ. Grenoble Alpes, CNRS, Grenoble INP \\
  Grenoble, 38000, France \\
  \texttt{firstname.lastname@univ-grenoble-alpes.fr} 
}

\begin{document}

\maketitle

\begin{abstract}
Wasserstein distributionally robust estimators have emerged as powerful models for prediction and decision-making under uncertainty. These estimators provide attractive generalization guarantees: the robust objective obtained from the training distribution is an exact upper bound on the true risk with high probability. However, existing guarantees either suffer from the curse of dimensionality, are restricted to specific settings, or lead to spurious error terms. 
In this paper, we show that these generalization guarantees actually hold on general classes of models, do not suffer from the curse of dimensionality, and can even cover distribution shifts at testing. We also prove that these results carry over to the newly-introduced regularized versions of Wasserstein distributionally robust problems.%
\end{abstract}

\input{body.tex}

\begin{ack}
This work has been supported by MIAI Grenoble Alpes (ANR-19-P3IA-0003).
\end{ack}

\bibliography{references}
\bibliographystyle{abbrvnat}

\newpage
\renewcommand{\appendixpagename}{Appendices}%
\begin{appendices}

\input{intro_app.tex}

\input{appendix.tex}

\end{appendices}

\end{document}

%% file: Macros.tex
\usepackage{amsmath}		%
\usepackage{amssymb}		%
\usepackage{amsfonts}		%
\usepackage{amsthm}		%

\usepackage{mathtools}		%

\mathtoolsset{%
}

\usepackage[utf8]{inputenc}		%
\usepackage[T1]{fontenc}		%

\usepackage{dsfont}		%

\usepackage{acronym}		%

\usepackage[labelfont={bf,small},labelsep=colon,font=small]{caption}	%
\usepackage[dvipsnames,svgnames]{xcolor}		%
\colorlet{MyRed}{Crimson!75!black}
\colorlet{MyGreen}{DarkGreen!80!black}
\colorlet{MyBlue}{MediumBlue!80!black}

\usepackage{titletoc}
\usepackage[page]{appendix}
\titlecontents{section}[0.5em]{\smallskip\bfseries}%
{\thecontentslabel\enspace}%
{}%
{\titlerule*[1.2pc]{\enspace}\contentspage}%

\usepackage{cancel}		%
\usepackage{latexsym}		%

\usepackage{pifont}		%

\usepackage{subcaption}		%
\usepackage{tikz}		%
\usetikzlibrary{calc,patterns}		%
\usepackage{pgfplots}
\usepgfplotslibrary{groupplots,dateplot}
\usetikzlibrary{patterns,shapes.arrows}
\pgfplotsset{compat=newest}

\usepackage{array}		%
\usepackage{booktabs}		%
\usepackage[inline,shortlabels]{enumitem}		%
\setenumerate{itemsep=\smallskipamount,topsep=\medskipamount}		%

\usepackage{xspace}		%

\usepackage[round]{natbib}		%

\usepackage{hyperref}
\hypersetup{
colorlinks=true,
linktocpage=true,
pdfstartview=FitH,
breaklinks=true,
pdfpagemode=UseNone,
pageanchor=true,
pdfpagemode=UseOutlines,
plainpages=false,
bookmarksnumbered,
bookmarksopen=false,
bookmarksopenlevel=1,
hypertexnames=true,
pdfhighlight=/O,
urlcolor=MyBlue,linkcolor=MyBlue,citecolor=MyBlue,	%
pdftitle={},
pdfauthor={},
pdfsubject={},
pdfkeywords={},
pdfcreator={pdfLaTeX},
pdfproducer={LaTeX with hyperref}
}

\usepackage{algorithm}		%
\usepackage{algpseudocode}		%

\usepackage{thmtools}		%
\usepackage{thm-restate}		%

\makeatletter
\def\cleartheorem#1{\expandafter\let\csname#1\endcsname\relax
    \expandafter\let\csname c@#1\endcsname\relax
}
\makeatother
\cleartheorem{example}
\cleartheorem{theorem}
\cleartheorem{lemma}
\cleartheorem{proposition}
\cleartheorem{corollary}
\cleartheorem{remark}

\newtheorem{theorem}{Theorem}[section]
\newtheorem{lemma}[theorem]{Lemma} 
\newtheorem{proposition}[theorem]{Proposition} 
\newtheorem{corollary}[theorem]{Corollary}
\newtheorem{definition}[theorem]{Definition}
\newtheorem{remark}[theorem]{Remark}
\newtheorem{example}[theorem]{Example}

\theoremstyle{definition}
\newtheorem{assumption}{Assumption}		%

\newcounter{proofpart}

\numberwithin{example}{section}		%

\usepackage[sort&compress,capitalize,nameinlink]{cleveref}		%
\crefname{algo}{Algorithm}{Algorithms}
\crefname{assumption}{Assumption}{Assumptions}
\crefname{lemma}{Lemma}{Lemmas}
\crefformat{equation}{(#2#1#3)}
\crefmultiformat{equation}{(#2#1#3)}%
{ and~(#2#1#3)}{, (#2#1#3)}{ and~(#2#1#3)}
\crefrangeformat{equation}{\upshape(#3#1#4)\textendash(#5#2#6)}

\usepackage{autonum}

\usepackage[suppress]{color-edits}		%
\usepackage[normalem]{ulem}		%

\newcommand{\debug}[1]{#1}		%

\newcommand{\newmacro}[2]{\newcommand{#1}{{\debug{#2}}}}		%
\newcommand{\newop}[2]{\DeclareMathOperator{#1}{\debug{#2}}}		%

\DeclarePairedDelimiter{\braces}{\{}{\}}		%
\DeclarePairedDelimiter{\bracks}{[}{]}		%
\DeclarePairedDelimiter{\parens}{(}{)}		%

\DeclarePairedDelimiter{\abs}{\lvert}{\rvert}		%
\DeclarePairedDelimiter{\pospart}{[}{]_{+}}		%

\DeclarePairedDelimiterX{\setdef}[2]{\{}{\}}{#1:#2}		%
\DeclarePairedDelimiterXPP{\exclude}[1]{\mathopen{}\setminus}{\{}{\}}{}{#1}

\newcommand{\R}{\mathbb{R}}		%

\DeclareMathOperator*{\argmax}{arg\,max}		%
\DeclareMathOperator*{\argmin}{arg\,min}		%
\DeclareMathOperator{\bd}{bd}		%
\DeclareMathOperator{\bigoh}{\mathcal{O}}		%
\DeclarePairedDelimiterXPP{\bigof}[1]{\mathcal{O}}{(}{)}{}{#1}		%
\DeclareMathOperator{\conv}{conv}		%
\DeclareMathOperator{\dist}{dist}		%
\DeclareMathOperator{\grad}{\nabla}		%
\DeclareMathOperator{\Rgrad}{grad}		%
\DeclareMathOperator{\RHess}{Hess}		%
\DeclareMathOperator{\Rexp}{exp}		%
\DeclareMathOperator{\one}{\mathds{1}}		%
\DeclareMathOperator{\relint}{ri}		%
\DeclareMathOperator{\supp}{supp}		%
\newcommand{\eg}{e.g.,\xspace}		%
\newcommand{\ie}{i.e.,\xspace}		%
\newcommand{\wrt}{w.r.t.\xspace}		%

\newcommand{\alt}[1]{#1'}		%
\newcommand{\altalt}[1]{#1''}		%

\newmacro{\dd}{\mathrm{d}}		%
\newcommand{\eps}{\varepsilon}		%
\newcommand{\sml}{\eta} %

\newmacro{\const}{c}		%
\newmacro{\Const}{C}		%
\usepackage{xparse}
\NewDocumentCommand{\coef}{O{\lambda}}{\debug{#1}}
\newmacro{\param}{\theta}		%
\newmacro{\params}{\Theta}		%

\newmacro{\beforestart}{0}		%
\newmacro{\start}{1}		%
\newmacro{\afterstart}{2}		%
\newmacro{\running}{\start,\afterstart,\dotsc}		%
\newmacro{\halfrunning}{1,3/2,2\dotsc}		%

\newmacro{\run}{t}		%
\newmacro{\runalt}{s}		%
\newmacro{\runaltalt}{\tau}		%
\newmacro{\nRuns}{T}		%
\newmacro{\runs}{\mathcal{\nRuns}}		%

\newmacro{\state}{X}		%
\newmacro{\statealt}{Y}		%
\newmacro{\statealtalt}{Z}		%

\newcommand{\init}[1][\state]{\debug{#1}_{\start}}		%
\newcommand{\afterinit}[1][\state]{\debug{#1}_{\afterstart}}		%

\newmacro{\tstart}{0}		%
\renewcommand{\time}{\debug{t}}		%
\newmacro{\timealt}{s}		%
\newmacro{\horizon}{T}		%

\newmacro{\traj}{x}		%
\newmacro{\trajalt}{y}		%
\newmacro{\trajaltalt}{z}		%

\newmacro{\flowmap}{\Theta}		%
\DeclarePairedDelimiterXPP{\flowof}[2]{\flowmap_{#1}}{(}{)}{}{#2}		%

\newop{\Nash}{NE}		%
\newop{\CE}{CE}		%
\newop{\CCE}{CCE}		%
\newop{\NI}{NI}		%

\newop{\brep}{br}		%
\newop{\preg}{\overline{Reg}}		%
\newop{\val}{val}		%

\newmacro{\play}{i}		%
\newmacro{\playalt}{j}		%
\newmacro{\playaltlalt}{k}		%
\newmacro{\nPlayers}{N}		%
\newmacro{\players}{\mathcal{\nPlayers}}		%

\newmacro{\pure}{\alpha}		%
\newmacro{\purealt}{\beta}		%
\newmacro{\purealtalt}{\gamma}		%
\newmacro{\nPures}{A}		%
\newmacro{\pures}{\mathcal{\nPures}}		%

\newmacro{\loss}{\ell}		%
\newmacro{\pay}{u}		%
\newmacro{\payv}{v}		%
\newmacro{\pot}{f}		%

\newmacro{\game}{\mathcal{G}}		%
\newmacro{\gamefull}{\game(\players,\points,\pay)}		%

\newmacro{\fingame}{\Gamma}		%
\newmacro{\fingamefull}{\Gamma(\players,\pures,\pay)}		%

\newmacro{\gmat}{g}		%
\newmacro{\gdist}{\dist_{\gmat}}
\newmacro{\mfld}{M}		%
\newmacro{\form}{\omega}		%
\newmacro{\curve}{\gamma}

\newmacro{\tvec}{z}		%
\newmacro{\uvec}{u}		%

\newmacro{\ball}{\overline{\mathbb{B}}}		%
\newmacro{\openball}{\mathbb{B}}		%
\newmacro{\sphere}{\mathbb{S}}		%

\newmacro{\normalexp}{E}

\newmacro{\graph}{\mathcal{G}}
\newmacro{\vertices}{\mathcal{V}}
\newmacro{\edges}{\mathcal{E}}

\newmacro{\mat}{M}		%
\newmacro{\hmat}{H}		%

\newmacro{\ones}{\mathbf{1}}		%
\newmacro{\eye}{I}		%
\newmacro{\zer}{\mathbf{0}}		%

\DeclarePairedDelimiter{\norm}{\lVert}{\rVert}		%
\DeclarePairedDelimiterXPP{\dnorm}[1]{}{\lVert}{\rVert}{_{\ast}}{#1}		%
\DeclarePairedDelimiterXPP{\sqnorm}[1]{}{\lVert}{\rVert}{^2}{#1}		%

\DeclarePairedDelimiter{\absval}{|}{|}		%
\DeclarePairedDelimiterXPP{\onenorm}[1]{}{\lVert}{\rVert}{_{1}}{#1}		%
\DeclarePairedDelimiterXPP{\twonorm}[1]{}{\lVert}{\rVert}{_{2}}{#1}		%
\DeclarePairedDelimiterXPP{\supnorm}[1]{}{\lVert}{\rVert}{_{\infty}}{#1}		%

\DeclarePairedDelimiterX{\braket}[2]{\langle}{\rangle}{#1,#2}		%
\DeclarePairedDelimiterX{\inner}[1]{\langle}{\rangle}{#1}

\newmacro{\vecspace}{\mathcal{V}}		%
\newmacro{\subspace}{\mathcal{W}}		%

\newmacro{\coord}{i}		%
\newmacro{\coordalt}{j}		%
\newmacro{\coordaltalt}{k}		%
\newmacro{\nCoords}{d}		%
\newmacro{\dims}{\nCoords}		%
\newmacro{\dimsalt}{p}		%
\newmacro{\vdim}{\nCoords}		%

\newmacro{\pvec}{v}		%

\newmacro{\bvec}{e}		%
\newmacro{\bvecs}{\mathcal{E}}		%

\newmacro{\pspace}{\mathcal{X}}		%
\newmacro{\dspace}{\pspace^{\ast}}		%

\newmacro{\dvec}{u}		%
\newmacro{\dbvec}{\eps}		%

\newmacro{\dpoint}{y}		%
\newmacro{\dpointalt}{\alt\dpoint}		%
\newmacro{\dpointaltalt}{\altalt\dpoint}		%
\newmacro{\dpoints}{\mathcal{Y}}		%

\newmacro{\dstate}{Y}		%
\newmacro{\dbase}{w}		%

\newmacro{\dualvar}{\lambda}
\newmacro{\dualvaralt}{\mu}
\newmacro{\dualfunc}{\phi}
\newmacro{\dualfuncalt}{\psi_\radius}
\newmacro{\empdualfunc}{\phi_\nsamples}
\newmacro{\lbdualvar}{\lb{\dualvar}}
\newmacro{\ubdualvar}{\ub{\dualvar}}
\newmacro{\basedualvar}{\dualvar^*_0}
\newcommand{\apprx}[1]{\overline{#1}}
\newcommand{\ddfunc}{Y} %

\newcommand{\defeq}{\coloneqq}		%

\newcommand{\from}{\colon}		%

\newop{\Opt}{Opt}		%
\newop{\Sol}{Sol}		%
\newop{\gap}{Gap}		%
\newop{\orcl}{Or}		%

\newop{\primal}{(P)}
\newop{\dual}{(Q)}

\newmacro{\tfun}{g}		%
\newmacro{\obj}{f}		%
\newmacro{\objalt}{g}		%
\newmacro{\objaltalt}{h}		%
\newmacro{\sobj}{F}		%

\newmacro{\gvec}{g}		%
\newmacro{\oper}{A}		%
\newmacro{\vecfield}{v}		%

\newcommand{\sol}[1][\point]{{#1}^{\star}}		%
\newmacro{\vecsol}{\vecfield(\sol)}		%

\newmacro{\signal}{V}		%
\newmacro{\step}{\gamma}		%
\newmacro{\learn}{\eta}		%

\newmacro{\vbound}{G}		%

\newmacro{\lips}{Lip}		%
\newmacro{\lipcst}{L}
\newmacro{\bdcst}{M}

\newmacro{\bdcstalt}{\widetilde{F}}
\newmacro{\fbound}{\bdcstalt(0)}		%

\newmacro{\varbdcst}{K}
\newmacro{\lbcst}{a}
\newmacro{\ubcst}{b}
\newmacro{\strong}{\mu^*}		%
\newmacro{\strongOpt}{\mu}		%
\newmacro{\strongMap}{\nu}
\newmacro{\smooth}{L_2}
\newmacro{\hsmooth}{L_3}
\newmacro{\plainhsmooth}{L}
\newmacro{\tmplipcst}{L}
\newmacro{\tmpbdcst}{B}

\newop{\cone}{cone}
\newop{\tspace}{T}		%
\newop{\nspace}{N}		%
\newop{\tcone}{TC}		%
\newop{\dcone}{DC}		%
\newop{\ncone}{NC}		%
\newop{\regncone}{\widehat{NC}}		%
\newop{\pcone}{PC}		%
\newop{\hull}{\Delta}		%

\newmacro{\cvx}{\mathcal{C}}		%
\newmacro{\subd}{\partial}		%

\newmacro{\minmax}{\mathcal{L}}		%

\newmacro{\minvar}{\point_{1}}		%
\newmacro{\minvaralt}{\alt\minvar}		%
\newmacro{\minvars}{\points_{1}}		%
\newmacro{\minsol}{\sol[\minvar]}		%

\newmacro{\maxvar}{\point_{2}}		%
\newmacro{\maxvaaltr}{\alt\maxvar}		%
\newmacro{\maxvars}{\points_{2}}		%
\newmacro{\maxsol}{\sol[\maxvar]}		%

\newop{\Eucl}{\Pi}		%
\newop{\logit}{\Lambda}		%
\newop{\dkl}{KL}		%

\newmacro{\hreg}{h}		%
\newmacro{\hconj}{\hreg^{\ast}}		%
\newmacro{\breg}{D}		%
\newmacro{\mprox}{P}		%
\newmacro{\mirror}{Q}		%
\newmacro{\fench}{F}		%
\newmacro{\hstr}{K}		%
\newmacro{\depth}{H}		%
\newmacro{\proxdom}{\points_{\hreg}}		%
\newmacro{\zone}{\mathbb{D}}		%
\newmacro{\bregkernel}{\theta} %

\DeclarePairedDelimiterXPP{\proxof}[2]{\mprox_{#1}}{(}{)}{}{#2}		%

\newmacro{\bregexp}{\alpha}
\newmacro{\bregcst}{M}

\newmacro{\kernelcst}{C}
\newmacro{\kernelexp}{q}

\newmacro{\point}{x}		%
\newmacro{\pointalt}{\alt\point}		%
\newmacro{\pointaltalt}{\altalt\point}		%
\newmacro{\points}{\mathcal{K}}		%
\newmacro{\intpoints}{\relint\points}		%

\newmacro{\basealt}{q}		%
\newmacro{\basealtalt}{u}		%

\newmacro{\open}{\mathcal{U}}		%
\newmacro{\closed}{\mathcal{C}}		%
\newmacro{\cpt}{\mathcal{K}}		%
\newmacro{\nbd}{\mathcal{U}}		%
\newmacro{\nhd}{\nbd}		%
\newmacro{\nbdalt}{\mathcal{V}}		%
\newmacro{\nbdaltalt}{\mathcal{W}}		%
\newmacro{\mset}{A}

\NewDocumentCommand{\ex}{E{_}{\coupling} o}{\mathbb{E}_{#1}\IfValueT{#2}{\bracks*{#2}}}		%
\newop{\prob}{P}%
\newcommand{\dirac}[1]{\debug{\delta}_{#1}}
\newcommand{\probs}[1][\samples]{\debug{\mathcal{P}}(#1)}
\newop{\empirical}{\WAedit{\widehat{\prob}_{\nsamples}}}%
\newop{\emp}{\empirical}%
\newmacro{\empex}{\frac{1}{\nsamples}\sum_{\ind=1}^\nsamples}
\newop{\probalt}{Q}
\newop{\coupling}{\pi}
\newop{\couplingalt}{\alt\pi}
\NewDocumentCommand{\couplings}{e{_} O{\samples}}{\debug{\mathcal{P}}\IfValueT{#1}{_{#1}}(#2\times#2)}
\newop{\Var}{Var}		%
\newop{\simplex}{\hull}		%

\newop{\rad}{s}

\newcommand{\partition}[1][\sample,\sdev]{Z(#1)}

\newcommand{\ubpartition}[1][\sdev]{\overline{Z}(#1)}

\DeclarePairedDelimiterXPP{\exof}[1]{\ex}{[}{]}{}{%
 #1}

\DeclarePairedDelimiterXPP{\probof}[2]{#1}{(}{)}{}{%
 #2}

\DeclarePairedDelimiterXPP{\oneof}[1]{\one}{\{}{\}}{}{%
 #1}

\newmacro{\sample}{\xi}		%
\newmacro{\samplealt}{\altsample} %
\newmacro{\altsample}{\zeta}
\newmacro{\samplealtalt}{\alt\altsample}
\newmacro{\samples}{\Xi}		%

\newmacro{\optsample}{\xi^*}
\newmacro{\optoptsample}{\xi^{**}}

\newmacro{\nsamples}{n}
\newmacro{\Nsamples}{N}
\newmacro{\nsamplesalt}{m}

\newmacro{\target}{y}
\newmacro{\Point}{X}
\newmacro{\Target}{Y}
\newmacro{\Pointalt}{\alt\Point}
\newmacro{\Targetalt}{\alt\Target}

\newmacro{\targetalt}{\alt\target}
\newmacro{\targetcost}{\kappa}

\newmacro{\distance}{d}
\newmacro{\distfunc}{D}
\newmacro{\Hdist}{D_H}

\newmacro{\cost}{\distance^2} %
\newmacro{\Cost}{C}
\newmacro{\mincost}{\sol[c]}
\newmacro{\maxcost}{\sol[C]}

\NewDocumentCommand{\wass}{s e{^} O{} o m}{\debug{W}\IfValueT{#2}{^{{#2}}}\IfValueT{#4}{^{#4}}_{#3}\IfBooleanT{#1}{\left}(#5\IfBooleanT{#1}{\right})}

\newmacro{\filter}{\mathcal{F}}		%
\newmacro{\probspace}{(\samples,\filter,\prob)}		%

\newmacro{\radius}{\rho}
\newmacro{\radiusalt}{\alt\radius}
\newmacro{\minradius}{{\overline{\rho}^2_\nsamples}}
\newmacro{\cstminradius}{{\rho_{\mathrm{thres}}}}
\newcommand{\regcrit}[1][\reg, \sdev]{\radius_c^2(#1)}

\newmacro{\GeoRad}{R}
\newmacro{\OptGeoRad}{R^*}
\newmacro{\georad}{r}
\newmacro{\ballradius}{r}
\newmacro{\ballradiusalt}{r'}

\newmacro{\Margin}{\Delta}

\newmacro{\event}{E}       %
\newmacro{\eventalt}{H}       %
\newmacro{\mean}{\mu}		%
\newmacro{\sdev}{\sigma}		%
\newmacro{\variance}{\sdev^{2}}		%

\newmacro{\level}{\alpha}
\newmacro{\levelalt}{p}

\NewDocumentCommand{\Lspace}{E{^}{{1}} O{\prob}}{\debug{L^{#1}(#2)}}

\NewDocumentCommand{\Obj}{e{^} E{_}{\radius} D(){\obj, \prob}}{
    {F}\IfValueT{#1}{^{#1}}\IfValueT{#2}{_{#2}}(#3)
}

\NewDocumentCommand{\Objalt}{e{^} E{_}{\radius} D(){\obj}}{
    {G}\IfValueT{#1}{^{#1}}\IfValueT{#2}{_{#2}}(#3)
}

\newmacro{\ind}{i}
\newmacro{\indalt}{j}

\newmacro{\reg}{\varepsilon}
\newmacro{\regalt}{\tau}
\newmacro{\basereg}{\base[\reg]}
\newmacro{\basesdev}{\base[\sdev]}
\newmacro{\baseregalt}{\base[\reg]'}
\newmacro{\basesdevalt}{\base[\sdev]'}
\newmacro{\regparam}{\varepsilon}
\newmacro{\regparamalt}{\delta}
\newmacro{\regparamaltalt}{\tau}
\newmacro{\Regparam}{\Delta}

\newmacro{\rv}{X}
\newmacro{\altrv}{Y}
\newmacro{\discrprob}{p}
\newmacro{\discrprobalt}{q}

\NewDocumentCommand{\measures}{O{\samples^2} e{^}}{\mathcal{M}\IfValueT{#2}{^{#2}}(#1)}

\newcommand{\contfuncs}[1][\samples^2]{\mathcal{C}(#1)}

\newcommand{\base}[1][\coupling]{{#1}_0}
\newcommand{\modbase}[1][\func/\reg]{\coupling_\sdev^{#1}}
\newcommand{\condmodbase}[1][\func/\reg]{\coupling_\sdev^{#1}(\cdot | \sample)}
\newcommand{\basecpl}{\coupling_\sdev}
\newcommand{\empbasecpl}{\coupling^\nsamples_\sdev}

\newmacro{\scalar}{u}

\newmacro{\exponent}{p}
\newmacro{\pexp}{p}
\newmacro{\qexp}{q}

\newmacro{\volcst}{V}

\newmacro{\funcs}{\mathcal{F}}
\newcommand{\dudley}[1][\funcs, \norm{\cdot}_\infty]{\mathcal{I}(#1)}
\newmacro{\func}{\obj}
\newmacro{\funcalt}{\objalt}
\newmacro{\funcaltalt}{\objaltalt}

\newmacro{\proper}{\tau}		%

\newmacro{\error}{Z}		%
\newmacro{\bias}{b}		%
\newmacro{\brown}{W}		%

\newmacro{\serror}{\theta}		%
\newmacro{\snoise}{\xi}		%
\newmacro{\sbias}{\psi}		%

\newmacro{\sbound}{M}		%
\newmacro{\bbound}{B}		%
\newmacro{\noisepar}{\sdev}		%
\newmacro{\noisevar}{\variance}		%

\newcounter{cnstcnt}
\newcommand{\newcst}[1]{%
\refstepcounter{cnstcnt}%
\ensuremath{c_\thecnstcnt}%
\label{#1}
}
\newcommand{\cst}[1]{\ensuremath{c_{\ref{#1}}}}

\newcommand{\assumptionRadiusAsympt}{
    \radius \leq \bdRadiusAsympt
}
\newcommand{\bdRadiusAsympt}{
    \min\parens*{\frac{\init[\reg]}{\basereg}, \frac{\basedualvar}{32(\init[\dualvar] + \smooth)},
     \frac{\strong(\basedualvar)^2}{4096 \smooth},
         \sqrt{\frac{\cst{approx_phi_rate}^{{3}}\basedualvar}{8 \basereg}} \parens*{
            \log\parens*{\frac{4096 \basereg \cst{approx_phi_coeff}}{\strong (\basedualvar)^2}}
        }^{-\frac{3}{2}}_+ %
    }
    }

\newmacro{\uncertainty}{U}

\NewDocumentCommand{\risk}{ooo}{
    {\mathcal{R}}\IfValueT{#2}{^{#2}}\IfValueT{#1}{_{#1}}(\IfValueTF{#3}{#3}{\obj})
}

\NewDocumentCommand{\emprisk}{ooo}{
    \widehat{\mathcal{R}}\IfValueT{#2}{^{#2}}\IfValueT{#1}{_{#1}}(\IfValueTF{#3}{#3}{\obj})
}

\usepackage{comment}
\addauthor[Waïss]{WA}{DarkGreen}

\newcommand{\half}[1][1]{\frac{#1}{2}}
\newcommand{\third}[1][1]{\frac{#1}{3}}
\newcommand{\unfrac}[2]{{#1}/{#2}}
\newcommand{\inv}[1]{\frac{1}{#1}}
\newcommand{\lb}[1]{\underline{#1}}
\newcommand{\ub}[1]{\overline{#1}}

\usepackage[final]{showlabels}

\addauthor[Franck]{FI}{Purple}

\addauthor[Jérôme]{JM}{DarkRed}

\addauthor[Pan]{PM}{MediumBlue}

\newmacro{\fn}{f}		%
\newmacro{\fixmap}{F}		%
\newmacro{\silver}{\Phi}		%
\newmacro{\energy}{E}		%
\newmacro{\Lyap}{L}		%
\newmacro{\diff}{\alpha} 

\newmacro{\thres}{\delta}		%
\newmacro{\basin}{\mathcal{B}}		%
\newmacro{\inhd}{\init[\nbd]}

\newmacro{\seed}{\theta}		%
\newmacro{\seeds}{\Theta}		%
\newmacro{\pdist}{P}		%
\newmacro{\history}{\mathcal{H}}		%

\newacro{LHS}{left-hand side}
\newacro{RHS}{right-hand side}
\newacro{iid}[i.i.d.]{independent and identically distributed}
\newacro{lsc}[l.s.c.]{lower semi-continuous}
\newacro{usc}[u.s.c.]{upper semi-continuous}
\newacro{rv}[r.v.]{random variable}
\newacro{NE}{Nash equilibrium}

\newacro{ABP}{abstract Bregman proximal}
\newacro{BP}{Bregman proximal}

\newacro{DGF}{distance-generating function}
\newacro{EG}{extra-gradient}
\newacro{MP}{mirror-prox}
\newacro{MD}{mirror descent}
\newacro{OMD}{optimistic mirror descent}
\newacro{OMWU}{optimistic multiplicative weights update}
\newacro{PMP}{past mirror-prox}
\newacro{AMP}{abstract mirror-prox}
\newacro{MPT}{mirror-prox template}

\newacro{VI}{variational inequality}
\newacro{VIP}{variational inequality problem}
\newacro{KKT}{Karush\textendash Kuhn\textendash Tucker}
\newacro{FOS}{first-order stationary}
\newacro{SOO}{second-order optimality}
\newacro{SOS}{second-order sufficient}
\newacro{DGF}{distance-generating function}
\newacro{SFO}{stochastic first-order oracle}

\newacro{DRO}{distributionally robust optimization}
\newacro{WDRO}{Wasserstein distributionally robust optimization}
\newacro{MMD}{Maximum Mean Discrepancy}
\newacro{ML}{machine learning}
\newacro{SVM}{support vector machines}
\newacro{ERM}{empirical risk minimization}
\newacro{OT}{optimal transport}
\newacro{ELBO}{evidence lower bound}
\newacro{MCMC}{Monte Carlo Markov Chain}
\newacro{SAEM}{stochastic approximation expectation-maximization}
\newacro{AD}{automatic differentiation}
\newacro{OR}{operational research}
\newacro{PAC}{probably approximately correct}
\newacro{SA}{stochastic approximation}

\newacro{KL}{Kullback-Leibler}

%% file: body.tex
\section{Introduction}
\label{sec:introduction}

\subsection{Generalization and (Wasserstein) Distributionally Robust Models}

We consider the fundamental question of generalization of machine learning models. Let us denote by $\obj_\param$ the loss induced by a model parametrized by $\param$ for some uncertain variable $\sample$ (typically a data point). 
When $\sample$ follows some distribution $\prob$, 
seeking the best parameter $\param$ writes as minimizing the expected loss
\begin{align}\label{eq:ideal}
    \min_{\param \in \params} ~\ex_{\sample \sim \prob}[\obj_\param(\sample)]\,.
\end{align}
We usually do not have a direct knowledge of $\prob$ but rather we have access to samples  $(\sample_\ind)_{\ind=1}^\nsamples$ independently drawn from $\prob$. 
The empirical risk minimization approach then consists in minimizing the expected loss over the associated empirical distribution $\empirical = \empex \dirac{\sample_\ind}$ (as a proxy for the expected loss over $\prob$), \ie
\begin{align}
  \label{eq:sto}
  \min_{\param \in \params} ~\ex_{\sample \sim \empirical}[\obj_\param(\sample)] ~ \left(= \frac{1}{\nsamples} \sum_{\ind=1}^\nsamples \obj_\param(\sample_\ind)\right) .
\end{align}
Classical statistical learning theory ensures that, with high probability, $\ex_{\prob}[\func_\param]$ is close to $\ex_{\empirical}[\func_\param]$ up to $\bigoh \parens*{{1}/{\sqrt \nsamples}}$ error terms, see \eg the monographs \citet{boucheronconcentrationinequalitiesnonasymptotic2013,wainwright2019high}.

A practical drawback of empirical risk minimization is that it can lead to over-confident decisions (when $\ex_{\empirical}[\func_\param] <  \ex_{\prob}[\func_\param]$, the real loss can be higher that the empirical one \citep{esfahani2018data}).
In addition, this approach is also sensitive to distribution shifts between training and application.
To overcome these drawbacks, an approach gaining momentum in machine learning is \emph{distributionally robust} optimization, which consists in minimizing the \emph{worst expectation} of the loss when the distribution lives in a {neighborhood} %
of $\empirical$:
\begin{align}
    \min_{\param \in \params} ~ \sup_{\probalt \in \nhd(\empirical)}\ex_{\sample \sim \probalt}[\obj_\param(\sample)]\,,
    \label{eq:intro_dro}
\end{align}
where the inner $\sup$ is thus taken over %
$\probalt$
in the {neighborhood} $\nhd(\empirical)$ 
of $\empirical$ in the space of probability distributions.
Popular choices of distribution neighborhoods are based on the \ac{KL} divergence  \citep{laguel2020first,levy2020large}, kernel tools \citep{zhu2021kernel,staib2019distributionally,zhu2021adversarially}, moments \citep{delageDistributionallyRobustOptimization2010,gohDistributionallyRobustOptimization2010}, or Wasserstein distance \citep{shafieezadehabadehdistributionallyrobustlogistic2015,esfahani2018data}. 
If $\prob \in \nhd(\empirical)$, distributionally robust models can benefit from direct generalization guarantees as, %
\begin{align}
    \sup_{\probalt \in \nhd(\empirical)}\ex_{\sample \sim \probalt}[\obj_\param(\sample)] ~~  \geq ~~ \ex_{\sample \sim \prob}[\obj_\param(\sample)]  .
    \label{eq:conc}
\end{align}
Thus, for well-chosen neighborhoods $\nhd(\empirical)$, distributionally robust objectives are able to provide \emph{exact upper-bounds} on the expected loss over distribution $\prob$, \ie the true risk.

\ac{WDRO} problems correspond to \eqref{eq:intro_dro}  with 
\begin{align}\label{eq:nhd}
    \nhd(\empirical) = \left\{\probalt \in \probs: \wass{\empirical, \probalt} \leq \radius\right\} \, ,
\end{align}
where $\wass{\empirical, \probalt}$ denotes the  Wasserstein distance between $\empirical$ and $\probalt$  and $\radius > 0$ controls the required level of robustness around $\empirical$.  
As a natural metric to compare discrete and absolutely continuous probability distributions, the Wasserstein distance has attracted a lot of interest in both machine learning \citep{shafieezadehabadehdistributionallyrobustlogistic2015,sinha2018certifying,shafieezadehabadehregularizationmasstransportation2019,lifastepigraphicalprojectionbased2020,kwon2020principled} and operation research \citep{zhao2018data,arrigo2022wasserstein} communities; see \eg the review articles \cite{blanchet2021statistical,kuhn2019wasserstein}.

\ac{WDRO} benefits from out-of-the-box generalization guarantees in the form of \eqref{eq:conc} since it inherits the concentration properties of the Wasserstein distance \WAedit{\citep{esfahani2018data}}.
More precisely,  under mild assumptions on $\prob$, \citep{fournier2015rate} \WAedit{establishes that} $\wass{\empirical, \prob} \leq \radius$ with high probability as soon as $\radius \sim 1/\nsamples^{1/\dims}$ where $\dims$ denotes the dimension of the samples space. 
Thus, a major issue is the prescribed radius $\radius$ suffers from the curse of the dimensionality:  when $\dims$ is large, $\radius$ decreases slowly as the number of samples $\nsamples$ increases.
\WAedit{This constrasts with other \acl{DRO} ambiguity sets, such as \ac{MMD} \citep{staib2019distributionally,zeng2022generalization}, where the radius scales as $1/\sqrt \nsamples$.}
\WAedit{Moreover, the existing scaling for \ac{WDRO} is overly conservative for \ac{WDRO} objectives since recent works} \citep{blanchet2022confidence,blanchet2023statistical} prove that a radius behaving as $1/\sqrt \nsamples$ is asymptotically optimal.
\WAedit{
    The main difference with \citep{esfahani2018data} is that they --- and us --- consider the \ac{WDRO} objective as a whole, instead of proceeding in two steps: first considering the Wasserstein distance independently and invoking concentration results on the Wasserstein distance and then plugging this result in the \ac{WDRO} problem.
}

\subsection{Contributions and related works} %

In this paper, we show that \ac{WDRO} provides exact upper-bounds on the true risk with high probability. %
More precisely, we prove non-asymptotic generalization bounds of the form of \cref{eq:conc}, that hold for general classes of functions, and that only require $\radius$ to scale as $1 / \sqrt \nsamples$ and not $1 / \nsamples^{1/\dims}$.
To do so, we construct an interval for the radius $\radius$ for which it is both sufficiently large so that we can go from the empirical to the true estimator (\ie at least of the order of $1 / \sqrt \nsamples$) and  sufficiently small so that the robust problem does not become degenerate (\ie smaller than some critical radius, that we introduce as an explicit constant).
Our results imply proving concentration results on Wasserstein Distributionally Robust objectives that are of independent interest.

This work is part of a rich and recent line of research about theoretical guarantees on \ac{WDRO} for machine learning. One of this first results, \citet{lee2018minimax}, provides generalization guarantees, for a general class of models and a fixed $\radius$, that, however, become degenerate as the radius goes to zero. 
In the particular case of linear models, \ac{WDRO} models admit an explicit form 
that allows
\citet{shafieezadehabadehregularizationmasstransportation2019, chen2018robust} to provide generalization guarantees \cref{eq:conc} \WAedit{with the radius scaling as} $1 / \sqrt \nsamples$. %
The case of general classes of models\WAedit{, possibly non-linear,} is more intricate. \citet{sinha2018certifying} showed that a modified version of \cref{eq:conc} holds at the price of non-negligible error terms. \cite{gao2022finite,an2021generalization} made another step towards broad generalization guarantees for \ac{WDRO} but with error terms that vanish only when $\radius$ goes to zero.

In contrast, our analysis provides exact generalization guarantees \WAedit{in the form \cref{eq:conc}} without additional error terms, that hold for general classes of functions
and allow for a non-vanishing uncertainty radius to cover for distribution shifts at testing.
Moreover, our guarantees also carry over to the recently introduced regularized versions of \ac{WDRO} \citep{wang2021sinkhorn,azizian2022regularization}, whose statistical properties have not been studied yet.

This paper is organized as follows.
In \cref{sec:setup}, we introduce notations and our blanket assumptions.
In \cref{sec:main}, we present our main results, an idea of proof, and discussions. The complete proofs are deferred to the appendix.

\section{Setup and Assumptions}
\label{sec:setup}

In this section, we formalize our setting and introduce Wasserstein Distributionally Robust risks.

\subsection{Wasserstein Distributionally Robust risk functions} \label{sec:risk}

In this paper, we consider as a samples space $\samples$ a subset of $\R^\dims$ equipped with the Euclidean norm $\norm{\cdot}$. We rely on Wasserstein distances of order $2$, in line with the seminal work \citet{blanchet2022confidence} on generalization of \ac{WDRO}. This distance is defined for two distributions $\probalt, \alt\probalt$ in the set of probability distributions on $\samples$, denoted by  $\probs$,  as
\begin{align}
    \wass[2][]{\probalt, \alt\probalt} \defeq \left( \inf_{ \coupling \in \couplings, \coupling_1 = \probalt, \coupling_2 = \alt\probalt } \ex_{(\sample,\samplealt)\sim\coupling} \left[ \frac{1}{2} \| \sample-\samplealt \|^2  \right] \right)^{1/2}\,,
\end{align}
where $\couplings$ is the set of probability distributions in the product space $\samples\times\samples$, and $ \coupling_1$ (resp. $ \coupling_2$) denotes the first (resp. second) marginal of $\coupling$.

We denote by $\obj: \samples \to \R$ the loss function of some model over the sample space.
The model may depend on some parameter $\param$, that we drop for now to lighten the notations; instead, we consider a class of functions $\funcs$ encompassing our various models and losses of interest  (we come back to classes of parametric models of the form $ \funcs =  \{ \obj_\param : \param\in\params \}$  in \cref{sec:examples}).

We define the empirical Wasserstein Distributionally Robust risk $\emprisk[\radius^2]$ centered on $\empirical$ and similarly the true robust risk $\risk[\radius^2]$ centered on $\prob$ as
\begin{align}\label{eq:emprisk-risk}
    \emprisk[\radius^2] \defeq  \sup_{\substack{\probalt \in \probs \\ \wass[2][2]{\empirical, \probalt} \leq \radius^2}}\ex_{\sample \sim \probalt}\left[ \obj(\sample) \right] \qquad \text{and} \qquad \risk[\radius^2] \defeq  \sup_{\substack{\probalt \in \probs \\ \wass[2][2]{\prob, \probalt} \leq \radius^2 }}\ex_{\sample \sim \probalt}\left[ \obj(\sample) \right]\,.
\end{align}
Note that $\emprisk[\radius^2] $, which is based on the empirical distribution $\empirical$, is a computable proxy for the true robust risk $\risk[\radius^2]$. 
Note also that the true robust risk $\risk[\radius^2]$ immediately upper-bounds the true (non-robust) risk  $\ex_{\sample \sim \prob} \left[ \obj(\sample) \right]$ and also upper-bounds $\ex_{\sample \sim \probalt} \left[ \obj(\sample) \right]$ for neighboring distributions $\probalt$ that correspond to distributions shifts of magnitude smaller than $\radius$ in Wasserstein distance.

\subsection{Regularized versions} \label{sec:reg}

Entropic regularization of \ac{WDRO} problems was recently \WAedit{studied} in \cite{wang2021sinkhorn,blanchet2020semi,piat2022regularized,azizian2022regularization} \WAedit{and used in \citet{dapogny2023entropy,song2023provably,wang2022data,wang2022improving}}. Inspired by the entropic regularization in \ac{OT} \citep[Chap.~4]{peyre2019computational}, the idea is to regularize the objective by adding a \ac{KL} divergence, that is defined, for any transport plan $\coupling \in \probs[\samples \times \samples]$ and a fixed reference $\overline{\coupling} \in \probs[\samples \times \samples]$, by
\begin{equation}
    \dkl(\coupling |\overline{\coupling}) = 
    \begin{cases}
    \int \log\frac{\dd\coupling}{\dd\overline{\coupling}} \,\dd\coupling
    &\text{when}
    \coupling \ll \overline{\coupling}\\
        +\infty &\text{otherwise}.
    \end{cases}
\end{equation}
Unlike in \ac{OT} though, the choice of the reference measure in \ac{WDRO} is not neutral and introduces a bias in the robust objective \citep{azizian2022regularization}.
For their theoretical convenience, %
we take reference measures that have Gaussian conditional distributions
\begin{align}
    \label{eq:baseconditional}
\basecpl(\dd\samplealt | \sample) \propto \one_{\altsample \in \samples} e^{-\frac{ \sqnorm{\sample- \altsample}}{2\sdev^2}} \dd\samplealt, \qquad \text{for all }\sample \in \samples\,,
\end{align}
where $\sdev > 0$ controls the spread of the second marginals,
following \cite{wang2021sinkhorn,azizian2022regularization}.
Then, the regularized version of $\emprisk[\radius^2]$ (\ac{WDRO} empirical risk) is given by
\begin{align}
    \label{eq:emp-reg-risk}
\emprisk[\radius^2][\reg] &\defeq \hspace*{-1em} \sup_{\substack{ \coupling \in \couplings, \coupling_1 = \empirical \\ \ex_{(\sample,\samplealt)\sim\coupling} \left[ \frac{1}{2} \| \sample-\samplealt \|^2  \right] \leq \radius^2}}\hspace*{-1em} \ex_{\sample \sim \coupling_2}\left[ \obj(\sample) \right] - \reg \dkl \left(\coupling  \left| \empbasecpl \right. \right) \qquad\text{with }\empbasecpl = \empirical(\dd\sample) \basecpl(\dd\samplealt | \sample)
\end{align}
and similarly, the regularized version of $\risk[\radius^2]$ is given by
\begin{align}
    \label{eq:true-reg-risk}
\risk[\radius^2][\reg] &\defeq \hspace*{-1em} \sup_{\substack{ \coupling \in \couplings, \coupling_1 = \prob \\ \ex_{(\sample,\samplealt)\sim\coupling} \left[ \frac{1}{2} \| \sample-\samplealt \|^2  \right] \leq \radius^2}}\hspace*{-1em} \ex_{\sample \sim \coupling_2}\left[ \obj(\sample) \right] - \reg \dkl \left(\coupling \left| \basecpl  \right. \right)
\qquad\text{with }\basecpl = \prob(\dd\sample) \basecpl(\dd\samplealt | \sample)
.
\end{align}

These regularized risks have been studied in terms of computational or approximation properties, but their statistical properties have not been investigated yet. The analysis we develop for \ac{WDRO} estimators is general enough to carry over to these settings.

\WAedit{
    In \cref{eq:emp-reg-risk,eq:true-reg-risk}, note finally that the regularization is added as a penalization in the supremum, rather than in the constraint. As in \citet{wang2021sinkhorn}, penalizing in the constraint leads to an ambiguity set defined by the regularized Wasserstein distance, that we introduce in \cref{sec:main-reg}. We refer to \citet{azizian2022regularization} for a unified presentation of the two penalizations.
}

\subsection{Blanket assumptions} \label{sec:asm}

Our analysis is carried under the following set of assumptions that will be in place throughout the paper.
First, we assume that the sample space $\samples \subset \R^\dims$ is convex and compact, which is in line with previous work, \eg \citep{lee2018minimax,an2021generalization}.
\vspace*{1em}
\begin{assumption}[On the set $\samples$]\label{assumption:set}
    The sample space $\samples$ is a compact convex subset of $\R^\dims$.
\end{assumption}
Second, we require the class of loss functions $\funcs$ to be sufficiently regular. In particular, we assume that they have Lipschitz continuous gradients.
\vspace*{1em}
\begin{assumption}[On the function class]\label{assumption:funcs-simple}
    The functions of $\funcs$ are twice differentiable, uniformly bounded, and their derivatives are uniformly bounded and uniformly Lipschitz.
\end{assumption}
Finally, we assume that $\empirical$ is made of \ac{iid} samples of $\prob$ and that $\prob$ is supported on the interior of $\samples$ (which can be done without loss of generality by slightly enlarging $\samples$ if needed).
\vspace*{1em}
\begin{assumption}[On the distributions]\label{assumption:support_interior-simple}\label{assumption:iid}\label{assumption:support_interior}
   $\empirical = \frac{1}{\nsamples}\sum_{i = 1}^\nsamples \dirac{\sample_i}$ where $\sample_1,\dots,\sample_n$ are \ac{iid} samples of $\prob$.
    We further assume that there is some $\GeoRad > 0$ such that $\prob$ satisfies $\supp \prob + \ball(0, \GeoRad) \subset \samples$.
\end{assumption}

\section{Main results and discussions} \label{sec:main}

The main results of our paper establish that the empirical robust risk provide high probability bounds, of the form of \cref{eq:conc}, on the true risk. %
Since the results and assumptions slightly differ between the \ac{WDRO} models and their regularized counterparts, we present them separately in \cref{sec:main-unreg} and \cref{sec:main-reg}.
In \cref{sec:outline}, we provide the common outline for the proofs of these results, the proofs themselves being provided in the appendix. Finally, in \cref{sec:examples}, we detail some examples.

\subsection{Exact generalization guarantees for \ac{WDRO} models}\label{sec:main-unreg}

In this section, we require the two following additional assumptions on the function class.
The first assumption is common in the \ac{WDRO} litterature, see \eg \cite{blanchet2022confidence,blanchet2023statistical,gao2022finite,an2021generalization}.
\begin{assumption}%
    \label{assumption:gradient-funcs}
    The quantity $\inf_{\func \in \funcs} \ex_{\prob}[\norm{\grad \func}^2]$ is positive. 
\end{assumption}
The second assumption we consider in this section makes use of the notation $\distance(\sample, A)$, for a set $A \subset \samples$ and a point $\sample \in \samples$, to denote the distance between $\sample$ and $A$, \ie $\distance(\sample, A) = \inf_{\samplealt \in A} \norm{\sample - \samplealt}$.
\begin{assumption}
\label{assumption:add-geometric-simple}\label{assumption:add-geometric}
\leavevmode
    \begin{enumerate}
    \item For any $\GeoRad > 0$, there exists $\Margin > 0$ such that,\label{item:criterion-Hausdorff-compactness}
        \begin{align}
            \forall \func \in \funcs,\, \forall \samplealt \in \samples, \quad
            \distance(\samplealt, \argmax \func) \geq \GeoRad \implies \func(\samplealt) - \max \func \leq - \Margin\,.
        \end{align}
    \item The following growth condition holds: there exist $\strongOpt > 0$ and $\plainhsmooth > 0$ such that, for all $\func \in \funcs$, $\sample \in \samples$ and $\optsample$ a projection of $\sample$ on $\argmax \func$, \ie $\optsample \in \argmin_{\argmax \func} \norm{\sample - \cdot}$,
        \label{item:criterion-quadratic-growth}
        \begin{align}
            \func(\optsample) \geq \func(\sample) + \half[\strongOpt] \norm{\sample - \optsample}^2 - \frac{\plainhsmooth}{6} \norm{\sample - \optsample}^3\,.
        \end{align}
    \end{enumerate}
\end{assumption}

The first item of this assumption has a natural interpretation: we show in \cref{lemma:Hausdorff-compactness}, that it is equivalent to the relative compactness of the function space $\funcs$ \wrt to the distance
\begin{equation}
        \distfunc(\func, \funcalt) \defeq \norm{\func - \funcalt}_{\infty} + \Hdist(\argmax \func, \argmax \funcalt)\,,
\end{equation}
where $\Hdist$ denotes the (Hausdorff) distance between sets and $\norm{\func}_\infty \defeq \sup_{\sample \in \samples} \abs{\func(\sample)}$ is the infinity norm. %
The last one is a structural assumption on the functions $\funcs$ that is new in our context but is actually very close the so-called parametric Morse-Bott condition, introduced in of bilevel optimization \citep{arbelNonConvex}, see \cref{sec:morse-bott}.

We now state our main generalization result for \ac{WDRO} risks.

\begin{theorem}\label{thm:informal-unreg}
    Under \cref{assumption:gradient-funcs,assumption:add-geometric-simple}, there is an explicit constant $\radius_c$ depending only on $\funcs$ and $\prob$ such that for any $\thres\in(0,1)$ and $\nsamples\geq 1$, if
    \begin{equation}\label{eq:thm-concentration-wdro-unreg-radius}
        \bigoh \parens*{\sqrt{\frac{1 + {\log 1/\thres}}{ \nsamples}}} \leq \radius \leq \half[\radius_c] - \bigoh \parens*{\sqrt{\frac{1 + \log 1/\thres}{\nsamples}}} 
    \end{equation}
  then, there is $\radius_\nsamples = \bigoh \parens*{\sqrt{\frac{1 + {\log 1/\thres}}{ \nsamples}}} $ such that, with probability $1 - \thres$, 
    \begin{align}\label{eq:thm-concentration-wdro-unreg-rob}
        \forall \obj\in\funcs, \quad    \emprisk[\radius^2] \geq \ex_{\sample \sim \probalt} \left[ \obj(\sample) \right] \qquad
        \text{for all }  \probalt \text{such that $\wass[2][2]{\prob, \probalt} \leq \radius(\radius -  \radius_\nsamples)$}\,.
    \end{align}
    In particular, with probability $1-\thres$, we have 
    \begin{align}\label{eq:thm-concentration-wdro-unreg-gen}
    \forall \obj\in\funcs, \quad    \emprisk[\radius^2] \geq \ex_{\sample \sim \prob} \left[ \obj(\sample) \right] .
    \end{align}
\end{theorem}

The second part of the result, \cref{eq:thm-concentration-wdro-unreg-gen}, is an \emph{exact} generalization bound: it is an actual upper-bound on the true risk $\ex_{\sample \sim \prob}[\func(\sample)]$, that we cannot access in general, through a quantity that we can actually compute with $\empirical$. %
The first part of the result, \cref{eq:thm-concentration-wdro-unreg-rob} gives us insight into the robustness guarantees offered by the \ac{WDRO} risk. %
Indeed, it tells us that, when $\radius$ is greater than the minimal radius $\radius_\nsamples \propto 1/\sqrt \nsamples$ by some margin, the empirical robust risk $\emprisk[\radius^2]$ is an upper-bound on the loss even with some perturbations of the true distribution. %
Hence, as long as $\radius$ is large enough, the \ac{WDRO} objective enables us to guarantee the performance of our model even in the event of a distribution shift at testing time. 
In other words, the empirical robust risk  is an \emph{exact} upper-bound on the true robust risk $\risk[\radius(\radius - \radius_\nsamples)]$ with a reduced radius. %

The range of admissible radiuses is described by \cref{eq:thm-concentration-wdro-unreg-radius}. The lower-bound, roughly proportional to $1/\sqrt \nsamples$, is optimal, following the results of \citet{blanchet2022confidence}. The upper-bound, almost independent of $\nsamples$, depends on a constant $\radius_c$, that we call \emph{critical radius} and that has an interesting interpretation, that we formalize in the following remark.
Note, finally, that, the big-O notation in this theorem has a slightly stronger meaning\footnote{Eg.,~$\radius_\nsamples\!=\!\bigoh \parens*{\!\!\sqrt{\frac{1 + {\log 1/\thres}}{ \nsamples}}}$ means that $\exists C > 0$ such that $\radius_\nsamples \leq  C {\sqrt{\frac{1 + {\log 1/\thres}}{ \nsamples}}}$  for all $\thres \in (0, 1)$ and $\nsamples \geq 1$.}
than the usual one, being non-asymptotic in $\nsamples$ and $\thres$.
\begin{remark}[Interpretation of critical radius]\label{remark:radius-c}
    The critical radius $\radius_c$, appearing in \cref{eq:thm-concentration-wdro-unreg-radius}, is defined~by
    \begin{equation}\label{eq:radius-c}
    \radius_c^2 \defeq \inf_{\func \in \funcs} \ex_{\sample \sim \prob}[\half \distance^2(\sample, \argmax \func)]\, .
\end{equation}
It can be interpreted as the threshold at which the \ac{WDRO} problem \wrt $\prob$ starts becoming degenerate. Indeed, when $\radius^2 \geq \ex_{\sample \sim \prob}[\half \distance^2(\sample, \argmax \func)]$ for some $\func \in \funcs$ that we fix,
 the distribution $\probalt$ given by the second marginal of the transport plan $\coupling$ defined by,
\begin{equation}
   \coupling(\dd \sample, \dd \samplealt) \defeq \prob(\dd \sample) \dirac{\sol[\samplealt](\sample)}(\dd \samplealt)
    \quad\text{ where }\quad
    \sol[\samplealt](\sample) \in \argmin_{\samplealt \in \argmax \func} \cost(\sample, \samplealt)\,,
\end{equation}
satisfies
\begin{equation}
    \wass[2][2]{\prob, \probalt} \leq \ex_{(\sample, \samplealt) \sim \coupling}[\half \sqnorm{\sample - \samplealt}] = \ex_{\sample \sim \prob}[\half \distance^2(\sample, \argmax \func)] \leq \rho^2\,.
\end{equation}
As a consequence, the robust problem is equal to
\begin{equation}
\risk[\radius^2] =  \sup_{\substack{\probalt \in \probs \\ \wass[2][2]{\prob, \probalt} \leq \radius^2 }}\ex_{\sample \sim \probalt}\left[ \obj(\sample) \right]
= \max_{\sample\in\samples} \func(\sample)\,.
\end{equation}
Thus, when the radius exceeds $\radius_c$, there is some $\func$ such that the robust problem becomes degenerate as it does not depend on $\prob$ nor $\radius$ anymore.
\end{remark}
Finally, note that we can obtain the same generalization guarantee as \cref{thm:informal-unreg} without \cref{assumption:add-geometric-simple} at the expense of losing the above interpration on the condition on the radius. More precisely, we have the following result. %
\begin{theorem}\label{thm:informal-unreg-weak}
    Let \cref{assumption:gradient-funcs} hold.
    For any $\thres\in(0,1)$ and $\nsamples\geq 1$, if $\radius$ satisfies \cref{eq:thm-concentration-wdro-unreg-radius}, and if, in addition, it is smaller than a positive constant which depends only on $\prob$, $\funcs$ and $\samples$, then both conclusions of \cref{thm:informal-unreg} hold.
\end{theorem}

This theorem can be compared to existing results, and in particular with \citet{gao2022finite,an2021generalization}. These two papers provide generalization bounds for \ac{WDRO} under a similar assumption on $\funcs$ and a weakened version of \cref{assumption:iid}. However, these generalization bounds involve extra error terms, that require $\radius$ to be vanishing. %
In comparison, with a similar set of assumptions, \cref{thm:informal-unreg-weak} improves on these two issues, by allowing $\radius$ not to vanish as $\nsamples \to \infty$ and by providing the exact upper-bound \cref{eq:thm-concentration-wdro-unreg-gen}.
Allowing non-vanishing radiuses is an attractive feature of our results that enables us to cover distribution shifts.

\subsection{Regularized \ac{WDRO} models}\label{sec:main-reg}

The analysis that we develop for the standard \ac{WDRO} estimators is general enough to also cover the regularized versions presented in \cref{sec:reg}. We thus obtain the following \cref{thm:informal-reg} which is the first generalization guarantee for regularized \ac{WDRO}. This theorem is very similar to \cref{thm:informal-unreg} with still a couple of differences.
First, the regularization leads to ambiguity sets defined in terms of $\wass[2,\regalt]{\prob, \cdot}$, the \emph{regularized} Wasserstein distance to the true distribution $\prob$, defined, for some regularization parameter $\regalt>0$, as  %
\begin{align}\label{eq:reg-wass}
        \wass[2,\regalt][2]{\prob, \probalt} \defeq  \inf \left\{\ex_{\coupling} [ \half \sqnorm{\sample - \samplealt}] + \regalt \dkl( \coupling | \basecpl) : \coupling \in \couplings, \coupling_1 = \prob,\, \coupling_2 = \probalt\right\}\,,
    \end{align}
where $\basecpl$ appears in the definition of the regularized robust risk \cref{eq:true-reg-risk}.
    Besides, the regularization allows us to avoid \cref{assumption:gradient-funcs,assumption:add-geometric-simple} to show our generalization result.

\begin{theorem}
    \label{thm:informal-reg}
    For $\sdev = \base[\sdev] \radius$ with $\base[\sdev] > 0$, $\reg = \base[\reg] \radius$ with $\base[\reg] > 0$ such that $\base[\reg] / \base[\sdev]^2$ is small enough depending on $\funcs$, $\prob$, $\samples$,
    there is an explicit constant $\radius_c$ depending only on $\funcs$, $\prob$ and $\samples$ such that for all $\thres\in(0,1)$ and $\nsamples\geq 1$, if
    \begin{equation}
        \bigoh \parens*{\sqrt{\frac{1 + \log 1/\thres}{\nsamples}}} \leq \radius \leq \half[\radius_c] - \bigoh \parens*{\frac{1}{\sqrt \nsamples}}\,,
\quad\text{ and }\quad
\radius_c \geq \bigoh \parens*{\frac{1}{\nsamples^{1/6}} + \parens*{\frac{ 1 + \log 1/\thres}{\nsamples}}^{1/4}}\,,
    \end{equation}
then, there are $\regalt = \bigoh \parens*{\reg \radius}$ and $\radius_\nsamples = \bigoh \parens*{\sqrt{\frac{1 + {\log 1/\thres}}{ \nsamples}}}$ such that, with probability at least $1 - \thres$,
    \begin{align}\label{eq:thm-concentration-wdro-reg-rob}
        \forall \obj\in\funcs, \quad    \emprisk[\radius^2][\reg] \geq \ex_{\sample \sim \probalt} \left[ \obj(\sample) \right] 
        \qquad \text{for all } \probalt \text{ such that }\wass[2,\regalt][2]{\prob, \probalt} \leq \radius(\radius -  \radius_\nsamples)\,.
    \end{align}
    Furthermore, when $\basesdev$ and $\sdev$ are small enough depending on $\prob$ and $\samples$, with probability $1 - \thres$,
    \begin{align}\label{eq:thm-concentration-wdro-reg-gen}
    \forall \obj\in\funcs, \quad    \emprisk[\radius^2][\reg] \geq \ex_{\sample \sim \prob} \ex_{\samplealt \sim \basecpl(\cdot | \sample)} \left[ \obj(\samplealt) \right] .
    \end{align}
\end{theorem}

The first part of the theorem, \eqref{eq:thm-concentration-wdro-reg-rob}, guarantees that the empirical robust risk is an upper-bound on the loss even with some perturbations of the true distribution.
As in \ac{OT}, the regularization added to the Wasserstein metric induces a bias that may prevent $\wass[2,\regalt][2]{\prob, \prob}$ from being null. 
As a result, the second part of the theorem involves a smoothed version of the true risk: the empirical robust risk provides an \emph{exact} upper-bound the true expectation of a convolution of the loss with $\basecpl$.

A few additional comments are in order:
\begin{itemize}
    \item  Our result prescribes the scaling of the regularization parameters: $\reg$ and $\sdev$ should be taken proportional to $\radius$.
    \item The critical radius $\radius_c$ has a slighlty more intricate definition, yet the same interpretation as in the standard \ac{WDRO} case in\cref{remark:radius-c}; see \cref{sec:criticalreg}.
    \item The regularized \ac{OT} distances do not suffer from the curse of dimensionality \citep{genevay2019sample}. However this property does not directly carry over to regularized \ac{WDRO}. Indeed, we cannot choose the same reference measure as in \ac{OT} and {we have to fix the  measure $\basecpl$}, introducing a bias. As a consequence, we have to extend the analysis of the previous section to obtain the claimed guarantees that avoid the curse of dimensionality.
\end{itemize}

\subsection{Idea of the proofs} \label{sec:outline}

In this section, we present the main ideas of the proofs of \cref{thm:informal-unreg}, \cref{thm:informal-unreg-weak}, and \cref{thm:informal-reg}. The full proofs are detailed in appendix; we point to relevant sections along the discussion.
First, we recall the duality results for \ac{WDRO} that play a crucial role in our analysis. Second, we present a rough sketch of proofs that is common to both the standard and the regularized cases. Finally, we provide a refinement of our results that is a by-product of our analysis.

\paragraph{Duality in \ac{WDRO}.}
Duality has been a central tool in both the theoretical analyses and computational schemes of \ac{WDRO} from the onset \citep{shafieezadehabadehdistributionallyrobustlogistic2015,esfahani2018data}.
The expressions of the dual of \ac{WDRO} problems for both the standard case \citep{gao2016distributionally,blanchet2019quantifying} and the regularized case \citep{wang2021sinkhorn,azizian2022regularization} can be written with the following dual generator function $\dualfunc$ defined as 
\begin{align}\label{eq:dual_gen}
    \dualfunc(\func, \sample, \dualvar, \reg,\sdev)  &\defeq  \left\{ 
\begin{array}{ll}
    \sup_{\samplealt \in \samples} \braces*{\func(\samplealt) - \half[\dualvar] \sqnorm{\sample - \samplealt}} 
        &  \text{ if } \reg = 0\\
    \reg\log \parens*{\ex_{\samplealt \sim \basecpl(\cdot | \sample)} \exp\parens*{
            \frac{\func(\samplealt) - \dualvar \sqnorm{ \sample - \samplealt }/2}{\reg}}} 
            &\text{ if } \reg > 0\,,
\end{array}
    \right.
\end{align}
where $\dualvar$ is the dual variable associated to the Wasserstein constraint in \cref{eq:emprisk-risk,eq:emp-reg-risk,eq:true-reg-risk}.
The effect of regularization appears here clearly as a smoothing of the supremum. Note also that this function depends on the conditional reference measures $\basecpl(\cdot | \sample)$ but not on other probability distributions.
Then, under some general assumptions (specified in \cref{sec:asm} in appendix), the existing strong duality results yield that the (regularized) empirical robust risk writes
\begin{align}
    \emprisk[\radius^2][\reg] &= \inf_{\dualvar \geq 0} \dualvar \radius^2 + \ex_{\sample\sim\empirical} \left[ \dualfunc(\func, \sample, \dualvar, \reg,\sdev) \right]\,, \label{eq:dualreg} 
\end{align}
and, similarly, the (regularized) true robust risk writes
\begin{align}
    \risk[\radius^2][\reg] &= \inf_{\dualvar \geq 0} \dualvar \radius^2 + \ex_{\sample\sim\prob} \left[ \dualfunc(\func, \sample, \dualvar, \reg,\sdev) \right]  \label{eq:dualregtrue} \,.
\end{align}
These expressions for the risks are the bedrock of our analysis.

\paragraph{Sketch of proof.}
In both the standard case and the regularized case, our proof is built on two main parts: the first part is to obtain a concentration bound on the dual problems that crucially relies on a lower bound of the dual multiplier; the second part then consists in establishing such a lower bound.
All the bounds are valid with high probability, and we drop the dependency on the confidence level $\thres$ of the theorems for simplicity.

For the first part of the proof (\cref{sec:concentration}), we assume that there is a deterministic lower-bound $\lbdualvar > 0$ on the optimal dual multiplier in \cref{eq:dualreg} that holds with high-probability. As a consequence, we can restrict the range of $\dualvar$ in \cref{eq:dualreg} to obtain:
   \begin{align}
            &\emprisk[\radius^2][\reg] = \inf_{\dualvar\geq\lbdualvar} \left\{ \dualvar \radius^2 + \ex_{\sample\sim\empirical} \left[ \dualfunc(\func, \sample, \dualvar, \reg,\sdev) \right] \right\}\\
                                            &= \inf_{\dualvar\geq\lbdualvar}\left\{ \dualvar \radius^2 + \ex_{\sample\sim\prob} \left[ \dualfunc(\func, \sample, \dualvar, \reg,\sdev) \right] - \dualvar \frac{\ex_{\sample \sim \prob}[\dualfunc(\func, \sample, \dualvar, \reg,\sdev)] - \ex_{\sample \sim \emp}[\dualfunc(\func, \sample, \dualvar, \reg,\sdev)]}{\dualvar} \right\}\\
                                            &\geq \inf_{\dualvar\geq\lbdualvar}\left\{ \dualvar \radius^2 + \ex_{\sample\sim\prob} \left[ \dualfunc(\func, \sample, \dualvar, \reg,\sdev) \right] - \dualvar \sup_{\alt\dualvar\geq\lbdualvar}\frac{\ex_{\sample \sim \prob}[\dualfunc(\func, \sample, \alt\dualvar, \reg,\sdev)] - \ex_{\sample \sim \emp}[\dualfunc(\func, \sample, \alt \dualvar, \reg,\sdev)]}{\alt\dualvar} \right\}\\
                                            &\geq \inf_{\dualvar\geq\lbdualvar}\left\{ \dualvar \radius^2 + \ex_{\sample\sim\prob} \left[ \dualfunc(\func, \sample, \dualvar, \reg,\sdev) \right] - \dualvar \minradius\right\}\\
                                            &\geq \risk[\radius^2- \minradius][\reg]\,.
                                            \label{eq:sketch-of-proof}
        \end{align}        
        In the above, we used that the inner supremum, which is random, can be bounded by a deterministic and explicit quantity that we call $\minradius$, \ie
        \begin{align}
            \minradius \geq \sup_{\alt\dualvar\geq\lbdualvar}\frac{\ex_{\sample \sim \prob}[\dualfunc(\func, \sample, \alt\dualvar, \reg,\sdev)] - \ex_{\sample \sim \emp}[\dualfunc(\func, \sample, \alt \dualvar, \reg,\sdev)]}{\alt\dualvar}
            \quad \text{ with high probability.}
        \end{align}
        Hence, we obtain an upper-bound on the robust risk \wrt the true distribution with radius $\radius^2 - \minradius$. Moreover,  we show that $\minradius = \bigoh \parens*{1/\parens*{\lbdualvar \sqrt \nsamples}}$ which highlights the need for a precise lower bound $\lbdualvar$ to control the decrease in radius.

The second part of the proof thus consists in showing that the dual variable is indeed bounded away from~$0$, which means that the Wasserstein constraint is \emph{sufficiently active}. We have to handle two cases differently:
\begin{itemize}
\item when $\radius$ is small, \ie close to $\radius_\nsamples$ (\cref{sec:small}),
\item when $\radius$ is large, \ie close to the critical radius $\radius_c$ (\cref{sec:critical}).
    Note that the additional \cref{assumption:add-geometric-simple} is required here: where we need to control the behaviors of $\func \in \funcs$ close to their maxima (see \cref{eq:dual_gen} for $\reg = 0$ and small $\dualvar$).
\end{itemize}
In both cases we obtain that $\lbdualvar$ scales as $1/\radius$ for the respective ranges of admissible radiuses.  As a consequence $\minradius$ is bounded by $\radius \radius_n$ with $\radius_n = \bigoh \parens*{1 / \sqrt \nsamples}$ and \cref{eq:sketch-of-proof} becomes
\begin{equation}
    \label{eq:sketch-of-proof-end}
    \emprisk[\radius^2][\reg]
    \geq
    \risk[\radius(\radius - \radius_\nsamples)][\reg]\,,
\end{equation}
which leads to our main results.

\paragraph{Extension: upper and lower bounds on the empirical robust risk.}
The proof that we sketched above actually shows that $\risk[\radius(\radius - \radius_n)][\reg]$ is a lower bound of $\emprisk[\radius^2][\reg]$. This proof technique also yields an upper bound by exchanging the roles of $\prob$ and $\empirical$.
\begin{theorem}\label{thm:sandwich}
    In the setting of either \cref{thm:informal-unreg}, \cref{thm:informal-unreg-weak} or \cref{thm:informal-reg} (with $\reg = 0$ or $\reg > 0$), with probability at least $1 - \thres$, it holds that
    \begin{equation}
        \forall \func \in \funcs,\quad
        \risk[\radius (\radius - \radius_\nsamples)][\reg]
        \leq 
        \emprisk[\radius^2][\reg]
        \leq 
        \risk[\radius(\radius + \radius_\nsamples)][\reg]\,,
    \end{equation}
    with $\radius_\nsamples = \bigoh \parens*{\sqrt \frac{1 + \log 1 / \thres}{\nsamples}}$.
\end{theorem}
This result shows how two robust objectives \wrt$\prob$ provide upper and lower bounds on the empirical robust risk, with only slight variations in the radius.
Furthermore, when the number of data points $\nsamples$ grows, both \WAedit{sides} of the bound converge to the same quantity $\risk[\radius^2][\reg]$. Hence our generalization bounds of the form \cref{eq:sketch-of-proof-end} are asymptotically tight.

\WAedit{
    As a final remark, we underline that the proofs of this theorem and of the previous ones rely on the cost being the squared Euclidean norm and the extension to more general cost functions is left as future work. In particular, the Laplace approximation of \cref{sec:approx-dual-generator} in the regularized case and the analysis of \cref{sec:criticalunreg} in the standard \ac{WDRO} case would need further work to accomodate general cost functions.
}

\section{Examples: parametric models}
\label{sec:examples}

Our main theorems \cref{thm:informal-unreg,thm:informal-unreg-weak,thm:informal-reg} involve a general class $\funcs$ of loss functions. 
We explain in this section how to instantiate our results in the important class of parametric models. We then illustrate this setting with logistic regression and linear regression in \cref{ex:logistic,ex:l2}.

Let us consider the class of functions of the form
\begin{equation}
    \label{eq:parametric-models}
    \funcs = \setdef{\sample\mapsto\func(\param, \sample)}{\param \in \params}
    \qquad 
    \text{ with }
    \func:\params \times \samples \longrightarrow \R
\end{equation}
where $\params$, the parameter space, is a subset of $\R^\dimsalt$ and $\samples$, the sample space, is a subset of $\R^\dims$.

For instance, this covers the case of linear models of the form $\func(\param, \sample) = \ell(\inner{\sample, \param})$ with $\ell$ a convex loss. This class of models is studied by \citet{shafieezadehabadehregularizationmasstransportation2019,chen2018robust} in a slightly different setting, where they obtain a closed form for the robust objective and then establish a generalization bound similar to~\cref{eq:thm-concentration-wdro-unreg-gen}. %

Let us show how to instantiate our theorems in the case of \cref{eq:parametric-models}.
\begin{itemize}
    \item If $\func$ is twice continuously differentiable on a neighborhood of $\params \times \samples$ with $\params$ and $\samples$ both compact, then \cref{assumption:funcs-simple} is immediately satisfied.
Therefore, \cref{thm:informal-reg} can be readily applied and its generalization guarantee hold.
\item As for \cref{assumption:gradient-funcs}, it is equivalent to, for all $\param \in \params$, $\prob \parens*{\grad_\sample \func(\param, \sample) \neq 0} > 0$.
    \WAedit{Thus disregarding the degenerate case of $\grad_\sample \func(\param, \sample)$ being null for $\prob$-almost every $\sample$ (\eg when the loss does not depend on $\sample$),} we are in the setting of \cref{thm:informal-unreg-weak}.
\item Satisfying \cref{assumption:add-geometric-simple}, needed for \cref{thm:informal-unreg}, requires some problem-dependent developments, \WAedit{see the examples below}.
Note though that the second item of \cref{assumption:add-geometric-simple} is implied by the parametric Morse-Bott property \citep{arbelNonConvex}; see \cref{sec:morse-bott}.%
\end{itemize}

\WAedit{
We discuss linear and non-linear examples of this framework. In light of the above, we focus our discussion on \cref{assumption:add-geometric-simple}.
We first present the examples of linear models, \cref{ex:logistic,ex:l2}, where the latter assumption is satisfied. %
We then consider several examples of nonlinear models: kernel regression (\cref{ex:kernelreg}), smooth neural networks (\cref{ex:nn}) and families of invertible mappings (\cref{ex:diffeo}).
In \cref{app:num}, we also provide numerical illustrations for linear models.
}
\begin{example}[Logistic Regression]\label{ex:logistic}
    For a training sample $(\point, \target) \in \R^\dimsalt \times \{-1,+1\}$, the logistic loss for a parameter $\param \in \R^\dimsalt$ is given by $\log \parens*{1 + e^{-\target \inner{\point, \param}}}$. It fits into our framework by defining $\func(\param, \sample) = \log \parens*{1 + e^{\inner{\sample, \param}}}$ with $\sample$ playing the role of $-\target \times \point$.
    We assume that $\params$ is a compact set that does not include the origin, and, for the sake of simplicity, we take $\samples$ as a closed Euclidean ball, \ie $\samples = \ball(0, \ballradius)$. \    %
    We are going to show that \cref{assumption:add-geometric-simple} is satisfied, and, for this, we need the following elements. 
    For any $\param$, the maximizer of $\func(\param, \cdot)$ over $\samples = \ball(0, \ballradius)$ is reached at $\optsample \defeq \frac{\ballradius \param}{\norm{\param}}$. Besides, for any $\sample \in \samples$, it holds that
    \begin{equation}
        \ballradius^2
        \geq
        \sqnorm{\sample} 
        = \sqnorm{\optsample} + 2 \inner{\optsample, \sample - \optsample} +
        \sqnorm{\sample - \optsample} \, ,
    \end{equation}
    so that, since $\norm{\optsample} = \ballradius$, we have
    \begin{equation}
        \label{eq:ex-logistic-scalar-norm}
         \inner{\optsample, \optsample - \sample}
         \geq \half
        \sqnorm{\sample - \optsample}\,.
    \end{equation}
    We can now turn to the verification of \cref{assumption:add-geometric-simple}.
    \begin{enumerate}
        \item Take some $\GeoRad > 0$ and some $\sample \in \samples$ such that $\norm{\sample - \optsample} \geq \GeoRad$. Then, \cref{eq:ex-logistic-scalar-norm} yields
            \begin{equation}
                \label{eq:scalar}
                \inner{\param, \sample} - \inner{\param, \optsample}
                = \frac{\norm{\param}}{\ballradius} \inner{\optsample, \sample - \optsample}
                \leq -\frac{\norm{\param}}{2\ballradius} \sqnorm{\sample - \optsample}
                \leq -\frac{\distance(0, \params)\GeoRad^2}{2\ballradius}\,.
            \end{equation}
            Since $\scalar\mapsto\log(1+e^{\scalar})$ is increasing, this yields that $\func(\param, \sample) - \func(\param, \optsample)$ is bounded away from 0 by a negative constant uniformly in $\param$ .  The first item of \cref{assumption:add-geometric-simple} is thus satisfied.
        \item %
            Fix $\param \in \params$; by Taylor expanding  $\scalar\mapsto\log(1+e^{\scalar})$  around $\inner{\param, \optsample}$ we get 
            \begin{equation}
                \func(\param, \sample)
                =
                \func(\param, \optsample)
                +
                \frac{1}{1 + e^{-\inner{\param, \optsample}}}
                \inner{\param, \sample - \optsample}
                +
              \bigoh  \parens*{\inner{\param, \sample - \optsample}}^2 \,,
            \end{equation}
            where the big-O remainder is uniform over $\param \in \params$.
            Using the first inequality in \eqref{eq:scalar}, we get 
            for $\sample$ close enough to $\optsample$ %
\begin{align}
                \func(\param, \sample)
                &\leq 
                \func(\param, \optsample)
                -
                \frac{1}{2(1 + e^{-\inner{\param, \optsample}})}
                \inner{\param, \sample - \optsample}
                \leq 
                \func(\param, \optsample)
                 - \frac{\norm{\param}}{4\ballradius} \sqnorm{\sample - \optsample}
                \,.
            \end{align}
            This shows that the second item of \cref{assumption:add-geometric-simple} is satisfied locally around $\optsample$. It can be made global by using the uniform Lipschitz-continuity of $\func$, which introduces a term of the form $\frac{\plainhsmooth}{6} \norm{\sample - \optsample}^3$.
    \end{enumerate}
\end{example}

\begin{example}[Linear Regression]\label{ex:l2}
    With samples of the form $\sample = (\point, \target) \in \R^\dimsalt \times \R$ and parameters $\param \in \R^{\dimsalt}$, the loss is given by $ \func(\param, \sample) = \half (\inner{\param, \point} - \target)^2$. 
    Similarly to the previous example, we take $\params$ as a compact set of $\R^\dims$ that does not include the origin and $\samples$ of the form $\ball(0, \ballradius) \times [-\ballradiusalt, \ballradiusalt]$. 
    The maximizers of $\func(\param, \cdot)$ on $\samples$ are $\sample^*_1 = \parens*{\ballradius \param / \norm{\param}, -\ballradiusalt}$ and $\sample^*_2 = \parens*{-\ballradius \param / \norm{\param}, \ballradiusalt}$. 
    By symmetry, one can restrict to the case of $\sample^*_1$ and  $\inner{\param, \point} - \target\geq 0$; the same rationale as above can then be applied.
\end{example}

\WAedit{
    \begin{example}[Kernel Ridge Regression]\label{ex:kernelreg}
Using a kernel $k : \mathcal{X} \times \mathcal{X} \to \mathbb R$ with $\mathcal{X}$ compact and $k$ smooth, for instance Gaussian or polynomial,  we consider the following class of loss functions:
\begin{equation}f(\theta, \xi) = \frac{1}{2} \left(\sum_{i = 1}^m \alpha_i k(x, x_i) - y\right)^2 + \frac{\mu}{2} \|\alpha\|^2_2\,.\end{equation}
where $\sample = (x, y)$, $\samples$ is some compact subset of $\mathcal X \times \mathbb R$,  $\param=(\alpha_1,\dots,\alpha_m,x_1,\dots,x_m)$, $\params = A_m \times \mathcal X_m$, $m$ is a fixed integer, $A$ is a compact subset of $\mathbb R^m$, $\mathcal X_m$ can be any closed subset of $\mathcal X^m$ and $\mu \geq 0$ is the regularization parameter. A typical choice for $\mathcal X_m$ would be the datapoints of the training set.
This class then fits into our framework of parametric models above.
Finally, further information on the kernel would be needed to ensure that \cref{assumption:add-geometric-simple} is satisfied.

\end{example}

\begin{example}[Smooth Neural Networks]\label{ex:nn}
Denote by $\mathcal{NN}(x, \theta, \sigma)$ a multi-linear perceptron that takes $x$ as input, has weights $\theta$ and a smooth activation function $\sigma$, for instance the hyperbolic tangent or the Gaussian Error Linear Units (GELU). We choose $\ell(\hat y, y)$ a smooth loss function and we consider the loss $\func(\param, (x, y)) = \ell(\mathcal{NN}(x, \theta, \sigma), y)$ with $\param \in \Theta$ some compact set.
Provided that the inputs $(x, y)$ lie in a compact set $\Xi$, this class fits the parametric framework above.
Note that we require $\sigma$ to be smooth, further work would be required for non-smooth activation functions.
\end{example}

\begin{example}[Family of diffeomorphisms]\label{ex:diffeo}
    Consider maps $\funcaltalt \from \samples \to \samples$ and $(\param, \sample) \in\params \times \samples \mapsto \funcalt_\param(\sample) \in\samples$ and define the parametric loss $\func(\param, \sample)
    = \funcaltalt(\funcalt_\param(\sample))$. Assume that these functions are twice differentiable, that ${\funcaltalt}$ satisfies the second item of \cref{assumption:add-geometric-simple} and that, for every  $\param \in \params$, $\funcalt_\param$ has a inverse $\funcalt_\param^{-1}$ which is also continuously differentiable in a neighborhood of $\params \times \samples$.
    
    As before, this setting fits into the framework above. We now show that \cref{assumption:add-geometric-simple} is satisfied.
    \begin{enumerate}
        \item Since $\funcaltalt$ is continuous, $\funcaltalt$ satisfies the first item of \cref{assumption:add-geometric-simple}. It is satisfied by $\funcs$ as well thanks to $\funcalt_\param^{-1}$ being Lipschitz-continuous in $\sample$ uniformly in $\theta$ by compactness of $\params \times \samples$.
        \item 
 Take $C$ such that both $\funcalt_\param$ and $\funcalt_\param^{-1}$ are $C$-Lipschitz in $\sample$ uniformly in $\param$.
             Since $\argmax \func(\param, \cdot) = \funcalt_\param^{-1}(\argmax \funcaltalt)$, it holds that $\min_{\sol[\samplealt] \in \argmax \funcaltalt} \norm{\funcalt_\param(\sample) - \sol[\samplealt]}
                =
                \min_{\sol[\sample] \in \argmax \func(\param, \cdot)} \norm{\funcalt_\param(\sample) - \funcalt_\param(\sol[\sample])}$ which lies between ${C}^{-1} \min_{\sol[\sample] \in \argmax \func(\param, \cdot)} \norm{\sample - \sol[\sample]}$ and $C \min_{\sol[\sample] \in \argmax \func(\param, \cdot)} \norm{\sample - \sol[\sample]}$.
                Combined with $\funcaltalt$ satisfying the second item of \cref{assumption:add-geometric-simple}, this shows that $\func$ satisfies this condition as well.
    \end{enumerate}
\end{example}
}

\section{Conclusion and perspectives}
\label{sec:cl}

In this work,  we provide generalization guarantees for \ac{WDRO} models that improve over existing literature in the following aspects:
our results avoid the curse of dimensionality, provide exact upper bounds without spurious error terms, and allow for distribution shifts during testing.
We obtained these bounds through the development of an original concentration result on the dual of \ac{WDRO}.
\WAedit{Our framework is general enough to cover regularized versions of the \ac{WDRO} problem: they enjoy similar generalization guarantees as standard \ac{WDRO}, with less restrictive assumptions.}

Our work could be naturally extended in several ways. For instance, it might be possible to relax any of the assumptions (on the sample space, the sampling process, the Wasserstein metric, and the class of functions)
at the expense of additional technical work. \WAedit{Moreover, the crucial role played by the radius of the Wasserstein ball calls for a principled and efficient procedure to select it.
} %

%% file: intro_app.tex
We provide here the proofs of our main results \cref{thm:informal-reg,thm:informal-unreg,thm:informal-unreg-weak}, along with detailed versions that include explicit bounds.
We start, in \cref{sec:prelim}, by referencing preliminary results, reformulate some of our assumptions, and introduce quantities that appear in the final bounds.

As sketched in \cref{sec:outline}, our proof is built on two main parts.
The first part is to obtain a concentration bound on the dual problems \cref{eq:dualreg,eq:dualregtrue} by leveraging on a lower bound on the dual multiplier.
This concentration result is presented in \cref{sec:concentration} where we assume that such a lower-bound is given.
The second part of the proof then consists in establishing the lower-bound. We have to distinguish two cases: when $\radius$ is small (\cref{sec:small}) and when $\radius$ is close to the critical radius $\radius_c$ (\cref{sec:critical}). For the latter, we also need to treat separately the cases where the \ac{WDRO} problem is regularized or not (respectively \cref{sec:criticalreg} and \cref{sec:criticalunreg}): this is where the \cref{assumption:gradient-funcs,assumption:add-geometric-simple}, that are not required in the regularized case, come into play.
Putting together these two parts, we obtain our precise theorems in \cref{sec:proofcl} and show how they imply our main results.
\cref{sec:upperbounds} then complements our theorems to obtain \cref{thm:sandwich}.
Finally, some variations of known results and technical computations are compiled in \cref{app:lemmas} as standalone lemmas \WAedit{and, in \cref{app:num}, we provide numerical illustrations for linear models.}

\startcontents[app]{}
\printcontents[app]{l}{1}[2]{}

%% file: appendix.tex
\section{Preliminaries} \label{sec:prelim}
This section presents preliminary results before we start the proofs in \cref{sec:concentration}.
In the first part of this section \cref{sec:detailedhyp}, we present a weaker and more detailed version of \cref{assumption:funcs-simple}, namely \cref{assumption:funcs}, that will suffice for all the proofs in the appendix. We also introduce several quantities that will appear in the final bounds.
Then, in \cref{sec:detailedhyp} we recall the dual problems introduced in \cref{sec:outline} and justify that strong duality holds.
Preliminary approximation results on the dual are then given in \cref{sec:approx-dual-generator}.
We then proceed to show the relative compactness of the class $\funcs$ \wrt several metrics in \cref{sec:compactness}. These properties provide a convenient way of ensuring that quantities involving $\funcs$ are finite, \eg complexity measures or supremums over $\funcs$.
Finally, we introduce the so-called parametric Morse-Bott condition of \citet{arbelNonConvex} and show how it implies the second item of \cref{assumption:funcs-simple} in a Riemannian setting.

\subsection{Detailed assumption on the function class and important quantities}\label{sec:detailedhyp}

Here we present the precise assumptions that we will refer to in the proofs. 
While \cref{assumption:set,assumption:support_interior-simple} are used as presented in the main text, we slightly weaken \cref{assumption:funcs-simple} to \cref{assumption:funcs}. 
We also introduce some quantities that we will be of interest for the proofs and the final results.

\begin{assumption}[On the function class]\label{assumption:funcs}
    Consider $\funcs$ a set of real-valued non-negative continuous functions on $\samples$.
    We assume that:
    \begin{itemize}
        \item the functions $\func \in \funcs$ are uniformly $\smooth$-smooth;
        \item the gradients are uniformly bounded, \ie
        \begin{align}
            \vbound &\defeq \sup_{\func \in \funcs} \sup_{\sample \in \supp \prob} \norm{\nabla \func(\sample)}_2 < + \infty
        \end{align}
        \item when $\regparam=0$, the supremum in \eqref{eq:dual_gen} is finite, \ie 
        \begin{align}
            \bdcstalt(\dualvar) &\defeq \sup_{\func \in \funcs} \sup_{\sample \in \supp \prob} ~~\sup_{\argmax\{\func - \half[\dualvar]\sqnorm{\sample -  \cdot}\}} {\func} < + \infty \,.
        \end{align}
    \end{itemize}
\end{assumption}

Note that the non-negativity assumption is without loss of generality since otherwise, it suffices to consider $\widetilde \funcs \defeq \setdef{\func - \min \func}{\func \in \funcs}$ and our results are invariant by addition of a constant.

\newcommand{\blanketapp}{\cref{assumption:set,assumption:funcs,assumption:support_interior-simple}\xspace}%
\textbf{The blanket assumptions for the remaining of the appendix will be \blanketapp.}

The following finite quantities are relevant for the proofs and appear in the quantitative versions of \cref{thm:informal-unreg,thm:informal-unreg-weak,thm:informal-reg}.
\begin{align}
    \maxcost &\defeq \sup \setdef*{\half \sqnorm{\sample -  \sol[\sample]}}{\sample \in \supp \prob,\, \func \in \funcs,\, \sol[\sample] \in \argmax \func}\,, \\
    \Cost(\sdev) &\defeq \sup_{\sample \in \supp \prob} \ex_{\samplealt\sim\basecpl(\cdot | \sample)} \left[ \half \sqnorm{\sample -  \samplealt} \right]\,,\\
    \Cost_\funcs(\reg, \sdev) &\defeq \sup_{\func \in \conv(\funcs)} \sup_{\sample \in \supp \prob} \ex_{\samplealt \sim \modbase(\cdot | \sample)} \half \sqnorm{\sample -  \samplealt}\,,\\
   \text{and } ~~~ \Var(\reg,\sdev) &\defeq \sup_{\func \in \conv(\funcs)} \sup_{\sample \in \supp \prob}  \sup_{\dualvar \geq 0} \Var_{\samplealt \sim \modbase[\frac{\func - \dualvar \sqnorm{\sample -  \cdot}/2}{\reg}](\cdot | \sample)} \half \sqnorm{\sample -  \samplealt}\,, 
\end{align}
where  $\reg > 0$, $\sdev > 0$, and  $\basecpl$ is given by \cref{eq:baseconditional}.

\subsection{Strong duality}\label{sec:strong-duality}

As mentioned in \cref{sec:outline}, duality plays a central role in our proofs. 
Let us recall the central notion of dual generator functions, introduced in \eqref{eq:dual_gen}: for any $\func \in \funcs$, $\sample \in \samples$, $\dualvar, \reg \geq 0$ and $\sdev > 0$, the dual generator $\dualfunc$ is given as 
\begin{align}%
    \dualfunc(\func, \sample, \dualvar, \reg,\sdev)  &\defeq  \left\{ 
\begin{array}{ll}
    \sup_{\samplealt \in \samples} \braces*{\func(\samplealt) - \half[\dualvar] \sqnorm{\sample - \samplealt}} 
        &  \text{ if } \reg = 0\\
    \reg\log \parens*{\ex_{\samplealt \sim \basecpl(\cdot | \sample)} \exp\parens*{
            \frac{\func(\samplealt) - \dualvar \sqnorm{\sample - \samplealt}/2}{\reg}}} 
            &\text{ if } \reg > 0 \, .
\end{array}
    \right. 
\end{align}

Our proofs are based on the (strong) dual formulations of \ac{WDRO}, as given by the following lemma that summarizes results of the literature for the regularized and unregularized cases.
\begin{lemma}\label{lemma:strong-duality}
    Under the blanket assumptions,
    for $\func \in \funcs$, $\radius > 0$, $\reg \geq 0$ and $\sdev > 0$,
\begin{align}
    \emprisk[\radius^2][\reg] &= \inf_{\dualvar \geq 0} \dualvar \radius^2 + \ex_{\sample\sim\empirical} \left[ \dualfunc(\func, \sample, \dualvar, \reg,\sdev) \right]   \\
   \text{ and } ~~  \risk[\radius^2][\reg] &= \inf_{\dualvar \geq 0} \dualvar \radius^2 + \ex_{\sample\sim\prob} \left[ \dualfunc(\func, \sample, \dualvar, \reg,\sdev) \right]   \,.
\end{align}
\end{lemma}
\begin{proof}
See \citet{blanchet2019quantifying,gao2016distributionally} for the unregularized case and \citet{azizian2022regularization} for the regularized case.
\end{proof}

\subsection{Approximation of the dual generator $\dualfunc$}\label{sec:approx-dual-generator}

Important preliminary results for our upcoming concentration bounds (in \cref{sec:concentration}) are quantitative approximations of the dual generator $\dualfunc$, namely \cref{lemma:bound_dualfunc,lemma:approx_laplace_dualfunc}. In particular, these results also imply bounds on $\dualfunc$ in \cref{lemma:bound-rv}.

\begin{proposition}[Bounding the distance between $\dualfunc$ and $\func$]\label{lemma:bound_dualfunc}
    There are positive constants $\init[\dualvar]$, $\init[\reg]$, $\init[\sdev]$, $\cst{approx_phi_coeff}$, $\cst{approx_phi_rate}$ which depend on $\vbound$, $\GeoRad$ and $\dims$ such that taking some $\ubdualvar \geq \lbdualvar  \geq \init[\dualvar] + \smooth $, we have for any $\func \in \funcs$, $\sample \in \supp \prob$, $\dualvar \in [\lbdualvar,\ubdualvar]$, $\reg \in [0, \init[\reg] ] $ and $\sdev \in (0, \init[\sdev]]$
    \begin{align}
         \abs{{\dualfunc}(\func, \sample, \dualvar , \reg,\sdev) - \func(\sample)} \leq  
         \bdcst(\lbdualvar, \ubdualvar, \reg, \sdev)
    \end{align}
    where 
    \begin{align}
        \bdcst(\lbdualvar, \ubdualvar, \reg, \sdev)
       \defeq 
        \inv{2 \lbdualvar} \vbound^2 + \half[\reg \dims] \log \parens*{\frac{\ubdualvar}{\reg} + \inv{\sdev^2}}
        + \reg \log 2+ \reg \dims \abs{\log \sdev}+ 
    \reg\cst{approx_phi_coeff} e^{-\cst{approx_phi_rate}\parens*{\frac{\lbdualvar - \smooth}{\reg}}^{\third}}\,.
    \end{align}
\end{proposition}

The proof of this result is based on the following second approximation result which gives a precise approximation of $\dualfunc$ that will be used several times in the upcoming proof. 

More precisely, we want to approximate $\dualfunc$ by a Taylor development $\apprx{\dualfunc}$ defined for any $\func \in \funcs$, $\sample \in \supp \prob$, $\dualvar \geq 0$, $\reg \geq 0$ and $\sdev > 0$ as
\begin{align}\label{eq:phi_bar}
    \apprx{\dualfunc}(\func, \sample, \dualvar, \reg,\sdev) \defeq \func(\sample) + \inv{2 \left(\dualvar + \frac{\reg}{\sdev^2}\right)} \norm*{\grad \func(\sample)}^2_2 - \half[\reg \dims] \log \parens*{\frac{\dualvar}{\reg} + \inv{\sdev^2}}
    + \reg \log \frac{(2\pi)^{\half[\dims]}}{\partition}
\end{align}
where $\partition \defeq \int_{\samples} e^{-\frac{\norm{\sample - \samplealt}^2_2}{2 \sdev^2}} \dd \samplealt$.
The distance between $\dualfunc$ and  $\apprx{\dualfunc}$ is then controlled by the following Laplace approximation lemma.

\begin{lemma}[Approximation of $\dualfunc$]\label{lemma:approx_laplace_dualfunc}
    There are positive constants $\init[\dualvar], \init[\reg],\, \init[\sdev],\, \cst{approx_phi_coeff},\, \cst{approx_phi_rate}$ which depend on $\vbound,\, \GeoRad$ and $\dims$ such that $\init[\reg] \leq \init[\dualvar]$ and, when $\reg \in [0, \init[\reg] ] $, $\sdev \in (0, \init[\sdev]]$ and $\dualvar \geq \init[\dualvar] + \smooth$, we have  for any $\func \in \funcs$, $\sample \in \supp \prob$
    \begin{align}
        \apprx{\dualfunc}(\func, \sample, \dualvar + \smooth, \reg,\sdev) - \reg\cst{approx_phi_coeff} e^{-\cst{approx_phi_rate}\parens*{\frac{\dualvar + \smooth}{\reg}}^{\third}}
        \leq {\dualfunc}(\func, \sample, \dualvar , \reg,\sdev) \leq 
        \apprx{\dualfunc}(\func, \sample, \dualvar - \smooth, \reg,\sdev)
        +
\reg\cst{approx_phi_coeff} e^{-\cst{approx_phi_rate}\parens*{\frac{\dualvar - \smooth}{\reg}}^{\third}}\,.
    \end{align}
\end{lemma}

\begin{proof}
    Fix $\func \in \funcs$, $\sample \in \supp \prob$, $\dualvar \geq 0$, $\reg \geq 0$ and $\sdev > 0$.
    To bound the error between $\dualfunc$ and its approximation $\apprx{\dualfunc}$, we introduce an intermediate approximation $\widetilde{\dualfunc}$ defined as 
    \begin{align}
        \widetilde{\dualfunc}(\func, \sample, \dualvar, \reg,\sdev) \defeq
        \begin{cases}
    \reg\log \parens*{\ex_{\samplealt \sim \basecpl(\cdot | \sample)} \exp\parens*{
                \frac{\func(\sample) + \inner{\grad \func(\sample), \samplealt - \sample} - \dualvar \sqnorm{\sample -  \samplealt}/2}{\reg}
    }} &\text{if } \reg > 0\\
    \sup_{\samplealt \in \samples} [{\func(\sample) + \inner{\grad \func(\sample), \samplealt - \sample} - \half[\dualvar]\sqnorm{\sample -  \samplealt}}]
    &\text{if } \reg = 0\,,
        \end{cases}
    \end{align}
    which corresponds to $\dualfunc$ applied to the Taylor approximation of $\func$ at $\sample$ (instead of $\func$ itself).
    By smoothness of the functions in $\funcs$ (\cref{assumption:funcs}), we readily have that,
    \begin{align}
        \widetilde{\dualfunc}(\func, \sample, \dualvar + \smooth, \reg,\sdev)
        \leq {\dualfunc}(\func, \sample, \dualvar , \reg,\sdev) \leq 
        \widetilde{\dualfunc}(\func, \sample, \dualvar - \smooth, \reg,\sdev)\,.
    \end{align}
    Now, all that is left to bound, is the error between $\widetilde{\dualfunc}$ and $\apprx{\dualfunc}$.
    Consider first the case where $\reg > 0$ and let us rewrite $\widetilde{\dualfunc}$ by using the definition of $\basecpl$:
    \begin{align}
        \widetilde{\dualfunc}(\func, \sample, \dualvar, \reg,\sdev) &=
    \reg\log \parens*{\ex_{\samplealt \sim \basecpl(\cdot | \sample)} \exp\parens*{
                \frac{\func(\sample) + \inner{\grad \func(\sample), \samplealt - \sample} - \dualvar \sqnorm{\sample -  \samplealt}/2}{\reg}
            }}\\
    &=\reg\log \parens*{\int_{\samples} \exp\parens*{\frac{1}{\reg} \parens*{
            \func(\sample) + \inner{\grad \func(\sample), \samplealt - \sample} - \left({\dualvar} + \frac{\reg}{\sdev^2}\right) \half \sqnorm{\sample -  \samplealt}
        }} \dd \samplealt
        } - \reg \log \partition\,.
    \end{align}
    But, looking at the inner expression, we have that
    \begin{align}
{\func(\sample) + \inner{\grad \func(\sample), \samplealt - \sample} } - \left({\dualvar} + \frac{\sdev^2}{\reg}\right) \half \sqnorm{\sample -  \samplealt}
=
    \func(\sample) + \inv{2\parens{\dualvar + \frac{\sdev^2}{\reg}}}\norm{\grad \func(\sample)}^2_2 - \frac{1}{2 \regalt}\norm*{\samplealt - \sol[\samplealt](\regalt)}_2^2\,,
    \label{eq:proof-approx-rewriting}
    \end{align}
    where we defined $\inv{\regalt} \defeq \dualvar + \frac{\reg}{\sdev^2}$ and $\sol[\samplealt](\regalt) = \sample +\regalt\grad \func(\sample)$.
    Hence,
   \begin{align}
        \widetilde{\dualfunc}(\func, \sample, \dualvar, \reg,\sdev) 
        &= \apprx{\dualfunc}(\func, \sample, \dualvar, \reg,\sdev)  +
        \reg\log \parens*{\int_{\samples} \exp\parens*{
                 - \frac{1}{2\reg \regalt}\norm*{\samplealt - \sol[\samplealt](\regalt)}_2^2            } \dd \samplealt
             } - \half[\reg \dims]\log \parens*{2\pi\reg\regalt} %
             \,.
    \end{align}

    Define $\init[\dualvar]' \defeq \frac{\sqrt 6 \vbound}{\GeoRad}$ and $\init[\regalt] \defeq \inv{\init[\dualvar]'} = \frac{\GeoRad}{2 \vbound}$ so that $\dualvar\geq\init[\dualvar]' $ implies that $\regalt\leq\frac{1}{\init[\dualvar]'} = \init[\regalt]$. Let us now check that the conditions of \cref{lemma:approx_laplace} are satisfied.
    \begin{enumerate}
    \item Since $\sol[\samplealt](0) = \sample$, by \cref{assumption:support_interior}, $\ball(\sol[\samplealt](0), \GeoRad)$ is contained in $\samples$.
    \item For $\regalt \leq \init[\regalt]$, we have that $\regalt^2 \norm{\grad \func(\sample)}^2_2 \leq \frac{\vbound^2}{(\init[\dualvar]')^2} = \frac{\GeoRad^2}{6}$ by definition.
        \end{enumerate}
        Hence, we can apply \cref{lemma:approx_laplace} to get that, for any $\dualvar \geq  \init[\dualvar]'$

    \begin{align}
       \reg \log \parens*{ 1 - 
 6^{\dims / 2}e^{-\frac{\GeoRad^2}{12 \reg \regalt}}}
        \leq 
        \widetilde{\dualfunc}(\func, \sample, \dualvar, \reg,\sdev)
        - \apprx{\dualfunc}(\func, \sample, \dualvar, \reg,\sdev)
        \leq 
        \reg\log \parens*{ 1 + 
 6^{\dims / 2}e^{-\frac{\GeoRad^2}{12 \reg \regalt}}}\,.
    \end{align}
        Now, using \cref{lemma:log_ineq}, we get that there are positive constants $\init[\reg]$, $\init[\dualvar]$, $\newcst{approx_phi_coeff}$,$\newcst{approx_phi_rate}$ depending on $\GeoRad$, $\vbound$ and $\dims$ such that,
    if $\reg \leq \init[\reg]$ and $\dualvar \geq \init[\dualvar]$, then
\begin{align}
|\widetilde{\dualfunc}(\func, \sample, \dualvar, \reg,\sdev)
        - \apprx{\dualfunc}(\func, \sample, \dualvar, \reg,\sdev)|
        \leq \reg\cst{approx_phi_coeff} e^{-\cst{approx_phi_rate}\parens*{\frac{\dualvar}{\reg}}^{\third}}\,.
\end{align}
    Moreover, $\init[\reg]$ can be reduced so that it is less than $\init[\dualvar]$ if it is not the case originally.

    To finish the proof, let us now come back to the case $\reg = 0$.
    First, note that \cref{eq:proof-approx-rewriting} is still valid even with $\reg = 0$ so that we have
    \begin{align}
        \widetilde{\dualfunc}(\func, \sample, \dualvar, 0)
        =
        \apprx{\dualfunc}(\func, \sample, \dualvar, 0)
- \inv{2 \regalt} \inf_{\samplealt \in \samples}\norm*{\samplealt - \sol[\samplealt](\regalt)}_2^2 \,.
    \end{align}
    But as seen above,
 for $\regalt = \dualvar^{-1} \leq \init[\regalt]$, $\sol[\samplealt](\regalt)$ is inside $\ball(\sol[\samplealt](0), \GeoRad)$ so that 
$\widetilde{\dualfunc}(\func, \sample, \dualvar, 0) 
        = \apprx{\dualfunc}(\func, \sample, \dualvar, 0)$. 
        
We conclude the proof by noticing that the obtained bounds are valid for any $0 < \reg \leq \init[\reg] $, $\dualvar \geq \init[\dualvar] + \smooth$, $\func \in \funcs$ and $\sample \in \supp \prob$.
\end{proof}

The following lemma is needed for the proof of \cref{lemma:bound_dualfunc}.

\begin{lemma}\label{lemma:bd-normalization-cst}
    There is a positive constant $\init[\sdev] > 0$ which depends on $\GeoRad$ and $\dims$ such that, for $\sdev \in (0, \init[\sdev]]$ and $\sample \in \supp \prob$,
    \begin{equation}
        \absval*{\log \frac{\partition}{(2 \pi )^{\dims / 2}}} \leq \dims \abs{\log \sdev} + \log 2\,.
    \end{equation}
\end{lemma}

\begin{proof}
    It suffices to show that 
    \begin{equation}
        \label{eq:proof-bd-normalizing-cst-precise}
        \frac{(2 \pi \sdev^2)^{\dims /2}}{2} \leq  \partition \leq  (2 \pi \sdev^2)^{\dims /2}\,,
    \end{equation}
    for any $\sdev \in (0, \init[\sdev]]$ with some $\init[\sdev] > 0$ suitably defined.
    We prove the \ac{RHS} by removing the constraint $\samples$ in the integral defining $\partition$:
    \begin{equation}
    \partition = \int_{\samples} e^{-\frac{\norm{\sample - \samplealt}^2}{2 \sdev^2}} \dd \samplealt
    \leq 
    \int_{\R^\dims} e^{-\frac{\norm{\sample - \samplealt}^2}{2 \sdev^2}} \dd \samplealt = (2 \pi \sdev^2)^{\dims / 2}\,.
    \end{equation}
    For the \ac{LHS}, we invoke 
    \cref{lemma:approx_laplace} using \cref{assumption:support_interior} to get that
    $\partition \geq \sdev^{\dims}$
    when $\sdev \leq \init[\sdev]$ with $\init[\sdev] > 0$ satisfying
    \begin{equation}
         {1 - 6^{\dims / 2}e^{-\frac{\GeoRad^2}{12 \init[\sdev]^2}}}
        \geq
        \half \,.
    \end{equation}
\end{proof}
We are now in a position to prove the main result of the section.

\begin{proof}[Proof of \Cref{lemma:bound_dualfunc}]
    Applying \cref{lemma:approx_laplace_dualfunc,lemma:bd-normalization-cst} and using the definition of $\vbound$ readily gives us that
        \begin{align}
             \abs{{\dualfunc}(\func, \sample, \dualvar , \reg,\sdev) - \func(\sample)} \leq  
        \inv{2 \dualvar} \vbound^2 + \half[\reg \dims] \abs{\log \parens*{\frac{\dualvar}{\reg} + \inv{\sdev^2}}}
        + \reg (\dims \absval{\log \sdev} + \log 2) + \reg \ubpartition + 
    \reg\cst{approx_phi_coeff} e^{-\cst{approx_phi_rate}\parens*{\frac{\dualvar - \smooth}{\reg}}^{\third}}\,.
        \end{align}
        Since $\dualvar$ is always greater or equal than $\reg$, ${\log \parens*{\frac{\dualvar}{\reg} + \inv{\sdev^2}}}$ is always non-negative and, with $\dualvar$ belonging to $[\lbdualvar, \ubdualvar]$, we get that
        \begin{align}
             \abs{{\dualfunc}(\func, \sample, \dualvar , \reg,\sdev) - \func(\sample)} \leq  
        \inv{2 \lbdualvar} \vbound^2 + \half[\reg \dims] \log \parens*{\frac{\ubdualvar}{\reg} + \inv{\sdev^2}}
            +  \reg (\dims \absval{\log \sdev} + \log 2)+ 
    \reg\cst{approx_phi_coeff} e^{-\cst{approx_phi_rate}\parens*{\frac{\lbdualvar - \smooth}{\reg}}^{\third}}\,,
        \end{align}
        which is the desired result.
    \end{proof}

    As a consequence of this result, we have the following bound on the dual generator.

    \begin{corollary}\label{lemma:bound-rv}
        For any $\func \in \funcs$, $\sample \in \supp \prob$, $\dualvar \in [\lbdualvar,\ubdualvar] $, $\reg \geq 0$ and $\sdev > 0$, the bound
        \begin{align}
            -\lbcst(\lbdualvar, \ubdualvar, \reg, \sdev)
            \leq 
            \dualfunc(\func, \sample, \dualvar, \reg,\sdev)
            \leq 
            \bdcstalt(\lbdualvar)\,,
        \end{align}
        holds where 
        \begin{align}
            \lbcst(\lbdualvar, \ubdualvar, \reg, \sdev) &\defeq
            \begin{cases}
                {0} &\text{when } \reg = 0\\
                { \bdcst(\lbdualvar, \ubdualvar, \reg, \sdev)}  &\text{when $0< \reg \leq \init[\reg]$, $0 < \sdev \leq  \init[\sdev]$, $\lbdualvar \geq \init[\dualvar] + \smooth$}\\
                 \vbound \sqrt{2\Cost(\sdev)} + \left({\smooth}+\ubdualvar\right) \Cost(\sdev)&\text{otherwise.}
            \end{cases}
        \end{align}
        with $\bdcst(\lbdualvar, \ubdualvar, \reg, \sdev)$ the bounding term appearing in \cref{lemma:bound_dualfunc}, as well as  $\init[\reg]$, $\init[\sdev]$, $\init[\dualvar]$.    
    \end{corollary}
    \begin{proof}
        For the upper-bound, it suffices to note that
        \begin{align}
           \dualfunc(\func, \sample, \dualvar, \reg,\sdev) &\leq \dualfunc(\func, \sample, \dualvar, 0) \leq \dualfunc(\func, \sample, \lbdualvar, 0) \leq \bdcstalt(\lbdualvar)
        \end{align}
        by definition of $\bdcstalt(\lbdualvar)$.
        Let us now turn to the lower bound. 
        
        \noindent When $\reg  = 0$, we have that $\dualfunc(\func, \sample, \dualvar, \reg,\sdev)\geq \func(\sample) \geq 0$.

        \noindent When $0< \reg \leq \init[\reg]$ and $\lbdualvar \geq \init[\dualvar] + \smooth$, we have from \cref{lemma:bound_dualfunc}
     \begin{align}
         \dualfunc(\func, \sample, \dualvar, \reg,\sdev) 
         &\geq \func(\sample) -  \bdcst(\lbdualvar, \ubdualvar, \reg, \sdev) \geq -  \bdcst(\lbdualvar, \ubdualvar, \reg, \sdev)\,.
     \end{align}

     \noindent  Otherwise, the  bound comes from the smoothness of $\func$ and Jensen's inequality as
     \begin{align}
         \dualfunc(\func, \sample, \dualvar, \reg,\sdev) &\geq  \reg\log \parens*{\ex_{\samplealt \sim \basecpl(\cdot | \sample)} \exp\parens*{
            \frac{\func(\sample) + \inner{\grad \func(\sample), \samplealt - \sample} - (\smooth + \dualvar) \sqnorm{\sample -  \samplealt}/2}{\reg} }} \\
         &\geq 
      \ex_{\samplealt \sim \basecpl(\cdot | \sample)}[
    \func(\sample) + \inner{\grad \func(\sample), \samplealt - \sample} - (\smooth + \dualvar) \half \sqnorm{\sample -  \samplealt}] \\
     &\geq  -\parens*{{\vbound \sqrt{2\Cost(\sdev)}} + \left({\smooth}+ \dualvar\right) \Cost(\sdev)} \geq  -\parens*{{\vbound \sqrt{2\Cost(\sdev)}} + \left({\smooth}+ \ubdualvar\right) \Cost(\sdev)}\,.
     \end{align}
    \end{proof}

\subsection{Relative compactness of the class $\funcs$ of loss functions}
\label{sec:compactness}

In this section we prove the relative compactness of the class $\funcs$ \wrt several metrics. First, we show in \cref{lemma:funcs-compact} that, under our blanket assumptions, $\funcs$ is relatively compact for for the infinity norm over $\samples$, defined by \ie $\norm{\func}_\infty \defeq \sup_{\sample \in \samples} \abs{\func(\sample)}$.
Then, in \cref{lemma:Hausdorff-compactness}, we establish the equivalence between the first item of \cref{assumption:add-geometric-simple} and the relative compactness of $\funcs$ \wrt another distance that we introduce, as mentioned below \cref{assumption:add-geometric-simple} in \cref{sec:main-unreg}.
Finally, we leverage these compactness properties to ensure that the Dudley integral of $\funcs$ \wrt those metrics, a standard complexity measure in concentration theory, is finite in \cref{lemma:dudley-finite}.
\begin{lemma}\label{lemma:funcs-compact}
    $\funcs$ and $\conv(\funcs)$ are relatively compact for the topology of the uniform convergence.
\end{lemma}

\begin{proof}
    First, the functions of $\funcs$ are uniformly Lipschitz-continuous: fix $\sample \in \supp \prob$, then, for any $\samplealt \in \samples$ $\norm{\grad \func(\samplealt)} \leq \smooth \norm{\sample - \samplealt} + \vbound \leq \smooth \sup_{(\sample, \samplealt) \in \samples}\norm{\sample - \samplealt} + \vbound$ which is finite by compactness. Using the compactness of $\samples$ again, the functions in $\funcs$ are also uniformly bounded. As a consequence, the functions in $\conv(\funcs)$ are also uniformly Lipschitz-continuous and uniformly bounded. By the Arzelà-Ascoli theorem, see \eg\citep[Thm.~11.28]{rudinrealcomplexanalysis1987}, $\funcs$ and $\conv(\funcs)$ are then relatively compact for the topology of uniform convergence.
\end{proof}

Recall that, for a set $A \subset \samples$ and a point $\sample \in \samples$, we denote by $\distance(\sample, A)$ the distance between $\sample$ and $A$, \ie $\distance(\sample, A) = \inf_{\samplealt \in A} \norm{\sample - \samplealt}$.
\begin{lemma}\label{lemma:Hausdorff-compactness}
    Consider the distance, defined on continuous functions on $\samples$ by
    \begin{equation}
        \distfunc(\func, \funcalt) \defeq \norm{\func - \funcalt}_{\infty} + \Hdist(\argmax \func, \argmax \funcalt)
    \end{equation}
    where $\Hdist$ denotes the Hausdorff distance between sets associated to $\distance$, \ie for $A, B \subset \samples$,
    \begin{align}
        \Hdist(A, B) &\defeq \max \parens*{\sup_{\sample \in A} \distance(\sample, A), \sup_{\sample \in B} \distance(\sample, B)}\,.
    \end{align}

    Under the blanket assumptions, we have that 
     \cref{item:criterion-Hausdorff-compactness} of \cref{assumption:add-geometric}, \ie that
        for any $\GeoRad > 0$, there exists $\Margin > 0$ such that,
        \begin{align}\label{eq:Hausdorff-compactness-criterion}
            \forall \func \in \funcs,\, \forall \samplealt \in \samples,\,
            \distance(\samplealt, \argmax \func) \geq \GeoRad \implies \func(\samplealt) - \max \func \leq - \Margin\,,
        \end{align}
    is equivalent to $\funcs$ being relatively compact for $\distfunc$.
\end{lemma}
\begin{proof}
    \begin{itemize}
        \item[$(\implies)$]
            Let us begin by showing that \cref{eq:Hausdorff-compactness-criterion} implies the relative compactness of $\funcs$ for $\distfunc$, \ie that the adherence of $\funcs$ is compact for $\distfunc$.

            Take $(\func_\run)_{\run = \running}$ a sequence of functions from $\funcs$, and we will show that there is a subsequence which converges to some function in $\funcs$ for $\distfunc$.
            By compactness of $\funcs$ for the infinity norm, \cref{lemma:funcs-compact}, there readily is a subsequence of $(\func_\run)_{\run = \running}$ that converges uniformly to some continuous function $\func: \samples \to \R$. Without loss of generality, let us assume that the whole sequence $(\func_\run)_{\run = \running}$ converges uniformly to $\func$, \ie that $\norm{\func - \func_\run}_\infty \to 0$ as $\run \to + \infty$. As a consequence, it holds also holds that $\max_\samples \func_\run$ converges to $\max_{\samples} \func$.

            We now show that $\Hdist(\argmax \func_\run, \argmax \func)$ converges to 0. $\funcs$ satisfy \cref{eq:Hausdorff-compactness-criterion} by assumption. Hence, for any fixed $\sml > 0$, we can invoke \cref{eq:Hausdorff-compactness-criterion} with $\GeoRad \gets \sml$ and it gives us some $\Margin > 0$.
            Now, since $\func$ is continuous, $\setdef{\samplealt \in \samples}{\distance(\samplealt, \argmax \func) \geq \sml}$ is a closed set inside a compact and therefore is compact as well. Hence, $\func$ reaches its maximum over this set and it is strictly less than $\max_\samples \func$ by construction. Substituting $\Margin$ with $\min \parens*{\Margin, \max_{\samples} \func - \max \setdef{\func(\samplealt)}{\samplealt \in \samples,\,\distance(\samplealt, \argmax \func) \geq \sml}}$ which is still positive, we get that, for any $\samplealt \in \samples$, both,
\begin{equation}
     \distance(\samplealt, \argmax \func) \geq \sml \implies \func(\samplealt) - \max \func \leq - \Margin\,,
\end{equation}
and, for any $\run = \running$,
\begin{equation}
     \distance(\samplealt, \argmax \func_\run) \geq \sml \implies \func_\run(\samplealt) - \max \func_\run \leq - \Margin\,.
\end{equation}

            By convergence of the sequence, as mentioned above, there is some $\nRuns \geq 1$ such that, for any $\run \geq \nRuns$, $\norm{\func - \func_\run}_\infty \leq \Margin / 3$ and $\abs{\max_\samples \func_\run - \max_\samples \func} \leq \Margin / 3$.
            These two inequalities imply that, for any $\sample \in \argmax \func$,
            \begin{align}
              \max_\samples \func_\run -  \func_\run(\sample)
              \leq 
              \max_\samples \func + \frac{\Margin}{3} - \func(\sample) + \frac{\Margin}{3} = \frac{2 \Margin}{3}\,.
            \end{align}\
            Therefore, by definition of $\Margin$, it holds that $\distance(\sample, \argmax \func_\run) < \sml$.
            Similarly, when $\sample \in \argmax \func_\run$, one shows that $ \max_\samples \func -  \func(\sample) \leq {2 \Margin}/{3}$ so that we have $\distance(\sample, \argmax \func) < \sml$ as well. Hence, for any $\run \geq \nRuns$, $\Hdist(\argmax \func_\run, \argmax \func)$ is at most $\sml$.

            Therefore, we have shown that $\Hdist(\argmax \func_\run, \argmax \func)$
            goes to zero. Since $\norm{\func - \func_\run}_{\infty}$ converges to zero as well by construction, this means that $\distfunc(\func_\run, \func)$ converges to zero, which concludes the proof.%
        \item[$(\impliedby)$]
            Let us proceed by contradiction, \ie assume that there is some $\GeoRad > 0$, some sequence $(\func_\run)_{\run = \running}$ of functions from $\funcs$ and some sequence $(\sample_\run)_{\run = \running}$ of points from $\samples$ such that,
            \begin{equation}
                \forall \run = \running,\, \distance(\argmax \func_\run, \sample_\run) \geq \GeoRad\quad\text{ yet }\quad 
                \func_\run(\sample_\run) - \max_\samples \func_\run \to 0\text{ as }\run \to +\infty\,.
            \end{equation}
            Since $\samples$ is compact and since we assume $\funcs$ to be relatively compact for $\distfunc$, without loss of generality, we can assume that $(\sample_\run)_{\run=\running}$ converges to some $\sample \in \samples$ while $(\func_\run)_{\run = \running}$ converges to some continuous function $\func$ for $\distfunc$.
            On the one hand, by definition of the Hausdorff distance, we have that, for any $\run = \running$,
            \begin{align}
                \distance(\argmax \func, \sample)
                &\geq
                \distance(\argmax \func_\run, \sample) - \Hdist(\argmax \func_\run, \argmax \func)\\
                &\geq
                \distance(\argmax \func_\run, \sample_\run) - (\distance(\sample, \sample_\run) + \Hdist(\argmax \func_\run, \argmax \func))\,,
            \end{align}
            so that, by taking $\run \to +\infty$, we get that $\distance(\argmax \func, \sample) \geq \GeoRad$.
            On the other hand, by uniform convergence, 
            one has that
            \begin{equation}
                \func(\sample) - \max \func = \lim_{\run \to +\infty} \func_\run(\sample_\run) - \max \func_\run = 0\,,
            \end{equation}
            which yields the contradiction since $\sample$ cannot belong to $\argmax \func$.
    \end{itemize}
\end{proof}

Note that, for parametric models (\cref{sec:examples}), this lemma  gives a computation-free approach to verifying the second item of \cref{assumption:add-geometric-simple}.

\begin{corollary}
    Consider $\params$ a compact subset of $\R^\dimsalt$ and $\func : \params \times \samples \to \R$ a continuous function. If the map $\param \in \params \mapsto \func(\param, \cdot)$ is continuous from $\params$ to the space of continuous functions on $\samples$ equipped with the distance $\distfunc$ defined in \cref{lemma:Hausdorff-compactness}, then $\funcs \defeq \setdef{\func(\param, \cdot)}{\param \in \params}$ is compact for $\distfunc$. 
\end{corollary}

In particular, this corollary allows one to easily check that \cref{ex:logistic,ex:l2} satisfy the second item of \cref{assumption:add-geometric-simple}.

We finally introduce Dudley's integral, which is a standard complexity measure in concentration theory.
\label{sec:dudley}
\begin{definition}\label{def:dudley}
    Dudley's entropy integral $\dudley[\pspace, \dist]$ is defined for a metric space $(\pspace, \dist)$ as
        \begin{align}
            \dudley[\pspace, \dist] \defeq \int_{0}^{+\infty} \sqrt{ \log N(t, \pspace, \dist)} \dd t
        \end{align}
        where $N(t, \pspace, \dist)$ denotes the $t$-packing number of $\pspace$, which is the maximal number of points in $\pspace$ which are at least at a distance $t$ from each other.
\end{definition}

\begin{lemma}\label{lemma:dudley-finite}
    The Dudley integral of $\funcs$ \wrt $\norm{\cdot}_\infty$, that we denote by $\dudley$, is finite.
    Under \cref{assumption:add-geometric}, the Dudley integral of $\funcs$ \wrt $\distfunc$, denoted by $\dudley[\funcs, \distfunc]$ is finite as well.
\end{lemma}
\begin{proof}
    \Cref{lemma:funcs-compact} shows that $\funcs$ is relatively compact for the norm $\norm{\cdot}_\infty$ and in particular bounded. Since Dudley's entropy integral is finite for balls \citep[Ex.~5.18]{wainwright2019high} and $\funcs$ is now included in some ball for $\norm{\cdot}_\infty$, the integral $\dudley$ is indeed finite.
    The second assertion is proven using the same reasoning and \cref{lemma:Hausdorff-compactness}.
\end{proof}

\subsection{Parametric Morse-Bott objectives}\label{sec:morse-bott}

In this section, we discuss the quadratic growth condition of the second item of \cref{assumption:add-geometric-simple} and its relation to the parametric Morse-Bott assumption of \citet{arbelNonConvex}.
Indeed, in the context of smooth manifolds and parametric models, we prove that the parametric Morse-Bott assumption implies  the quadratic growth condition of \cref{assumption:add-geometric-simple}.
In \cref{assumption:morse-bott}, we introduce the Riemannian and parametric settings that are necessary to formulate the parametric Morse-Bott condition and we then present a version of this condition adapted to our context.
We refer to \citet{lee2018introduction} for definitions relevant to Riemannian geometry.
The main result of this section is then \cref{prop:morse-bott}, which relies on \cref{lemma:morse-bott} for its proof.
\begin{assumption}[Parametric Morse-Bott]\label{assumption:morse-bott}
    Let $\funcs = \setdef{\sample\mapsto\func(\param, \sample)}{\param \in \params}$ where :
    \begin{itemize}
        \item $\samples$, $\params$ are smooth compact (connected embedded) submanifolds of $\R^\dims$ and $\R^\dimsalt$ respectively, endowed with the induced Euclidean metric.
        \item $\func : \params \times \samples \to \R$ is thrice continuously differentiable on the product manifold.
        \item $\func$ is a parametric Morse-Bott function \citep[Def.~2]{arbelNonConvex}: the set of augmented critical points of $\funcs$, defined as
            \begin{equation}
                \mfld \defeq \setdef{(\param, \sample) \in \params \times \samples}{\Rgrad_\sample \func(\param, \sample) = 0}\,.
            \end{equation}
            is a \WAedit{smooth} (embedded) submanifold of \WAedit{$\params \times \samples \setminus \bd \samples$} whose dimension at $(\param, \sample) \in \mfld$ is 
            $\dim_{\param}(\params) + \dim( \ker \RHess_\sample \func(\param, \sample))$\,.
    \end{itemize}
\end{assumption}

Under this assumption \cref{assumption:morse-bott}, the following result thus guarantees that the quadratic growth condition of \cref{assumption:add-geometric-simple} holds.

\begin{proposition}\label{prop:morse-bott}
    Under \cref{assumption:morse-bott} and the first item of \cref{assumption:add-geometric-simple}, the second item of \cref{assumption:add-geometric-simple} holds, \ie
        there exists $\strongOpt, \hsmooth > 0$ such that, for all $\param \in \params$, $\sample \in \samples$ and $\optsample \in \argmax \func$ a projection of $\sample$ on $\argmax \func$, \ie $\optsample \in \argmin_{\argmax \func} \norm{\sample - \cdot}$, it holds that
        \begin{align}
            \func(\param, \optsample) \geq \func(\param, \sample) + \half[\strongOpt] \norm{\sample - \optsample}^2 - \frac{\hsmooth}{6} \norm{\sample - \optsample}^3\,.
        \end{align}
\end{proposition}

To show this result, we rely on the following lemma that relates \cref{assumption:morse-bott} to a local quadratic growth condition.
\begin{lemma}\label{lemma:morse-bott}
    Under \cref{assumption:morse-bott}, for any $(\base[\param], \base[\sample]) \in \mfld$ such that $\base[\sample]$ is a local maximum of $\func(\base[\param], \cdot)$ and any neighborhood $\nbdaltalt$ of $(\base[\param], \base[\sample])$ in $\mfld$, there exists a neighborhood $\nbd$ of $(\base[\param], \base[\sample])$ in $\params \times \sample$ and $\strongOpt > 0$ such that, for any $(\param, \sample) \in \nbd$, there exists $\optsample \in \samples$ such that $(\param, \optsample) \in \nbdaltalt$ and
 \begin{align}
     \func(\param, \optsample) \geq \func(\param, \sample) + \half[\strongOpt] \sqnorm{\sample - \optsample}
\,.
        \end{align}
\end{lemma}
\begin{proof}
    By assumption, the tangent space of $\mfld$ at $(\param, \sample)$ is given by
    \begin{equation}
        \tspace_{(\param, \sample)} \mfld
        = \tspace_\param \params \times \ker \RHess_\sample \func(\param, \sample)\,,
    \end{equation}
    and so its normal space (in $\params \times \samples$) is equal to
    \begin{equation}
        \nspace_{(\param, \sample)} \mfld
        = \braces{0} \times \parens*{\ker \RHess_\sample \func(\param, \sample)}^\bot \subset \tspace_\param \params \times \tspace_\sample \samples\,.
    \end{equation}
    \WAedit{Applying the inverse function theorem to the normal exponential map $\normalexp: (\param, \sample, (0, \tvec)) \in N \mfld \mapsto (\param, \Rexp_\sample(\tvec))$ following the proof of \citet[Thm.~5.25]{lee2018introduction}}, there exists $\sml > 0$ and a neighborhood $\nbd_\sml$ of $\mfld$ in $\params \times \sample$ such that, with
    \begin{equation}
        \nbdalt_\sml \defeq \setdef*{(\param, \sample, \tvec)}{(\param, \sample) \in \mfld,\, \tvec \in \parens*{\ker \RHess_\sample \func(\param, \sample)}^\bot,\, \norm{\tvec} < \sml}\,,
    \end{equation}   
    \WAedit{the normal exponential map $\normalexp$} is a diffeomorphism from $\nbdalt_\sml$ to $\nbd_\sml$.
   Note that $\nbdalt_\sml$ is relatively compact and, as a consequence, the third derivative of $\time \in [0, 1] \mapsto \func(\param, \exp_\sample(\time\tvec))$ is a continuous function of $\time \in [0, 1]$ and $(\param, \sample, \tvec) \in \nbdalt_\sml$ and as a consequence is bounded uniformly by some constant $\hsmooth > 0$.
   Fix $(\base[\param], \base[\sample]) \in \mfld$ such that $\base[\sample]$ is a local maximum of $\func(\param, \cdot)$.
   Consider the map
   \begin{equation}
       \strongMap : (\param, \sample) \mapsto \inf 
       \setdef*{
           \inner*{\tvec, -\RHess_\sample \func(\param, \sample) \tvec}
       }{\tvec \in \parens*{\ker \RHess \func(\param, \sample)}^\bot,\, \norm{\tvec} = 1}\,.
   \end{equation}
   If $\strongMap(\base[\param], \base[\sample])$ is $ + \infty$, \ie if $\ker \RHess \func(\base[\param], \base[\sample])$ is equal to the whole $\tspace _{\base[\sample]} \samples$, then, since the dimension of a manifold is locally constant, there is a neighborhood of $(\base[\param], \base[\sample])$ in $\params \times \samples$ on which $\strongMap$ is identically equal to $+\infty$.
   Otherwise, if $\strongMap(\base[\param], \base[\sample])$ is finite, then it is positive by construction. Hence, the continuity of $\strongMap$ implies there is a positive constant $\strongOpt > 0$ and a neighborhood of $(\base[\param], \base[\sample])$ in $\params \times \samples$ on which $\strongMap$ is lower-bounded by $\strongOpt$.

   Hence, in both cases, there is $\strongOpt > 0$ and $\alt\nbdalt$ a neighborhood of $(\base[\param], \base[\sample])$ in $\params \times \samples$ such that $\strongMap$ is at least greater or equal to $\strongOpt$ on $\alt\nbdalt$. %
   Finally, take
   \begin{equation}
       \nbd \defeq \nbd_\sml \cap \normalexp \parens*{
       \setdef*{(\param, \sample, \tvec) \in \nbdalt_\sml}{(\param, \sample) \in \alt\nbdalt \cap \nbdaltalt,\, \norm{\tvec} < \frac{3 \strongOpt}{4 \hsmooth}}}\,.
   \end{equation}

   We are now in a position to prove the result. 
   Take $(\param, \sample) \in \nbd$. Since $\nbd$ is included in $\nbd_\sml$, there is some $\optsample \in \samples$ and $\tvec \in \parens*{\ker \RHess \func(\param, \sample)}^\bot$ such that $(\param, \optsample) \in \mfld \cap \alt\nbdalt \cap \nbdaltalt$, $\norm{\tvec} < \frac{3 \strongOpt}{4 \hsmooth}$ and $\Rexp_\optsample(\tvec) = \sample$.
   Let $\curve(\time) \defeq \Rexp_\optsample(\time \tvec)$ for $\time \in [0,1]$ denotes the geodesic curve going from $\optsample$ to $\sample$. Then, by the Taylor inequality applied to $\time \mapsto \func(\sample, \curve(\time))$ (see \citet[\S~5.9]{boumal2020intromanifolds}) and by definition of $\hsmooth$,   \begin{align}
       \func(\param, \sample) 
       \leq 
       &\func(\param, \optsample)
       + \inner{\Rgrad_\sample \func(\param, \optsample), \tvec}
       + \half \inner{\RHess_\sample \func(\param, \optsample) \tvec, \tvec}\\
       &+ \half \inner{\Rgrad_\sample \func(\param, \optsample), \curve''(0)}
       + \frac{\hsmooth}{6} \norm{\tvec}^3\,.
   \end{align}
   But $\curve''(\time)$ is null since $\curve$ is a geodesic and $\Rgrad_\sample \func(\param, \optsample)$ too by definition.
   Moreover, since $(\param, \optsample) \in \alt\nbdalt$ and $\tvec \in (\ker \RHess_\sample \func(\param, \optsample))^\bot$, the term $\inner{\RHess_\sample \func(\param, \optsample) \tvec, \tvec}$ is bounded by $- \strongOpt \norm{\tvec}^2$. But $\norm{\tvec}$ is also equal to $\norm{\sample - \optsample}$ by definition of $\tvec$ so we get,

   \begin{align}
       \func(\param, \sample) 
       \leq 
       &\func(\param, \optsample)
   - \half[\strongOpt] \sqnorm{\sample - \optsample}
   + \frac{\hsmooth}{6} \norm{\sample - \optsample}^3\\
       \leq 
       &\func(\param, \optsample)
       - \frac{\strongOpt}{4} \sqnorm{\sample - \optsample}\,,
   \end{align}
   since $\norm{\tvec} = \norm{\sample - \optsample} \leq \frac{3 \strongOpt}{4 \hsmooth}$,
   which gives the result.
\end{proof}

We are now ready to prove \cref{prop:morse-bott}.
\begin{proof}[Proof of \cref{prop:morse-bott}]
    We build upon the result of \cref{lemma:morse-bott}.
    Fix $(\base[\param], \base[\sample]) \in \mfld$ such that $\base[\sample]$ is a maximum of $\func(\base[\param], \cdot)$ and let $\ballradius > 0$ such that $\ball \parens*{(\base[\param], \base[\sample]), \ballradius} \cap \mfld$ is diffeomorphic to an Euclidean ball. Invoke the first item of \cref{assumption:add-geometric-simple}, with $\GeoRad \gets \ballradius / 2$ and let $\Margin > 0$ be the given positive quantity.
    Let $\nbd$, $\strongOpt$ be given by \cref{lemma:morse-bott} invoked with $\nbdaltalt \defeq \ball \parens*{(\base[\param], \base[\sample]), \half[\ballradius]} \cap \setdef*{(\param, \sample) \in \mfld}{\func(\param, \sample) > \max_{\samples} \func(\param, \cdot) - \Margin}$.

    Hence, for any $(\param, \sample) \in \nbd$, there is $\optsample \in \samples$ such that $(\param, \optsample) \in \mfld \cap \nbdaltalt$ and
 \begin{align}
     \func(\param, \optsample) \geq \func(\param, \sample) + \half[\strongOpt] \sqnorm{\sample - \optsample}\,.\label{eq:proof:prop:morse-bott-res-lemma}
        \end{align}
        But $(\param, \optsample)$ also satisfies ${\func(\param, \optsample) > \max_{\samples} \func(\param, \cdot) - \Margin}$ so that $\distance(\optsample, \argmax_\samples \func(\param, \cdot)) < \half[\ballradius]$ by definition of $\Margin$, \ie there exists $\optoptsample$ that is a maximizer of $\func(\param, \cdot)$ and that is at distance at most $< \half[\ballradius]$ from $\optsample$. But then both $\optsample$ and $\optoptsample$ belong to $\ball \parens*{(\base[\param], \base[\sample]), \ballradius}$ that is diffeomorphic to an Euclidean ball. Hence, since the derivative of $\func(\param, \cdot)$ is null on $\mfld$, $\func(\param, \optsample) = \func(\param, \optoptsample) = \max_{\samples} \func(\param, \cdot)$ so that $\optsample$ is a maximizer of $\func(\param, \cdot)$ too.
        Therefore, \cref{eq:proof:prop:morse-bott-res-lemma} becomes
 \begin{align}
     \max_{\samples} \func(\param, \cdot) = \func(\param, \optsample) \geq \func(\param, \sample) + \half[\strongOpt] \sqnorm{\sample - \optsample}
\geq \func(\param, \sample) + \half[\strongOpt] \distance^2\left(\sample, \argmax_\samples \func(\param, \cdot)\right)\,.
\end{align}

The final statement of the proposition follows by compactness and uniform Lipschitz-continuity of $\funcs$ (see the proof of \cref{lemma:funcs-compact}).
\end{proof}

\section{From empirical to true risk via duality} \label{sec:concentration}

In this part of the proof, our objective is to show that: if the dual variable $\dualvar$ in \eqref{eq:dualreg} can be bounded uniformly in $[\lbdualvar,\ubdualvar]$ with probability $1-\thres$, then we can concentrate the empirical expectation in \eqref{eq:dualreg} towards the one in \eqref{eq:dualregtrue}. 
The concentration error induces a loss in the radius, fortunately, captured by the variable $\minradius(\thres, \lbdualvar, \ubdualvar, \reg,\sdev)$ that we take as 
\begin{align}\label{eq:def-minradius}
        \minradius(\thres, \lbdualvar, \ubdualvar, \reg,\sdev) \defeq
        \frac{117}{\sqrt \nsamples \lbdualvar}
        \parens*{
             \dudley[\funcs, \norm{\cdot}_\infty]
            + \max \parens*{\bdcstalt(\lbdualvar), \lbcst(\lbdualvar, \ubdualvar, \reg, \sdev)}
            \parens*{
            1
    + \sqrt{ \log \frac{1}{\thres}}}}\,,
        \end{align}
        where $\dudley[\funcs, \norm{\cdot}_\infty]$ is the Dudley integral of $\funcs$ \wrt the infinity norm (\cref{def:dudley}), $\bdcstalt$ is defined in \cref{assumption:funcs}, and  $\lbcst(\lbdualvar, \ubdualvar, \reg, \sdev)$ is the bounding term appearing in \cref{lemma:bound-rv}.

    The main result of this part is \cref{prop:template_general}, stated below, and the remainder of the section will consist in proving it.

\begin{proposition}\label{prop:template_general}
    for $\radius > 0$, $\reg \geq 0$, $\sdev > 0$ and $\thres \in (0, 1)$, assume that there is some $0<\lbdualvar\leq\ubdualvar < +\infty$ such that, with probability at least $1 - \half[\thres]$,
    \begin{align}\label{condition_for_concentration}
        \forall \func \in \funcs,\quad
        \emprisk[\radius^2][\reg] &= \inf_{\lbdualvar \leq \dualvar \leq \ubdualvar} \dualvar \radius^2 + \ex_{\sample\sim\empirical} \left[ \dualfunc(\func, \sample, \dualvar, \reg,\sdev) \right] \,.
    \end{align}
    then, when $\radius^2 \geq \minradius(\thres, \lbdualvar, \ubdualvar, \reg,\sdev)$,  with probability $1 - \thres$, %
 \begin{align}
    \forall \func \in \funcs,\quad \emprisk[\radius^2][\reg] \geq \risk[\radius^2- \minradius(\thres, \lbdualvar, \ubdualvar, \reg,\sdev)][\reg] \,.
    \end{align}
\end{proposition}

The proof of this result mainly consists in verifying that under our standing assumptions, we can apply the concentration result presented in \cref{lemma:concentration_basic} in order to concentrate $\emprisk[\radius^2][\reg]$ towards $\risk[\radius^2][\reg]$ through their dual formulations.

We begin by showing that the dual generator \emph{divided by $\dualvar$} is Lipchitz continuous in $\func$ and in $\dualvar^{-1}$ (for convenience, we use the notation $\dualvaralt=\dualvar^{-1}$). 

\begin{lemma}\label{lemma:lipschitz}\label{lemma:lipschitz-dualvaralt}
    Fix some $\ubdualvar\geq \lbdualvar > 0$. For any $\sample \in \samples$, $\reg \geq 0$ and $\sdev > 0$ we have that
    \begin{itemize}
        \item[(a)] for any $\dualvar \in [\lbdualvar,\ubdualvar]$,  $\func\mapsto  \dualvar^{-1} \dualfunc(\func, \sample, \dualvar, \reg,\sdev)$ is $\lbdualvar^{-1}$-Lipschitz continuous \wrt the norm $\norm{\cdot}_\infty$;
        \item[(b)] for any $\func \in \funcs$, $\dualvaralt \mapsto \dualvaralt\dualfunc(\func, \sample, \dualvaralt^{-1}, \reg,\sdev)$ is $\max \parens*{\bdcstalt(\lbdualvar), \lbcst(\lbdualvar, \ubdualvar, \reg, \sdev)}$-Lipschitz continuous on $\bracks*{\ubdualvar^{-1}, \lbdualvar^{-1}}$.
    \end{itemize}
    \end{lemma}
    \begin{proof}
        \noindent\underline{Item (a).}
    When $\reg = 0$, $\func\mapsto\dualfunc(\func, \sample, \dualvar, \reg,\sdev)$ is a supremum of $1$-Lipschitz functions and is thus $1$-Lipschitz. For $\reg > 0$, take $\func,\, \funcalt \in \funcs$ and, for $t \in [0,1]$, define $\func_t = \func + t(\funcalt - \func)$. Differentiating $t \mapsto \dualfunc(\func_t, \sample, \dualvar, \reg,\sdev)$ yields
        \begin{align}
           \abs*{\frac{\dd}{\dd t} \dualfunc(\func_t, \sample, \dualvar, \reg,\sdev)} 
           = \abs*{ \frac{\ex_{\samplealt \sim \basecpl(\cdot | \sample)}[ (\funcalt(\samplealt) - \func(\samplealt)) e^\frac{\func_t(\samplealt) - \dualvar \sqnorm{\sample -  \samplealt}/2}{\reg}]}{\ex_{\samplealt' \sim \basecpl(\cdot | \sample)} [ e^\frac{\func_t(\samplealt) - \dualvar \sqnorm{\sample -  \samplealt'}/2}{\reg}] } } 
           \leq  \norm{\funcalt - \func}_{\infty}\,,
        \end{align}
            which gives that $\func\mapsto\dualfunc(\func, \sample, \dualvar, \reg,\sdev)$ is 1-Lipschitz continuous \wrt the norm $\norm{\cdot}_\infty$.
           
    Since this bound is uniform in $\dualvar$, we immediately get that $\func\mapsto\dualvar^{-1}\dualfunc(\func, \sample, \dualvar, \reg,\sdev)$ is $\lbdualvar^{-1}$-Lipschitz continuous for all $\dualvar \geq \lbdualvar$.

     \noindent\underline{Item (b).} Fix $\func \in \funcs$, $\sample \in \samples$, $\reg \geq 0$, $\sdev > 0$ and define $\funcalt(\dualvar) \defeq \dualvar \mapsto \dualvar^{-1}\dualfunc(\func, \sample, \dualvar, \reg,\sdev)$.

     Let us first begin with the case $\reg = 0$.
     Take $\dualvar, \dualvar' \in  [\lbdualvar,\ubdualvar] $. Without loss of generality, we can suppose that $ \funcalt(\dualvar) \geq \funcalt(\alt \dualvar)$.
     Since $\func$ is continuous and $\samples$ is a compact set, choose $\samplealt \in \argmax_{\samplealt\in\samples} \braces*{\func(\samplealt) - \half[\dualvar]\sqnorm{\sample -  \samplealt}}$.
     Then, the claim comes from the fact that
     \begin{align}
         0\leq \funcalt(\dualvar) - \funcalt(\alt \dualvar)
         \leq 
         \dualvar^{-1} \func(\samplealt) - \half \sqnorm{\sample -  \samplealt}
         -
         \parens*{
             \alt \dualvar^{-1} \func(\samplealt) - \half \sqnorm{\sample -  \samplealt}
         }%
         \leq 
         \abs{\dualvar^{-1} - \alt \dualvar^{-1}} \, \bdcstalt(\lbdualvar)  \,,
     \end{align}
     where we use that since $\func$ is non-negative by assumption, $\abs*{\bdcstalt\left(\dualvar\right)} ={\bdcstalt\left(\dualvar\right)} \leq \bdcstalt\left(\lbdualvar\right)$.
     
     Let us now turn to the case where $\reg > 0$, for which $\funcalt$ is differentiable on $[\lbdualvar,\ubdualvar]$ with derivative
     \begin{align}
         \funcalt'(\dualvar) = - \frac{1}{\dualvar^2} \dualfunc(\func, \sample, \dualvar, \reg,\sdev)
         - \frac{1}{\dualvar}  \frac{\ex_{\samplealt \sim \basecpl(\cdot | \sample)}[ \half \sqnorm{\sample -  \samplealt'} e^\frac{\func(\samplealt) - \dualvar \sqnorm{\sample -  \samplealt}/2}{\reg}]}{\ex_{\samplealt' \sim \basecpl(\cdot | \sample)} [ e^\frac{\func(\samplealt') - \dualvar \sqnorm{\sample -  \samplealt'}/2}{\reg}] }\,.
     \end{align}
     Since the claimed result is the Lipchitz continuity of $\funcaltalt: \dualvaralt \mapsto \funcalt( \dualvaralt^{-1} )$, it suffices to bound its derivative, \ie to bound $-\dualvar^{-2}  \funcalt'(\dualvar)$ for all $\ubdualvar \geq \dualvar \geq \lbdualvar$. 
     On the one hand, thanks to \cref{lemma:softmax_ineq}, it is bounded above as
     \begin{align}
        -\frac{1}{\dualvar^{2}} \funcalt'(\dualvar)
         &\leq   \frac{\ex_{\samplealt \sim \basecpl(\cdot | \sample)}[ (\func(\samplealt) - \half[\dualvar]\sqnorm{\sample -  \samplealt}) e^\frac{\func(\samplealt) - \dualvar \sqnorm{\sample -  \samplealt}/2}{\reg}]}{\ex_{\samplealt' \sim \basecpl(\cdot | \sample)} [ e^\frac{\func(\samplealt') - \dualvar \sqnorm{\sample -  \samplealt'}/2}{\reg}] }    + \dualvar  \frac{\ex_{\samplealt \sim \basecpl(\cdot | \sample)}[ \half \sqnorm{\sample -  \samplealt} e^\frac{\func(\samplealt) - \dualvar \sqnorm{\sample -  \samplealt}/2}{\reg}]}{\ex_{\samplealt' \sim \basecpl(\cdot | \sample)} [ e^\frac{\func(\samplealt') - \dualvar \sqnorm{\sample -  \samplealt'}/2}{\reg}] }\\
          &\leq \dualfunc(\func, \sample, \dualvar, 0)%
         \leq \bdcstalt(\dualvar) \leq \bdcstalt(\lbdualvar)\,.
     \end{align}

     On the other hand, invoking \cref{lemma:bound-rv} also yields that
          \begin{align}
        -\frac{1}{\dualvar^{2}}   \funcalt'(\dualvar) 
         &\geq    \reg\log \parens*{\ex_{\samplealt \sim \basecpl(\cdot | \sample)} e^{
            \frac{\func(\samplealt) - \dualvar \sqnorm{\sample -  \samplealt}/2}{\reg} }} 
            \geq - \lbcst(\lbdualvar, \ubdualvar, \reg, \sdev)\,,
     \end{align}
     which concludes the proof.
    \end{proof}

    We can now apply standard concentration for bounded Lipschitz quantities to bound the difference between the expectation of the dual generator over the empirical distribution $\empirical$ and true one $\prob$.

    \begin{lemma}\label{lemma:template_general_lemma}
        For $\radius > 0$, $\reg \geq 0$, $\sdev > 0$, $\thres \in (0, 1)$ and some  $0<\lbdualvar\leq\ubdualvar < +\infty$, we have with probability at least $1 - \half[\thres]$ that
            \begin{align}
                \sup_{(\func, \dualvar) \in \funcs \times [\lbdualvar, \ubdualvar]} \braces*{\frac{\ex_{\sample \sim \prob}[\dualfunc(\func, \sample, \dualvar, \reg,\sdev)] - \ex_{\sample \sim \emp}[\dualfunc(\func, \sample, \dualvar, \reg,\sdev)]}{\dualvar}}
            \leq 
            \minradius(\thres, \lbdualvar, \ubdualvar, \reg,\sdev)\,.
        \end{align}
    \end{lemma}
    \begin{proof}
        Our objective is to bound the quantity
        \begin{align}
            &\sup_{(\func, \dualvar) \in \funcs \times [\lbdualvar, \ubdualvar]} \braces*{\frac{\ex_{\sample \sim \prob}[\dualfunc(\func, \sample, \dualvar, \reg,\sdev)] - \ex_{\sample \sim \emp}[\dualfunc(\func, \emp, \dualvar, \reg,\sdev)]}{\dualvar}}\\
            &=
            \sup_{(\func, \dualvaralt) \in \funcs \times \bracks*{\ubdualvar^{-1}, \lbdualvar^{-1}}} \braces*{\ex_{\sample \sim \prob}[\dualvaralt\,\dualfunc(\func, \sample, \dualvaralt^{-1}, \reg,\sdev)] - \ex_{\sample \sim \emp}[\dualvaralt\,\dualfunc(\func, \sample, \dualvaralt^{-1}, \reg,\sdev)]}\\
            &=
            \sup_{(\func, \dualvaralt) \in \pspace}  \braces*{ \ex_{\sample \sim \prob}[\rv((\func, \dualvaralt), \sample)] - \ex_{\sample \sim \emp}[\rv((\func, \dualvaralt), \sample)]}\,,
        \end{align}
        where we used again the notation $\dualvaralt = \dualvar^{-1}$ and defined
        \begin{align}
            \pspace \defeq \funcs \times \bracks*{\ubdualvar^{-1}, \lbdualvar^{-1}} ~~ \text{ and } ~~ \rv((\func, \dualvaralt), \sample) \defeq \dualvaralt\, \dualfunc(\func, \sample, \dualvaralt^{-1}, \reg,\sdev)\,.
        \end{align}
        Let us endow $\pspace$ with the distance,
        \begin{align}
            \dist((\func, \dualvaralt), (\alt \func, \alt \dualvaralt)) \defeq \lbdualvar^{-1}
            \norm{\func - \alt \func}_{\infty}
            + \max \parens*{\bdcstalt(\lbdualvar), \lbcst(\lbdualvar, \ubdualvar, \reg, \sdev)} |\dualvaralt - \alt \dualvaralt|\,.
        \end{align}
        We now wish to apply \cref{lemma:concentration_basic} and check its three requirements:
        \begin{enumerate}
            \item For any $(\func, \dualvaralt) \in \funcs \times \bracks*{\ubdualvar^{-1}, \lbdualvar^{-1}}$, $\rv((\func, \dualvaralt), \cdot)$ is measurable since the functions of $\funcs$ are continuous and thus \emph{a fortiori} measurable;
            \item By \cref{lemma:lipschitz,lemma:lipschitz-dualvaralt}, for any $\reg \geq 0$ and any $\sample \in \supp \prob$, $\rv(\cdot, \sample)$ is 1-Lipschitz \wrt $ \dist$;
            \item Thanks to \cref{lemma:bound-rv}, for any $(\func, \dualvaralt) \in \pspace$, $\sample \in \supp \prob$, $\reg \geq 0$ and $\sdev > 0$, we have 
        \begin{align}
            - \frac{\lbcst(\lbdualvar, \ubdualvar, \reg, \sdev)}{\lbdualvar}
            \leq 
            \rv((\func, \dualvaralt), \sample)
            \leq 
            \frac{\bdcstalt(\lbdualvar)}{\lbdualvar}\,.
        \end{align}            
        \end{enumerate}
      
        As a consequence, applying statement $(b)$ of \cref{lemma:concentration_basic} yields that, with probability at least $1 - \half[\thres]$,
        \begin{align}
            &\sup_{(\func, \dualvar) \in \funcs \times [\lbdualvar, \ubdualvar]} \braces*{ \frac{\ex_{\sample \sim \prob}[\dualfunc(\func, \sample, \dualvar, \reg,\sdev)] - \ex_{\sample \sim \emp}[\dualfunc(\func, \emp, \dualvar, \reg,\sdev)]}{\dualvar} }\\
            &\leq 
    \frac{48\dudley[\pspace, \dist]}{\sqrt{\nsamples}}
    + \frac{2}{\lbdualvar}\left( {\bdcstalt(\lbdualvar)} + \lbcst(\lbdualvar, \ubdualvar, \reg, \sdev)\right) \sqrt{\frac{\log \frac{2}{\thres}}{2\nsamples}}\,.
        \end{align}
        We now proceed to bound $\dudley[\pspace, \dist]$. Exploiting the product space structure of $\pspace$ and $\dist$ with \cref{lemma:dudley-product-space}, one has that,
        \begin{align}
            \dudley[\pspace, \dist] &\leq 
            \lbdualvar^{-1} \dudley[\funcs, \norm{\cdot}_\infty] +\max \parens*{\bdcstalt(\lbdualvar), \lbcst(\lbdualvar, \ubdualvar, \reg, \sdev)} \dudley[{[0, \lbdualvar^{-1}], \abs{\cdot}}]\\
                                    &\leq  \lbdualvar^{-1} \parens*{\dudley[\funcs, \norm{\cdot}_\infty] + 
                                    \max \parens*{\bdcstalt(\lbdualvar), \lbcst(\lbdualvar, \ubdualvar, \reg, \sdev)}\half[1 + 2\log 2]}\,,
        \end{align}
        where we used \cref{lemma:dudley-segment}.
        Hence, we have shown that with probability at least $1 - \half[\thres]$,
            \begin{align}
                \sup_{(\func, \dualvar) \in \funcs \times [\lbdualvar, \ubdualvar]} \braces*{\frac{\ex_{\sample \sim \prob}[\dualfunc(\func, \sample, \dualvar, \reg,\sdev)] - \ex_{\sample \sim \emp}[\dualfunc(\func, \sample, \dualvar, \reg,\sdev)]}{\dualvar}}
            \leq 
            \minradius(\thres, \lbdualvar, \ubdualvar, \reg,\sdev)
        \end{align}
        where some numerical constants have been simplified.
    \end{proof}
    
    \begin{proof}[Proof of \Cref{prop:template_general}]
        Building on \cref{lemma:template_general_lemma}, we can now conclude the main result of this section. Using our boundedness assumption on $\dualvar$, we have that, with probability $1 - \thres$, the two following statements hold simultaneously
        \begin{align}
            &\bullet ~~  \forall \func \in \funcs,\quad
            \emprisk[\radius^2][\reg] = \inf_{\lbdualvar \leq \dualvar \leq \ubdualvar} \dualvar \radius^2 + \ex_{\sample\sim\empirical} \left[ \dualfunc(\func, \sample, \dualvar, \reg,\sdev) \right] \, ;\\
            &\bullet ~~ 
            \sup_{(\func, \dualvar) \in \funcs \times [\lbdualvar, \ubdualvar]} \braces*{\frac{\ex_{\sample \sim \prob}[\dualfunc(\func, \sample, \dualvar, \reg,\sdev)] - \ex_{\sample \sim \emp}[\dualfunc(\func, \sample, \dualvar, \reg,\sdev)]}{\dualvar}}
            \leq 
            \minradius(\thres, \lbdualvar, \ubdualvar, \reg,\sdev)\,.
        \end{align}
        As a consequence, on this event, for any $\func \in \funcs$,
        \begin{align}
            \emprisk[\radius^2][\reg] &= \inf_{\lbdualvar \leq \dualvar \leq \ubdualvar} \left\{ \dualvar \radius^2 + \ex_{\sample\sim\empirical} \left[ \dualfunc(\func, \sample, \dualvar, \reg,\sdev) \right] \right\}\\
                                            &= \inf_{\lbdualvar \leq \dualvar \leq \ubdualvar}\left\{ \dualvar \radius^2 + \ex_{\sample\sim\prob} \left[ \dualfunc(\func, \sample, \dualvar, \reg,\sdev) \right] - \dualvar \frac{\ex_{\sample \sim \prob}[\dualfunc(\func, \sample, \dualvar, \reg,\sdev)] - \ex_{\sample \sim \emp}[\dualfunc(\func, \sample, \dualvar, \reg,\sdev)]}{\dualvar} \right\}\\
                                            &\geq \inf_{\lbdualvar \leq \dualvar \leq \ubdualvar}\left\{ \dualvar \radius^2 + \ex_{\sample\sim\prob} \left[ \dualfunc(\func, \sample, \dualvar, \reg,\sdev) \right] - \dualvar \sup_{\lbdualvar \leq \alt\dualvar \leq \ubdualvar}\frac{\ex_{\sample \sim \prob}[\dualfunc(\func, \sample, \alt\dualvar, \reg,\sdev)] - \ex_{\sample \sim \emp}[\dualfunc(\func, \sample, \alt \dualvar, \reg,\sdev)]}{\alt\dualvar} \right\}\\
                                            &\geq \inf_{\lbdualvar \leq \dualvar \leq \ubdualvar}\left\{ \dualvar \radius^2 + \ex_{\sample\sim\prob} \left[ \dualfunc(\func, \sample, \dualvar, \reg,\sdev) \right] - \dualvar \minradius(\thres, \lbdualvar, \ubdualvar, \reg,\sdev)\right\} \\
                                            &\geq \risk[\radius^2- \minradius(\thres, \lbdualvar, \ubdualvar, \reg,\sdev)][\reg]
        \end{align}
        where $\radius^2- \minradius(\thres, \lbdualvar, \ubdualvar, \reg,\sdev) \geq 0$ by assumption.
    \end{proof}

    \begin{remark}\label{rmk:refined-template_general}
        Note that the proof of \cref{prop:template_general} actually gives us the slightly stronger result at the penultimate equation: with probability at least $1 - \thres$, for any $\func \in \funcs$,
     \begin{align}
        \emprisk[\radius^2][\reg] &\geq \inf_{\lbdualvar \leq \dualvar \leq \ubdualvar}\left\{ \dualvar ( \radius^2 -\minradius(\thres, \lbdualvar, \ubdualvar, \reg,\sdev) ) + \ex_{\sample\sim\prob} \left[ \dualfunc(\func, \sample, \dualvar, \reg,\sdev) \right] \right\}\,,
        \end{align}
        that we will require later.
    \end{remark}

\section{Dual bound when $\radius$ is small} \label{sec:small}

In this section, we show how the condition \eqref{condition_for_concentration} of \cref{prop:template_general} can be obtained when the robustness radius $\radius$ is small enough. The results of this section cover both the standard \ac{WDRO} setting of \cref{thm:informal-unreg,thm:informal-unreg-weak} and the regularized case of \cref{thm:informal-reg}.

In the following \cref{asm:rho_small}, we precise how small $\radius$ has to be; we also take $\reg$ and $\sdev$ proportional to $\radius$ in order to get close to the true risk with $\radius$, $\reg$ and $\sdev$ ``small'' at the same time. %
The main result of this section is \cref{prop:lb_dualvar_asymptotic}, whose proof relies on \cref{lemma:lb_dualvar_population_asymptotic}.

\begin{assumption}[$\radius$ is small]\label{asm:rho_small}
    Take $\reg = \basereg \radius$, $\sdev = \basesdev \radius $ with $\basereg \geq 0$, $\basesdev > 0$ and define
    \begin{align}
        \basedualvar &\defeq \basereg \dims + \sqrt{(\basereg \dims)^2 + 8 \inf_{\func \in \funcs}\ex_{\prob} \norm{\grad \func}^2_2} &
        \strong      &\defeq \frac{8 \inf_{\func \in \funcs}\ex_{\prob} \norm{\grad \func}^2_2}{(\basedualvar)^3} + \frac{2 \basereg \dims}{(\basedualvar^2)}\,.
\end{align}
Moreover, assume that $\basereg$ and $\basesdev$ satisfy
\begin{equation}
    \frac{\basereg}{\basesdev^2} \leq \frac{\basedualvar}{8}\,.
\end{equation}

    Assume that $\radius > 0$ is small enough so that,
    \begin{equation}
    \assumptionRadiusAsympt
    \end{equation}
where $\init[\dualvar], \init[\reg],\, \cst{approx_phi_coeff},\, \cst{approx_phi_rate}$ are positive constant given by \cref{lemma:approx_laplace_dualfunc} and $\sdev$ comes from \eqref{eq:baseconditional}.
\end{assumption}

Note that $\basedualvar$ and $\strong$ are both always positive, be it thanks to \cref{assumption:add-geometric} or the regularization with $\basereg > 0$.
For such values of $\radius$, the main result of this section \cref{prop:lb_dualvar_asymptotic} shows that the dual variable of \eqref{eq:dualreg} can be bounded with high probability.

\begin{proposition}\label{prop:lb_dualvar_asymptotic}
Let \cref{asm:rho_small} hold and fix a threshold $\thres \in (0, 1)$. 
Assume in addition that
    \begin{align}
        \radius \geq \frac{ 8192}{\sqrt{\nsamples} \, \strong (\basedualvar)^2} \parens*{
       12 \dudley
        + (\fbound +  \ub\bdcst(\radius)) \sqrt{1 + \log \frac{1}{\thres}}
        }
    \end{align}
    where $\dudley[\funcs, \norm{\cdot}_\infty]$, $\bdcstalt$ are defined in \cref{sec:detailedhyp} and  $\bdcst(\lbdualvar, \ubdualvar, \reg, \sdev)$, the bounding term appearing in \cref{lemma:bound_dualfunc}, is used to define
\begin{align}
    \ub\bdcst(\radius) \defeq \sup_{\alt \radius \in(0, \radius]}\bdcst \parens*{\max \parens*{\frac{\basedualvar}{32 \alt \radius},\init[\dualvar]+\smooth}, \frac{\basedualvar}{2 \alt \radius}, \basereg \alt \radius, \basesdev \alt \radius} .
\end{align}
    
Then, with probability at least $1 - \thres$, we have
\begin{align}
    \forall \func \in \funcs,\quad
        \emprisk[\radius^2][\reg] &= \inf_{\frac{\basedualvar}{32 \radius} \leq \dualvar \leq \frac{\basedualvar}{2 \radius}} \dualvar \radius^2 + \ex_{\sample\sim\empirical} \left[ \dualfunc(\func, \sample, \dualvar, \reg,\sdev) \right] \,.
\end{align}
\end{proposition}

To show \cref{prop:lb_dualvar_asymptotic}, we need the following helper lemma.

\begin{lemma}\label{lemma:lb_dualvar_population_asymptotic}
    Let \cref{asm:rho_small} hold.  Then,
    \begin{align}
        & \left( \frac{ \basedualvar}{4\radius} - \frac{\basereg \radius}{ \sdev^2} + \smooth\right)\rho^2 +  \ex_{\sample\sim\prob}\left[\dualfunc\left(\func, \sample, \frac{ \basedualvar}{4\radius} - \frac{\basereg \radius}{ \sdev^2} + \smooth, \reg\right)\right]
        + \frac{ \radius \strong}{1024}\parens*{\basedualvar}^2  \\
        &~~~\leq  
        \min \left( \left( \frac{ \basedualvar}{8\radius} - \frac{\basereg \radius}{2 \sdev^2} - \smooth \right) \rho^2 +  \ex_{\sample\sim\prob}\left[\dualfunc\left(\func, \sample, \frac{ \basedualvar}{8\radius} - \frac{\basereg \radius}{2 \sdev^2} - \smooth, \reg\right)\right], \right. \\
        &~~~~~~~~~~~~~~ \left. \left( \frac{ \basedualvar}{2\radius} - \frac{2\basereg \radius}{ \sdev^2} + \smooth \right) \radius^2 + \ex_{\sample\sim\prob}\left[\dualfunc\left(\func, \sample, \frac{ \basedualvar}{2\radius} - \frac{2\basereg \radius}{ \sdev^2} + \smooth, \reg\right)\right]\right)
    \end{align}
    \begin{align}
     \text{and}  ~~~~~   \max\left(\frac{\basedualvar}{32\radius}, \init[\dualvar]+\smooth\right)
        \leq 
         \frac{\basedualvar}{8\radius} - \frac{\basereg \radius}{2 \sdev^2} - \smooth
         \leq 
        \frac{\basedualvar}{4\radius} - \frac{\basereg \radius}{ \sdev^2} + \smooth
        \leq 
        \frac{\basedualvar}{2\radius}
    \end{align}
\end{lemma}
\begin{proof}
    Fix $\func \in \funcs$.
    Consider the function $\apprx{\dualfuncalt} : \dualvar \mapsto \dualvar\radius^2 + \ex_{\sample \sim \prob}[\apprx{\dualfunc}(\func, \sample, \dualvar, \reg,\sdev)]$ where $\apprx{\dualfunc}$ is defined in \eqref{eq:phi_bar}. By \cref{lemma:study_model} invoked with $a \gets \radius^2$, $b \gets \half \ex_\prob[\norm{\grad \func}^2_2]$, $c \gets \half[\reg \dims]$ and $r \gets \frac{\reg}{\sdev^2}$, its unique minimizer is
    \begin{align}
        \sol[\dualvar] \defeq \pospart*{\frac{\basereg \dims + \sqrt{(\basereg \dims)^2 +  8 \ex_{\prob} \norm{\grad \func}^2_2}}{4 \radius} - \frac{\basereg \radius}{\sdev^2}}
        =
         \pospart*{\frac{\basedualvar}{4 \radius} - \frac{\basereg \radius}{\sdev^2}}\,.
    \end{align}
    where we used that $\reg = \basereg \radius$.
    And, since $\frac{\basereg \radius^2}{\sdev^2} = \frac{\basereg}{\basesdev^2}\leq  \frac{\basedualvar}{8}$ by \cref{asm:rho_small}, $\sol[\dualvar]$ actually satisfies
    \begin{align}\label{eq:proof_lb_dualvar_sandwich}
        \frac{\basedualvar}{8\radius} \leq  \sol[\dualvar]  \leq  \frac{\basedualvar}{4 \radius}\,.
    \end{align}
    Moreover, \cref{lemma:study_model} also shows that, on $[0, 2\sol[\dualvar]]$, $\apprx{\dualfuncalt}$ is strongly convex with modulus
    \begin{align}
        \frac{\ex_{\prob}[\norm{\grad \func}^2_2]}{(2\sol[\dualvar]+\frac{\basereg \radius}{\sdev^2})^3} + \frac{\basereg \dims \radius}{2 (2\sol[\dualvar]+\frac{\basereg \radius}{\sdev^2})^2}
        = \frac{\ex_{\prob}[\norm{\grad \func}^2_2]}{(\frac{\basedualvar}{2 \radius}-\frac{\basereg \radius}{\sdev^2})^3} + \frac{\basereg \dims \radius}{2 (\frac{\basedualvar}{2 \radius}-\frac{\basereg \radius}{\sdev^2})^2}
        \geq
        \radius^3 \strong\,.
    \end{align}

    Now, we notice that  $\reg = \basereg \radius \leq \init[\reg]$ by \cref{asm:rho_small}. Then, if $\dualvar\in[\init[\dualvar] + \smooth  , 2\sol[\dualvar] - \smooth]$, then \cref{lemma:approx_laplace_dualfunc} (applied twice) and the strong convexity of $\apprx{\dualfuncalt}$ yield
    \begin{align}
        &\dualvar \radius^2  + \ex_{\sample\sim\prob}[\dualfunc(\func, \sample, \dualvar, \reg,\sdev)] \\
         &\geq 
        \dualvar \radius^2 +  \ex_{\sample \sim \prob}[\apprx{\dualfunc}(\func, \sample, \dualvar + \smooth, \reg,\sdev)] -\reg \cst{approx_phi_coeff} e^{-\cst{approx_phi_rate}\parens*{\frac{\dualvar + \smooth}{\reg}}^{\third}} \\ 
        &\geq 
        \sol[\dualvar]  \radius^2 + \ex_{\sample \sim \prob}[\apprx{\dualfunc}(\func, \sample, \sol[\dualvar], \reg,\sdev)]
        - \radius^2 \smooth
        + \frac{ \radius^3 \strong}{2}\parens*{\sol[\dualvar] - (\dualvar + \smooth)}^2         - \reg\cst{approx_phi_coeff} e^{-\cst{approx_phi_rate}\parens*{\frac{\dualvar + \smooth}{\reg}}^{\third}} \\
        &\geq
        \sol[\dualvar]  \radius^2 + \ex_{\sample\sim\prob}[\dualfunc(\func, \sample, \sol[\dualvar] + \smooth, \reg,\sdev)]
        - \radius^2 \smooth 
        + \frac{ \radius^3 \strong}{2}\parens*{\sol[\dualvar] - (\dualvar + \smooth)}^2
        - \reg\cst{approx_phi_coeff} e^{-\cst{approx_phi_rate}\parens*{\frac{\dualvar + \smooth}{\reg}}^{\third}}
        - \reg \cst{approx_phi_coeff} e^{-\cst{approx_phi_rate}\parens*{\frac{\sol[\dualvar]}{\reg}}^{\third}} \,.
        \label{eq:proof_lb_dualvar_intermediate_ineq}
    \end{align}

    We first wish to choose $\dualvar = \frac{\sol[\dualvar]}{2} - \smooth$. By \eqref{eq:proof_lb_dualvar_sandwich}, since $\radius \leq \frac{\basedualvar}{32(\init[\dualvar] + \smooth)}$ by \cref{asm:rho_small}, this choice of $\dualvar$ is indeed greater than or equal to $\init[\dualvar]+ \smooth$ and \eqref{eq:proof_lb_dualvar_intermediate_ineq} leads to
    \begin{align}
        & \left(\frac{\sol[\dualvar]}{2} - \smooth\right) \radius^2 + \ex_{\sample\sim\prob}[\dualfunc(\func, \sample, \frac{\sol[\dualvar]}{2} - \smooth, \reg,\sdev)] \\
        &\geq
        \sol[\dualvar]  \radius^2 + \ex_{\sample\sim\prob}[\dualfunc(\func, \sample, \sol[\dualvar] + \smooth, \reg,\sdev)]
        - \radius^2 \smooth 
        + \frac{ \radius^3 \strong}{8}\parens*{\sol[\dualvar]}^2
        - \reg\cst{approx_phi_coeff} e^{-\cst{approx_phi_rate}\parens*{\frac{\sol[\dualvar]}{2\reg}}^{\third}}
        - \reg \cst{approx_phi_coeff} e^{-\cst{approx_phi_rate}\parens*{\frac{\sol[\dualvar]}{\reg}}^{\third}}\\
        &\geq
        ( \sol[\dualvar] + \smooth)  \radius^2+ \ex_{\sample\sim\prob}[\dualfunc(\func, \sample, \sol[\dualvar] + \smooth, \reg,\sdev)]
       - 2 \radius^2 \smooth 
       + \frac{ \radius \strong}{512}\parens*{\basedualvar}^2
       - 2\reg \cst{approx_phi_coeff} e^{-\cst{approx_phi_rate}\parens*{\frac{\basedualvar}{8\reg\radius}}^{\third}}\,. \label{eq:proof_lb_dualvar_cl_first_part}
    \end{align}
    where we used \eqref{eq:proof_lb_dualvar_sandwich} again for the last inequality.

    To obtain the other inequality we pick $\dualvar = 2\sol[\dualvar] - \smooth$, which is greater or equal to  $\init[\dualvar]+ \smooth$  by   \cref{asm:rho_small} and \eqref{eq:proof_lb_dualvar_sandwich} as above. Then, \eqref{eq:proof_lb_dualvar_intermediate_ineq} yields
    \begin{align}
        & \left( 2\sol[\dualvar] - \smooth \right)\radius^2 + \ex_{\sample\sim\prob}[\dualfunc(\func, \sample, 2\sol[\dualvar] - \smooth, \reg,\sdev)] \\
        &\geq 
        \sol[\dualvar]  \radius^2+ \ex_{\sample\sim\prob}[\dualfunc(\func, \sample, \sol[\dualvar] + \smooth, \reg,\sdev)]
       - \radius^2 \smooth 
       + \frac{ \radius^3 \strong}{2}\parens*{\sol[\dualvar]}^2
       - \reg\cst{approx_phi_coeff} e^{-\cst{approx_phi_rate}\parens*{\frac{2\sol[\dualvar]}{\reg}}^{\third}}
       - \reg \cst{approx_phi_coeff} e^{-\cst{approx_phi_rate}\parens*{\frac{\sol[\dualvar]}{\reg}}^{\third}} \\
       &\geq 
       ( \sol[\dualvar] + \smooth) \radius^2+\ex_{\sample\sim\prob}[\dualfunc(\func, \sample, \sol[\dualvar] + \smooth, \reg,\sdev)]
      - 2 \radius^2 \smooth 
      + \frac{ \radius \strong}{512}\parens*{\basedualvar}^2
      - 2\reg \cst{approx_phi_coeff} e^{-\cst{approx_phi_rate}\parens*{\frac{\basedualvar}{8\reg\radius}}^{\third}}
   \end{align}
    where we used \eqref{eq:proof_lb_dualvar_sandwich} again, and degraded the constants to match those of \eqref{eq:proof_lb_dualvar_cl_first_part}.

    Thus, we have that $\sol[\dualvar] = \frac{\basedualvar}{4 \radius} - \frac{\basereg \radius}{\sdev^2}$ and 
    \begin{align}
       & ( \sol[\dualvar] + \smooth)  \radius^2 + \ex_{\sample\sim\prob}[\dualfunc(\func, \sample, \sol[\dualvar] + \smooth, \reg,\sdev)]
      - 2\radius^2 \smooth 
      + \frac{ \radius \strong}{512}\parens*{\basedualvar}^2
      - 2\reg \cst{approx_phi_coeff} e^{-\cst{approx_phi_rate}\parens*{\frac{\basedualvar}{8\reg\radius}}^{\third}} \\
      &\leq \min\left(\left(\frac{\sol[\dualvar]}{2} - \smooth\right) \radius^2 + \ex_{\sample\sim\prob}[\dualfunc(\func, \sample, \frac{\sol[\dualvar]}{2} - \smooth, \reg,\sdev)] , \left( 2\sol[\dualvar] - \smooth \right)\radius^2 + \ex_{\sample\sim\prob}[\dualfunc(\func, \sample, 2\sol[\dualvar] - \smooth, \reg,\sdev)]  \right)
   \end{align}

    All that is left to show for the main result of the lemma is that 
    \begin{align}
       & - 2 \radius^2 \smooth 
        + \frac{ \radius \strong}{512}\parens*{\basedualvar}^2
        - 2\reg \cst{approx_phi_coeff} e^{-\cst{approx_phi_rate}\parens*{\frac{\basedualvar}{8\reg\radius}}^{\third}} \geq \frac{\radius\strong}{1024}\parens*{\basedualvar}^2\\
     \Leftrightarrow ~~~~~ & 2\radius \smooth 
        + 2\basereg \cst{approx_phi_coeff} e^{-\cst{approx_phi_rate}\parens*{\frac{\basedualvar}{8\basereg\radius^2}}^{\third}} \leq \frac{\strong}{1024}\parens*{\basedualvar}^2 \, .     \label{eq:boundtermexp}
    \end{align}
    This is a consequence of \cref{asm:rho_small} which states that
    \begin{align}
       & \radius \leq  \frac{\strong}{4096 \smooth} (\basedualvar)^2       \quad\text{ and }\quad
        &\radius \leq 
         \sqrt{\frac{\cst{approx_phi_rate}^{{3}}\basedualvar}{8 \basereg}} \parens*{
            \log\parens*{\frac{4096 \basereg \cst{approx_phi_coeff}}{\strong (\basedualvar)^2}}
        }^{-\frac{3}{2}}_+\,, \\
   \text{     which imply that } ~~     & 2\radius \smooth 
        \leq \frac{\strong}{2048}\parens*{\basedualvar}^2\,,
        \quad\text{ and }\quad
             &{2 \basereg \cst{approx_phi_coeff}} e^{-\cst{approx_phi_rate}\parens*{\frac{\basedualvar }{8\basereg \radius^2}}^{\third}}
                     \leq \frac{\strong}{2048}\parens*{\basedualvar}^2\,,
    \end{align}
    so that \eqref{eq:boundtermexp} indeed holds, concluding the proof of the first part of the result.
     
  The supplementary bounds follow directly from \cref{eq:proof_lb_dualvar_sandwich} and our assumptions on $\radius$.
\end{proof}

We are now in a position to show our main result when $\radius$ is small, namely \cref{prop:lb_dualvar_asymptotic}.

\begin{proof}[Proof of \Cref{prop:lb_dualvar_asymptotic}]
    Let us first take any $\dualvar \in [ \max\left(\frac{\basedualvar}{32\radius}, \init[\dualvar]+\smooth\right) , \frac{\basedualvar}{2\radius} ]$.
    We want to instante \cref{lemma:concentration_basic} with $\rv(\func, \sample) \gets \dualfunc(\func, \sample, \dualvar, \reg,\sdev)$, $(\pspace,\dist)\gets(\funcs,\norm{\cdot}_\infty)$, whose requirements are checked since:
    \begin{enumerate}
        \item For any $\func\in \funcs$,  $\dualfunc(\func, \sample, \dualvar, \reg,\sdev)$ is measurable since the functions of $\funcs$ are continuous and thus \emph{a fortiori} measurable;
        \item By the proof of \cref{lemma:lipschitz}(a), we have that for any $\reg \geq 0$, $\sdev > 0$ and any $\sample \in \supp \prob$,  $\func\mapsto\dualfunc(\func, \sample, \dualvar, \reg,\sdev)$ is 1-Lipschitz continuous \wrt the norm $\norm{\cdot}_\infty$;
        \item With \cref{lemma:bound_dualfunc} with $\lbdualvar \gets \max\left(\frac{\basedualvar}{32\radius}, \init[\dualvar]+\smooth\right)$, $\ubdualvar\gets \frac{\basedualvar}{2\radius}$, for any $\func \in \funcs$, $\sample \in \supp \prob$  and $\reg \in [0, \init[\reg] ]$ (by \cref{asm:rho_small}), we have 
    \begin{align}
        - \ub\bdcst(\radius)
        \leq \func(\sample) - \bdcst(\lbdualvar, \ubdualvar, \reg, \sdev)
        \leq {\dualfunc}(\func, \sample, \dualvar , \reg,\sdev)  
       \leq  \func(\sample) + \bdcst(\lbdualvar, \ubdualvar, \reg, \sdev)
       \leq \fbound + \ub\bdcst(\radius) 
    \end{align}         
    where  $\ub\bdcst(\radius) $ is defined in \Cref{prop:lb_dualvar_asymptotic}.
    \end{enumerate}
  
    \noindent Since $\frac{\basedualvar}{8\radius} - \frac{\basereg \radius}{2 \sdev^2} - \smooth\geq \init[\dualvar]+\smooth $ by \cref{lemma:lb_dualvar_population_asymptotic}, we can apply statement $(b)$ of \cref{lemma:concentration_basic}  with  $\dualvar \gets \frac{\basedualvar}{8\radius} - \frac{\basereg \radius}{2 \sdev^2} - \smooth$ and $\thres \gets \frac{\thres}{4}$ to have that, with probability at least $1 - \frac{\thres}{4}$, for all $\func \in \funcs$
         \begin{align}
     &\ex_{\sample\sim \prob}[\dualfunc\left(\func, \sample,\frac{\basedualvar}{8\radius} - \frac{\basereg \radius}{2 \sdev^2} - \smooth, \reg \right)]
        - \ex_{\sample\sim \emp}[\dualfunc\left(\func, \sample, \frac{\basedualvar}{8\radius} - \frac{\basereg \radius}{2 \sdev^2} - \smooth, \reg\right)] \\
       & \quad\quad \leq   \frac{48\dudley}{\sqrt{\nsamples}}
        + 4(\fbound + \ub\bdcst(\radius)) \sqrt{\frac{\log \frac{4}{\thres}}{2\nsamples}}\,.
         \end{align}

    \noindent Similarly, we can apply statement $(a)$ of \cref{lemma:concentration_basic} with
$\dualvar \gets \frac{\basedualvar}{4\radius} - \frac{\basereg \radius}{ \sdev^2} + \smooth $ and $\thres \gets \frac{\thres}{4}$ to get that, with probability at least $1 - \frac{\thres}{4}$, for all $\func \in \funcs$,
  \begin{align}
             &\ex_{\sample\sim \emp}[\dualfunc\left(\func, \sample,\frac{\basedualvar}{4\radius} - \frac{\basereg \radius}{\sdev^2}   + \smooth, \reg\right)]
        - \ex_{\sample\sim \prob}[\dualfunc\left(\func, \sample, \frac{\basedualvar}{4\radius} - \frac{\basereg \radius}{2 \sdev^2} + \smooth, \reg\right)] \\
       & \quad\quad \leq  \frac{48\dudley}{\sqrt{\nsamples}}
        + 4(\fbound + \ub\bdcst(\radius)) \sqrt{\frac{\log \frac{4}{\thres}}{2\nsamples}}\,.
         \end{align}

        Combining the two statements above and using \cref{lemma:lb_dualvar_population_asymptotic}, we get that, with probability at least $1 - \half[\thres]$, for any $\func \in \funcs$, 
        \begin{align}
            & \left( \frac{\basedualvar}{8\radius} - \frac{\basereg \radius}{2 \sdev^2} - \smooth \right) \radius^2 + \ex_{\sample\sim \emp}[\dualfunc\left(\func, \sample, \frac{\basedualvar}{8\radius} - \frac{\basereg \radius}{2 \sdev^2} - \smooth, \reg\right)] \\
        &\geq  \left( \frac{\basedualvar}{8\radius} - \frac{\basereg \radius}{2 \sdev^2} - \smooth \right) \radius^2 + \ex_{\sample\sim \prob}[\dualfunc\left(\func, \sample,\frac{\basedualvar}{8\radius} - \frac{\basereg \radius}{2 \sdev^2} - \smooth, \reg \right)] \\
        &~~~~~~~~~     - \frac{48\dudley}{\sqrt{\nsamples}}    - 4(\fbound + \ub\bdcst(\radius)) \sqrt{\frac{\log \frac{4}{\thres}}{2\nsamples}} \\
        &\geq \left( \frac{ \basedualvar}{4\radius} - \frac{\basereg \radius}{ \sdev^2} + \smooth \right)\radius^2 +  \ex_{\sample\sim\prob}\left[\dualfunc\left(\func, \sample, \frac{ \basedualvar}{4\radius} - \frac{\basereg \radius}{ \sdev^2} + \smooth, \reg\right)\right] \\
        &~~~~~~~~~      + \frac{ \radius \strong}{1024}\parens*{\basedualvar}^2
        - \frac{48\dudley}{\sqrt{\nsamples}}     - 4(\fbound + \ub\bdcst(\radius)) \sqrt{\frac{\log \frac{4}{\thres}}{2\nsamples}} \\
        &\geq \left( \frac{ \basedualvar}{4\radius} - \frac{\basereg \radius}{ \sdev^2} + \smooth \right)\radius^2 + \ex_{\sample\sim\emp}\left[\dualfunc\left(\func, \sample, \frac{ \basedualvar}{4\radius} - \frac{\basereg \radius}{ \sdev^2} + \smooth, \reg\right)\right] \\
        &~~~~~~~~~       + \frac{ \radius \strong}{1024}\parens*{\basedualvar}^2
        - \frac{96\dudley}{\sqrt{\nsamples}}     - 8(\fbound + \ub\bdcst(\radius)) \sqrt{\frac{\log \frac{4}{\thres}}{2\nsamples}} \,.
        \end{align}

        Noting that the assumption on $\radius$ in \cref{prop:lb_dualvar_asymptotic} implies that 
\begin{align}
        \frac{ \radius \strong}{1024}\parens*{\basedualvar}^2
        \geq
        &\frac{96\dudley}{\sqrt{\nsamples}}
        + 8\left(\fbound + \ub\bdcst(\radius)\right) \sqrt{\frac{\log \frac{4}{\thres}}{2\nsamples}} \,
\end{align}
we have proven that, with probability at least $1 - \half[\thres]$, for any $\func \in \funcs$,
\begin{align}
    \dualfuncalt\left(\frac{ \basedualvar}{8\radius} - \frac{\basereg \radius}{2 \sdev^2} - \smooth\right)&\geq 
    \dualfuncalt\left( \frac{ \basedualvar}{4\radius} - \frac{\basereg \radius}{ \sdev^2} + \smooth \right)
\end{align}
where $\dualfuncalt : \dualvar \mapsto \dualvar\radius^2 + \ex_{\sample \sim \emp}[{\dualfunc}(\func, \sample, \dualvar, \reg,\sdev)]$. 
Now, since $\dualfuncalt$ is convex, this means that its minimizers on $\mathbb{R}_+$ are greater than 
\begin{align}
    \frac{ \basedualvar}{8\radius} - \frac{\basereg \radius}{2 \sdev^2} - \smooth \geq  \frac{\basedualvar}{32\radius}
\end{align}
where the inequality comes from \cref{lemma:lb_dualvar_population_asymptotic}. 

Using the same reasoning, one can get that  with probability at least $1 - \half[\thres]$ the minimizers are no greater than 
\begin{align}
    \frac{\basedualvar}{4\radius} - \frac{\basereg \radius}{ \sdev^2} + \smooth
        \leq 
        \frac{\basedualvar}{2\radius}.
\end{align}

Thus, we have shown that with probability at least $1 - \half[\thres]$, for any $\func \in \funcs$,
\begin{align}
        \emprisk[\radius^2][\reg] &= \inf_{ \dualvar \geq 0 } \dualvar \radius^2 + \ex_{\sample\sim\empirical} \left[ \dualfunc(\func, \sample, \dualvar, \reg,\sdev) \right] \\
        &= \inf_{\frac{\basedualvar}{32 \radius} \leq \dualvar \leq \frac{\basedualvar}{2 \radius}} \dualvar \radius^2 + \ex_{\sample\sim\empirical} \left[ \dualfunc(\func, \sample, \dualvar, \reg,\sdev) \right] \,.
\end{align}
\end{proof}

\section{Dual bound when $\radius$ is close to this maximal radius} \label{sec:critical}

Complementary to the previous section \cref{sec:small}, we consider the case where $\radius$ is close than the critical radius.
Though the bounds of this section are much worse that the one of \cref{sec:small} when $\radius$ goes to zero, they hold for the whole ranges of $\radius$ considered in the theorems.

As mentioned in \cref{remark:radius-c}, as $\radius$ grows, the Wasserstein ball constraint can stop being active, leading to a null dual variable. Thus, it is essential that $\radius$ be lower then the critical radius to stay in the distributionally robust regime and to avoid the worst-case regime. In that case, we are able to lower-bound the dual multiplier $\dualvar$.

We defined the critical radius in standard \ac{WDRO} case in \cref{remark:radius-c} and we extend it here to cover the regularized case:
\begin{align}
    \label{eq:gen-def-radiusc}
    \regcrit &\defeq 
        \left\{
            \begin{array}{ll}
            \inf_{\func \in \funcs}  \ex_{\sample \sim \prob} \left[ \ex_{\samplealt \sim \modbase(\cdot | \sample)} [\half \sqnorm{\sample -  \samplealt}] \right] &\text{ if } \reg > 0\\
            \inf_{\func \in \funcs}  \ex_{\sample \sim \prob}[\min \setdef{\half \sqnorm{\sample -  \samplealt}}{\samplealt \in \argmax \func}] &\text{ otherwise}
        \end{array}
        \right.
\end{align}
where $ \modbase[\funcalt](\dd \samplealt | \sample) \propto e^{\funcalt(\samplealt)} \basecpl(\dd \samplealt | \sample)$ is a conditional probability distribution parametrized by an $\samples\to\mathbb{R}$ function $\funcalt$, \ie
\begin{align}
    \ex_{\samplealt \sim \modbase[\funcalt](\cdot | \sample)} \left[ h(\sample, \samplealt) \right] = \frac{ \ex_{\samplealt \sim \basecpl(\cdot | \sample)}  \left[ e^{\funcalt(\samplealt)} h(\sample, \samplealt) \right] }{\ex_{\samplealt' \sim \basecpl(\cdot | \sample)}  \left[ e^{\funcalt(\samplealt')} \right] } .
\end{align}

For this part of the proof, the case when $\reg=0$ differs from the regularized one $\reg > 0$. We thus present them in separate sections \cref{sec:criticalunreg,sec:criticalreg}.

\subsection{Standard \ac{WDRO} case} \label{sec:criticalunreg}

The main result of this section in the standard \ac{WDRO} case is \cref{prop:lb_dualvar_unreg_crit} below.

\begin{proposition}\label{prop:lb_dualvar_unreg_crit}
    Let \cref{assumption:add-geometric} hold and fix a threshold $\thres \in (0, 1)$. Assume that 
        \begin{align}
           \radius^2 &\leq \regcrit[0,0]  - \frac{2 \tmpbdcst(\thres)}{\sqrt \nsamples}        \end{align}   
        \begin{align}
   \text{with } ~~         \tmplipcst &\defeq    \frac{ 16 \sup_{\func \in \funcs} \ex_{\sample \sim \prob}[\half \distance^2(\sample, \argmax \func)] }{ \strongOpt} \\
   \text{and } ~~          \tmpbdcst(\thres) &\defeq {48\dudley[\funcs, \distfunc]}
            + 2 \sqrt{\maxcost {\log {1}/{\thres}}}\,.
        \end{align}
        where $\dudley[\funcs, \distfunc]$ and $\distfunc$ are defined in \cref{sec:dudley}, and $\afterinit[\dualvar] > 0$ is a constant depending on $\samples$, $\funcs$, $\hsmooth$, $\strongOpt$ and $\maxcost$.
        
    Then, with probability at least $1 - \thres$, we have    
    \begin{align}
        \forall \func \in \funcs,\quad
        \emprisk[\radius^2][\reg] = \inf_{\lbdualvar \leq \dualvar} \dualvar \radius^2 + \ex_{\sample\sim\empirical} \left[ \dualfunc(\func, \sample, \dualvar, 0) \right] 
    \end{align}
    where the dual bound $\lbdualvar$ is defined as
    \begin{align}
        \lbdualvar \defeq \min \parens*{{\afterinit[\dualvar]},
        \frac{
           \regcrit[0,0] -  \radius^2
    }{2 \tmplipcst}} .
    \end{align}
\end{proposition}

Before proceeding with the proof, we need to prove the following lemma which leverages \cref{assumption:add-geometric-simple}.

\begin{lemma}\label{lemma:lipschitz-unreg-one-sample}
    Fix $\func \in \funcs$ and $\sample \in \supp \prob$. There exists a constant $\afterinit[\dualvar] > 0$ depending on $\samples$, $\funcs$, $\hsmooth$, $\strongOpt$ and $\maxcost$ such that, for $\dualvar \in [0,\afterinit[\dualvar]]$,
    \begin{align}
        \min \setdef*{\half \sqnorm{\sample -  \samplealt}}{\samplealt \in \argmax_{\samples} \func - \half[\dualvar]\sqnorm{\sample -  \cdot}}
        \geq
        \parens*{1 - \frac{16 \dualvar}{\strongOpt}}\min \setdef*{\half \sqnorm{\sample -  \samplealt}}{\samplealt \in \argmax_{\samples} \func} .
    \end{align}
\end{lemma}

\begin{proof}
    Fix $\func \in \funcs$ and $\sample \in \supp \prob$.
    Define, for convenience, $ \sol[\samples] \defeq \argmax \func$ and
    \begin{align}
        \ddfunc(\dualvar) \defeq \min \setdef*{\half \sqnorm{\sample -  \samplealt}}{\samplealt \in \argmax_{\samples} \func - \half[\dualvar]\sqnorm{\sample -  \cdot}}\,.
    \end{align}
    \noindent \underline{Step 1: Localization in a $\bigoh(1)$-neighborhood of $\sol[\samples]$.}
    For a fixed $\OptGeoRad >0$ that will be chosen later, we show that, for $\dualvar$ small enough, $\ddfunc(\dualvar)$ is equal to
    \begin{align}
        \min \setdef*{\half \sqnorm{\sample -  \samplealt}}{\samplealt \in \argmax_{\parens*{\sol[\samples]}^\OptGeoRad} \func - \half[\dualvar]\sqnorm{\sample -  \cdot}}\quad\text{ where }\quad
        \parens*{\sol[\samples]}^\OptGeoRad \defeq \setdef{\sample \in \samples}{\distance(\sample, \sol[\samples]) \leq \OptGeoRad}\,.
    \end{align}
    Indeed, by \cref{assumption:add-geometric}, there is some $\Margin(\OptGeoRad) > 0$ such that for all $ \func \in \funcs$ and $\samplealt \in \samples\setminus  \parens*{\sol[\samples]}^\OptGeoRad$,
    \begin{align}
        \func(\samplealt) - \max \func - \half[\dualvar]\sqnorm{\sample -  \samplealt} \leq  \func(\samplealt) - \max \func
        \leq - \Margin(\OptGeoRad)\,,
    \end{align}
    while, for any $\sol[\sample] \in \sol[\samples]$,
    \begin{align}
           \func(\sol[\sample]) - \max \func - \half[\dualvar]\sqnorm{\sample -  \sol[\sample]}
        = - \half[\dualvar]\sqnorm{\sample -  \sol[\sample]}
        \geq - \dualvar \maxcost\,.
    \end{align}
    Hence, for $\dualvar \leq \frac{\Margin(\OptGeoRad)}{\maxcost}$,  $\func(\samplealt) - \half[\dualvar]\sqnorm{\sample -  \samplealt} \leq  - \Margin(\OptGeoRad)  + \max \func \leq  - \dualvar \maxcost  + \max \func \leq  \func(\sol[\sample])  - \half[\dualvar]\sqnorm{\sample -  \sol[\sample]}$. This means that points in $ \samples\setminus  \parens*{\sol[\samples]}^\OptGeoRad$ cannot maximize $\func - \half[\dualvar]\sqnorm{\sample -  \cdot}$ and so it suffices to consider the $\argmax$ over $\parens*{\sol[\samples]}^\OptGeoRad$ in the definition of $\ddfunc(\dualvar)$.

    \noindent \underline{Step 2: Localization in a $\bigoh (\dualvar)$-neighborhood of $\sol[\samples]$.} 
    Take $\sol[\samplealt]  \in \argmax_{\parens*{\sol[\samples]}^\OptGeoRad} \func - \half[\dualvar]\sqnorm{\sample -  \cdot}$. Since $\distance(\sol[\samplealt], \sol[\samples]) \leq \OptGeoRad$, the Euclidean projection of $\sol[\samplealt]$ on $\sol[\samples]$, that we denote by $\sol[\sample]$, is  at most at distance $\OptGeoRad$ of $\sol[\samplealt]$ and $\sol[\samplealt] - \sol[\sample] \in \regncone_{\sol[\sample]}(\sol[\samples])$, see \eg \citet[Thm.~6.12]{rockafellarVariationalAnalysis1998}.
    By the growth condition of \cref{assumption:add-geometric}, we get that
    \begin{equation}
        \func(\sol[\sample])
        \geq
        \func(\sol[\samplealt])
        + \half[\strongOpt] \norm{\sol[\samplealt] - \sol[\sample]}^2
        - \frac{\hsmooth}{6} \norm{\sol[\samplealt] - \sol[\sample]}^3\,.
        \label{eq:proof-lbunregcrit-growth-cond}
    \end{equation}
    But, by definition of $\sol[\samplealt]$, we also have that
    \begin{equation}
        \func(\sol[\samplealt]) - \half[\dualvar]\sqnorm{\sample -  \sol[\samplealt]}
        \geq
        \func(\sol[\sample]) - \half[\dualvar]\sqnorm{\sample -  \sol[\sample]}\,.
    \end{equation}
    Plugging \cref{eq:proof-lbunregcrit-growth-cond} we get that
    \begin{equation}
        - \half[\dualvar]\sqnorm{\sample -  \sol[\samplealt]}
        \geq
        - \half[\dualvar]\sqnorm{\sample -  \sol[\sample]}
        + \half[\strongOpt] \norm{\sol[\samplealt] - \sol[\sample]}^2
        - \frac{\hsmooth}{6} \norm{\sol[\samplealt] - \sol[\sample]}^3\,.
    \end{equation}
    Rearranging and developing $\half \sqnorm{\sample -  \sol[\sample]}$ yields
    \begin{equation}
        \frac{\hsmooth}{6} \norm{\sol[\samplealt] - \sol[\sample]}^3
        + \dualvar \inner{\sample - \sol[\samplealt], \sol[\samplealt] - \sol[\sample]}
        \geq
        \half[\dualvar + \strongOpt] \norm{\sol[\sample] - \sol[\samplealt]}^2\,,
    \end{equation}
    which gives, by Cauchy-Schwarz inequality,
    \begin{equation}\label{eq:proof-unreg-crit-step2-2}
        \frac{\hsmooth}{6} \norm{\sol[\samplealt] - \sol[\sample]}^3
    + \dualvar \norm{\sample - \sol[\samplealt]}\norm{\sol[\samplealt] - \sol[\sample]}
        \geq
        \half[\dualvar + \strongOpt] \norm{\sol[\sample] - \sol[\samplealt]}^2\,,
    \end{equation}
  We now wish to obtain a bound on $\sol[\scalar] \defeq \norm{\sol[\samplealt] - \sol[\sample]}$. If it is zero, there is nothing to do. Otherwise, assuming that it is positive, \cref{eq:proof-unreg-crit-step2-2} gives the inequation
    \begin{equation}
        \half[\strongOpt + \dualvar] \sol[\scalar] \leq \frac{\hsmooth}{6} \parens*{\sol[\scalar]}^2 + \dualvar \norm{\sample - \sol[\sample]}\,.
    \end{equation}
    When $\frac{(\strongOpt + \dualvar)^2}{4} -  \frac{2\hsmooth \dualvar}{3} \norm{\sample - \sol[\sample]}$ is non-negative, this inequation is satisfied for
    \begin{equation}
        \sol[\scalar] \notin \bracks*{\frac{(\strongOpt + \dualvar) \pm \sqrt{(\strongOpt + \dualvar)^2 - 8 \hsmooth \dualvar \norm{\sample - \sol[\sample]}/3}}{\hsmooth / 3}}\,.
    \end{equation}
    Hence, in particular, if $\sol[\scalar] \leq \frac{3\strongOpt}{\hsmooth}$, then $\sol[\scalar]$ must be less or equal than
    \begin{equation}
    {\frac{(\strongOpt + \dualvar) - \sqrt{(\strongOpt + \dualvar)^2 - 8 \hsmooth \dualvar \norm{\sample - \sol[\sample]}/3}}{\hsmooth / 3}}
    =
    \frac{3(\strongOpt + \dualvar)}{\hsmooth}
    \parens*{1 -  \sqrt{1 - \frac{8 \hsmooth \dualvar \norm{\sample - \sol[\sample]}}{3(\strongOpt + \dualvar)^2}}}
    \leq  \frac{8 \dualvar \norm{\sample - \sol[\sample]}}{\strongOpt + \dualvar}
    \end{equation}
    when $\frac{8 \hsmooth \dualvar \norm{\sample - \sol[\sample]}}{3(\strongOpt + \dualvar)^2} \leq 1$, using that $1 - \sqrt{1 - x}\leq x$ for $x \in [0, 1]$.

    Thus, assuming that $\dualvar$ is small enough so that ${8 \hsmooth \dualvar \maxcost}\leq 3{(\strongOpt)^2}$ and choosing $\OptGeoRad \defeq \frac{3\strongOpt}{\hsmooth}$ so  that  $\sol[\scalar] \leq \frac{3\strongOpt}{\hsmooth}$ by construction, we have that for any $\sol[\samplealt]  \in \argmax_{\parens*{\sol[\samples]}^\OptGeoRad} \func - \half[\dualvar]\sqnorm{\sample -  \cdot}$, there is a point $\sol[\sample] \in \sol[\samples]$ such that 
    \begin{equation}
    \norm{\sol[\samplealt] - \sol[\sample]} \leq \frac{8 \dualvar \norm{\sample - \sol[\sample]}}{\strongOpt}\,.
    \end{equation}
\noindent \underline{Step 3: Conclusion.}
    Defining the constant 
    \begin{equation}
        \afterinit[\dualvar] \defeq \min \parens*{\frac{\Margin(\OptGeoRad)}{\maxcost}, \frac{3\strongOpt^2}{8 \hsmooth \maxcost}, \frac{\strongOpt}{16}}\,,
    \end{equation}
    and using the previous steps, we have for any $\dualvar \in [0,\afterinit[\dualvar]]$ and any $\sol[\samplealt]  \in \argmax_{\samples} \func - \half[\dualvar]\sqnorm{\sample -  \cdot}$
\begin{align}
    \half \sqnorm{\sample -  \sol[\samplealt]} = \half \norm{\sample - \sol[\samplealt]}^2
                                     &=\half \norm{\sample - \sol[\sample]}^2 +
                                     \half \norm{\sol[\sample] - \sol[\samplealt]}^2
                                     - \inner{\sample - \sol[\sample], \sol[\sample] - \sol[\samplealt]}\\
                                     &\geq \half \sqnorm{\sample -  \sol[\sample]} - \norm{\sample - \sol[\sample]}\norm{\sol[\sample] - \sol[\samplealt]}\\
                                     &\geq \left( 1- \frac{16\dualvar}{\strongOpt}  \right) \half \sqnorm{\sample -  \sol[\sample]} 
\end{align}
which concludes the proof.
\end{proof}

We can now turn to the proof of our proposition.

\begin{proof}[Proof of \cref{prop:lb_dualvar_unreg_crit}]
    Let $0 \leq \dualvar\leq \lbdualvar$. 
       For $\func \in \funcs$ and $\dualvar \geq 0$, we define $\hat{\dualfuncalt} : \dualvar \mapsto \dualvar \radius^2 +\ex_{\sample\sim\empirical} \left[ \dualfunc(\func, \sample, \dualvar, 0) \right] $ and its (right-sided) derivative  $\partial_\dualvar \hat{\dualfuncalt}$.
    This derivative is given by,
        \begin{align}
            \partial_\dualvar\hat{\dualfuncalt}(\dualvar) &=
            \radius^2 - \ex_{\sample \sim \empirical}[\min \setdef*{\half \sqnorm{\sample -  \samplealt}}{\samplealt \in \argmax_{\samples} \func - \half[\dualvar]\sqnorm{\sample -  \cdot}}] \\
            &\leq \radius^2 - \parens*{1 - \frac{16 \dualvar}{\strongOpt}} \ex_{\sample \sim \empirical}[ \min \setdef*{\half \sqnorm{\sample -  \samplealt}}{\samplealt \in \argmax_{\samples} \func} ] \label{eq:boundunreg}
            \end{align}
            where we used \cref{lemma:lipschitz-unreg-one-sample} with $\dualvar \leq \lbdualvar \leq  \afterinit[\dualvar]$. 

    We then instantiate \cref{lemma:concentration_basic} with $\rv(\func, \sample) \gets \half \distance^2(\sample, \argmax \func)$, $(\pspace,\dist)\gets(\funcs, \distfunc)$, whose requirements are checked since:
    \begin{enumerate}
        \item For any $\func\in \funcs$,  $\rv(\func, \cdot)$ is measurable since the functions of $\funcs$ are continuous and thus $\argmax \func$ is \emph{a fortiori} measurable;
        \item By definition of $\distfunc$, for any $\sample \in \supp \prob$, $\func \mapsto \distance(\sample, \argmax \func)$ is $1$-Lipschitz \wrt this distance so that $\rv(\sample, \cdot)$ is $\sqrt{2 \maxcost}$-Lipschitz.
        \item By construction, the range of values $\rv$ is included in $[0, \maxcost]$.
    \end{enumerate}
    \noindent We can thus apply statement $(b)$ of \cref{lemma:concentration_basic} to have that, with probability at least $1 -\thres$, for all $\func \in \funcs$,
    \begin{align}
\ex_{\sample \sim \empirical}[ \half \distance^2(\sample, \argmax \func) ]
\geq
\ex_{\sample \sim \prob}[\half \distance^2(\sample, \argmax \func) ] - \frac{\tmpbdcst(\thres)}{\sqrt{\nsamples}} \label{eq:concentration_l2l}
    \end{align}
    Hence, putting this bound together with \cref{eq:boundunreg} yields
 \begin{align}
            \partial_\dualvar\hat{\dualfuncalt}(\dualvar) 
            &\leq \radius^2 - \parens*{1 - \frac{16 \dualvar}{\strongOpt}} \ex_{\sample \sim \prob}[ \half \distance^2(\sample, \argmax \func)]+  \frac{\tmpbdcst(\thres)}{\sqrt{\nsamples}} \\
            &\leq \radius^2 - \radius_c^2 + \lipcst \dualvar + \frac{\tmpbdcst(\thres)}{\sqrt{\nsamples}} \,,
    \end{align}
    which is non-negative for $\dualvar \leq \lbdualvar$.
\end{proof}

\subsection{Regularized case} \label{sec:criticalreg}

The main bound on $\dualvar$ of this section are given by \cref{prop:lb_dualvar_reg_crit}.

\begin{proposition}\label{prop:lb_dualvar_reg_crit}
    Fix a threshold $\thres \in (0, 1)$. Assume that $\base[\reg] > 0$ and that
        \begin{align}
           \radius^2 \leq \regcrit
            - \parens*{ \frac{48\sqrt{\Var(\reg,\sdev)} \dudley}{\reg\sqrt{\nsamples}}
            + 2\Cost_\funcs(\reg,\sdev) \sqrt{\frac{\log \frac{1}{\thres}}{2\nsamples}}}
        \end{align}
        where $\dudley[\funcs, \norm{\cdot}_\infty]$ is defined in \cref{sec:detailedhyp}.
        
    Then, with probability at least $1 - \thres$, we have    
    \begin{align}
        \forall \func \in \funcs,\quad
        \emprisk[\radius^2][\reg] = \inf_{\lbdualvar_\nsamples \leq \dualvar \leq \ub\dualvar} \dualvar \radius^2 + \ex_{\sample\sim\empirical} \left[ \dualfunc(\func, \sample, \dualvar, \reg,\sdev) \right]
    \end{align}
    where the dual bounds are defined by
    \begin{align}
        \lbdualvar_\nsamples &\defeq \frac{\reg}{\Var(\reg,\sdev)}
        \parens*{
        \regcrit  - \radius^2
        - \parens*{ \frac{48\sqrt{\Var(\reg,\sdev)} \dudley}{\reg\sqrt{\nsamples}}
        + 2\Cost_\funcs(\reg,\sdev) \sqrt{\frac{\log \frac{1}{\thres}}{2\nsamples}}}
        }  \\
        \text{and } ~~~  \ub\dualvar &\defeq \max \parens*{\frac{12 \reg}{\GeoRad^2} \log(2\times 6^{\dims/2}),e^{\frac{\sup_{\func \in \funcs}\norm{\func}_\infty}{\reg}} \frac{\basereg}{\radius}} \, .
    \end{align}
\end{proposition}

\begin{proof}
    \underline{Lower-bound:} ~
    By \cref{assumption:funcs}, for any $\func \in \funcs$, $\sample \in \supp \prob$, $\dualvar\mapsto\dualfunc(\func, \sample, \dualvar, \reg,\sdev)$ is twice differentiable and its derivatives are for any $\dualvar \geq 0$
    \begin{align}
        \partial_\dualvar \dualfunc(\func, \sample, \dualvar, \reg,\sdev) &= -\ex_{\samplealt \sim \modbase[\frac{\func- \dualvar \sqnorm{\sample -  \cdot}/2}{\reg}](\cdot | \sample)}[ \half \sqnorm{\sample -  \samplealt}]\\
        \partial^2_\dualvar \dualfunc(\func, \sample, \dualvar, \reg,\sdev) &= \frac{1}{\reg}\Var_{\samplealt \sim \modbase[\frac{\func- \dualvar \sqnorm{\sample -  \cdot}/2}{\reg}](\cdot | \sample)}[ \half \sqnorm{\sample -  \samplealt}]\,,
    \end{align}
    and using $\Var(\reg, \sdev)$ which is defined in \cref{sec:detailedhyp}, we get that,  for any $\dualvar \geq 0$,
    \begin{align}
        0 \leq \partial^2_\dualvar \dualfunc(\func, \sample, \dualvar, \reg,\sdev) &\leq
        \frac{1}{\reg}\Var(\reg,\sdev)\,.
    \end{align}

    As a consequence,
    \begin{align}
        \partial_\dualvar\left\{ \dualvar \radius^2 + \ex_{\sample\sim\empirical} \left[ \dualfunc(\func, \sample, \dualvar, \reg,\sdev) \right] \right\} &= \radius^2 + \ex_{\sample\sim\empirical} \left[ \partial_\dualvar \dualfunc(\func, \sample, \dualvar, \reg,\sdev) \right] \\
        &\leq  \radius^2 + \ex_{\sample\sim\empirical} \left[ \partial_\dualvar \dualfunc(\func, \sample, 0, \reg,\sdev) \right] +  \frac{\dualvar}{\reg}\Var(\reg,\sdev)  \\
        &=  \radius^2 - \ex_{\sample\sim\empirical} \left[ \ex_{\samplealt \sim \modbase[\func/\reg](\cdot | \sample)}[ \half \sqnorm{\sample -  \samplealt}] \right] +  \frac{\dualvar}{\reg}\Var(\reg,\sdev) \,.  \label{eq:beforeConcentration}
     \end{align}

     Now, we want to instante \cref{lemma:concentration_basic} with $\rv(\func, \sample) \defeq \ex_{\samplealt \sim \modbase(\cdot | \sample)} \half \sqnorm{\sample -  \samplealt}$, $(\pspace,\dist)\gets(\funcs,\norm{\cdot}_\infty)$, whose requirements are checked since:
    \begin{enumerate}
        \item For any $\func\in \funcs$,  $\ex_{\samplealt \sim \modbase(\cdot | \sample)} \half \sqnorm{\sample -  \samplealt}$ is measurable since the functions of $\funcs$ are continuous and thus \emph{a fortiori} measurable;
        \item To show that $\func \mapsto \rv(\func, \sample)$ is $\inv{\reg}\sqrt{\Var(\reg,\sdev)}$-Lipschitz, we take $\func, \funcalt \in \funcs$ and define, for $t \in [0, 1]$, $\func_t = \func + t(\funcalt - \func)$.
        Since, $\norm{\func - \funcalt}_\infty < +\infty$ and $ \sup_{\sample \in \supp \prob}\ex_{\samplealt \sim \condmodbase[\func_t/\reg]} \half \sqnorm{\sample -  \samplealt} < + \infty$ by compactness of $\samples$, \cref{assumption:set}, $t \mapsto \rv(\func_t, \sample)$ is differentiable with derivative,
          \begin{align}
            \frac{\dd}{\dd t} \rv(\func_t, \sample)  &= \frac{1}{\reg}\ex_{\samplealt \sim \condmodbase[\unfrac{\func_t}{\reg}]} [\half \sqnorm{\sample -  \samplealt}(\funcalt(\samplealt) - \func(\samplealt))] \\
                                                     & \quad\quad  -\frac{1}{\reg} \ex_{\samplealt \sim \condmodbase[\unfrac{\func_t}{\reg}]} [\half \sqnorm{\sample -  \samplealt}]
                                                       \ex_{\samplealt \sim \condmodbase[\unfrac{\func_t}{\reg}]} [\funcalt(\samplealt) - \func(\samplealt)]\\
                                                    &= \frac{1}{\reg}\ex_{\samplealt \sim \condmodbase[\unfrac{\func_t}{\reg}]} \left[\parens*{\half \sqnorm{\sample -  \samplealt} - \ex_{\samplealtalt\sim \condmodbase[\unfrac{\func_t}{\reg}]} [\half \sqnorm{\sample -  \samplealtalt}]} \right. \\
                                                    & \quad\quad \times \left.
                                                            \parens*{(\funcalt(\samplealt) - \func(\samplealt)) -
                                                                \ex_{\samplealtalt \sim \condmodbase[\unfrac{\func_t}{\reg}]} [\funcalt(\samplealtalt) - \func(\samplealtalt)]}\right]\,.
        \end{align}
        By using Cauchy-Schwarz inequality, we get that,
        \begin{align}
            \frac{\dd}{\dd t} \rv(\func_t, \sample) &\leq \inv{\reg}  \sqrt{\Var_{\samplealt \sim \condmodbase[\unfrac{\func_t}{\reg}]} [\half \sqnorm{\sample -  \samplealt}]}
                                                           \sqrt{\Var_{\samplealt \sim \condmodbase[\unfrac{\func_t}{\reg}]} [\funcalt(\samplealt) - \func(\samplealt)]}\\
                                                    &\leq \inv{\reg} \sqrt{\Var_{\samplealt \sim \condmodbase[\unfrac{\func_t}{\reg}]} [\half \sqnorm{\sample -  \samplealt}]}
                                                           \sqrt{\ex_{\samplealt \sim \condmodbase[\unfrac{\func_t}{\reg}]} [\parens*{\funcalt(\samplealt) - \func(\samplealt)}^2]}\\
                                                    &\leq \inv{\reg} \sqrt{\Var_{\samplealt \sim \condmodbase[\unfrac{\func_t}{\reg}]} [\half \sqnorm{\sample -  \samplealt}]}
                                                            \norm{\funcalt - \func}_\infty\,,
        \end{align}
        which gives the desired Lipschitz condition;
    \item The random variables $\rv(\func, \sample)$ lie between 0 and $\Cost_\funcs(\reg,\sdev)$, which is defined in \cref{sec:detailedhyp}.
    \end{enumerate}
  
    \noindent We can thus apply statement $(b)$ of \cref{lemma:concentration_basic} to have that, with probability at least $1 - \thres$, for all $\func \in \funcs$
         \begin{align}
     &\ex_{\sample\sim \prob}[ \ex_{\samplealt \sim \modbase(\cdot | \sample)} \half \sqnorm{\sample -  \samplealt} ]
        - \ex_{\sample\sim \emp}[\ex_{\samplealt \sim \modbase(\cdot | \sample)} \half \sqnorm{\sample -  \samplealt}] \\
       & \quad\quad \leq   \frac{48\sqrt{\Var(\reg,\sdev)} \dudley}{\reg\sqrt{\nsamples}}
       + 2\Cost_\funcs(\reg,\sdev) \sqrt{\frac{\log \frac{1}{\thres}}{2\nsamples}}\,. \label{eq:critical_concentration}
         \end{align}

         Combining \eqref{eq:beforeConcentration} and \eqref{eq:critical_concentration}, we obtain that with probability at least $1 - \thres$
         \begin{align}
            &\partial_\dualvar\left\{ \dualvar \radius^2 + \ex_{\sample\sim\empirical} \left[ \dualfunc(\func, \sample, \dualvar, \reg,\sdev) \right] \right\} \\
             &\leq  \radius^2 - \ex_{\sample\sim\prob} \left[ \ex_{\samplealt \sim \modbase[\func/\reg](\cdot | \sample)}[ \half \sqnorm{\sample -  \samplealt}] \right] +  \frac{\dualvar}{\reg}\Var(\reg,\sdev) +  \frac{48\sqrt{\Var(\reg,\sdev)} \dudley}{\reg\sqrt{\nsamples}}
            + 2\Cost_\funcs(\reg,\sdev) \sqrt{\frac{\log \frac{1}{\thres}}{2\nsamples}} \\
            &\leq  \radius^2 - \regcrit +  \frac{\dualvar}{\reg}\Var(\reg,\sdev) +  \frac{48\sqrt{\Var(\reg,\sdev)} \dudley}{\reg\sqrt{\nsamples}}
            + 2\Cost_\funcs(\reg,\sdev) \sqrt{\frac{\log \frac{1}{\thres}}{2\nsamples}} \\
            &= \frac{1}{\reg}\Var(\reg,\sdev) \left( \dualvar - \lbdualvar_\nsamples \right) 
         \end{align} 
         where $\lbdualvar_\nsamples \geq 0 $ is as defined in the statement of the result.

         Hence, for all $ 0 \leq \dualvar \leq \lbdualvar_\nsamples $, the derivative of $\dualvar \mapsto \dualvar \radius^2 + \ex_{\sample\sim\empirical} \left[ \dualfunc(\func, \sample, \dualvar, \reg,\sdev) \right]$ is negative; and since this function is convex, this means that its minimizers are greater than $ \lbdualvar_\nsamples $ with probability at least $1 - \thres$ which is our result.

         \noindent\underline{Upper-bound:} ~ Almost surely, for any $\func \in \funcs$, let us begin by bounding the $ \partial_\dualvar \dualfunc(\func, \sample, \dualvar, \reg,\sdev)$ for $\dualvar \geq 0$, $\func \in \funcs$ and $\sample \in \supp \prob$. Its expression is given by
    \begin{align}
        - \partial_\dualvar \dualfunc(\func, \sample, \dualvar, \reg,\sdev) &= \ex_{\samplealt \sim \modbase[\frac{\func- \dualvar \sqnorm{\sample -  \cdot}/2}{\reg}](\cdot | \sample)}[ \half \sqnorm{\sample -  \samplealt}]
            \leq 
            e^{\frac{\norm{\func}_\infty}{\reg}}
            \frac{
            \int_{\samples} \half \sqnorm{\sample -  \samplealt}e^{- \left(\frac{\dualvar}{\reg} + \frac{1}{\sdev^2}\right) \half \sqnorm{\sample -  \samplealt}} \dd \samplealt
        }{
            \int_{\samples} e^{- \left(\frac{\dualvar}{\reg} + \frac{1}{\sdev^2}\right) \half \sqnorm{\sample -  \samplealt}} \dd \samplealt
        }\,.
    \end{align}
    On the one hand, we lower-bound the denominator using \cref{lemma:approx_laplace} and \cref{assumption:support_interior} as
    \begin{align}
    \frac{1}{(2 \pi)^{\dims/2}} \parens*{\frac{\dualvar}{\reg} + \frac{1}{\sdev^2}}^{\dims / 2} \int_{\samples} e^{- \left(\frac{\dualvar}{\reg} + \frac{1}{\sdev^2}\right) \half \sqnorm{\sample -  \samplealt}} \dd \samplealt
    &\geq
    1 - 6^{\dims/2}e^{-\frac{\GeoRad^2}{12} \left(\frac{\dualvar}{\reg} + \frac{1}{\sdev^2}\right)}
    \geq \half\,,
    \end{align}
    where we used that $\dualvar \geq \frac{12 \reg}{\GeoRad^2} \log(2\times 6^{\dims/2})$.

    On the other hand, the denominator is upper-bounded as
    \begin{align}
\frac{1}{(2 \pi)^{\dims/2}} \parens*{\frac{\dualvar}{\reg} + \frac{1}{\sdev^2}}^{\dims / 2}
\int_{\samples} \half \sqnorm{\sample -  \samplealt}e^{- \left(\frac{\dualvar}{\reg} + \frac{1}{\sdev^2}\right) \half \sqnorm{\sample -  \samplealt}} \dd \samplealt
\leq 
\half \parens*{\frac{\dualvar}{\reg} + \frac{1}{\sdev^2}}^{-1} \leq \frac{\reg}{2 \dualvar}\,.
    \end{align}
    Hence, we have shown that
$
        - \partial_\dualvar \dualfunc(\func, \sample, \dualvar, \reg,\sdev) 
            \leq 
            e^{\frac{ \norm{\func}_\infty}{\reg}}
            \frac{\reg}{\dualvar}
$
    and, as a consequence, 
    \begin{align}
        \radius ^2 + \ex_{\sample \sim \emp}[\partial_\dualvar \dualfunc(\func, \sample, \dualvar, \reg,\sdev)]
            &\geq \radius^2 -
            e^{\frac{\norm{\func}_\infty}{\reg}}
            \frac{\reg}{\dualvar}\,,
    \end{align}
    which is non-negative for $\dualvar \geq e^{\frac{ \norm{\func}_\infty}{\reg}} \frac{\basereg}{\radius}$.

    Hence, for 
    \begin{align}
        \dualvar \geq \ub\dualvar \defeq \max \parens*{\frac{12 \reg}{\GeoRad^2} \log(2\times 6^{\dims/2}),e^{\frac{\sup_{\func \in \funcs} \norm{\func}_\infty}{\reg}} \frac{\basereg}{\radius}}\,,
    \end{align}
     the derivative of $\dualvar \mapsto \dualvar \radius^2 + \ex_{\sample \sim \emp}[\dualfunc(\func, \sample, \dualvar, \reg,\sdev)]$ is non-negative, which means that its minimizers are smaller than $\ub\dualvar$.
\end{proof}

\section{Proof of the main results} \label{sec:proofcl}

In this section, we present our main results with explicit constants.
In \cref{sec:clunreg} we treat the case of standard \ac{WDRO}, \ie the setting of \cref{thm:informal-unreg,thm:informal-unreg-weak}, while in \cref{sec:clreg} we handle the regularized setting of \cref{thm:informal-reg}.

\subsection{Standard \ac{WDRO} case}\label{sec:clunreg}

The main results of this section are \cref{thm:precise-unreg,thm:precise-unreg-weak} which are more precise versions of \cref{thm:informal-unreg,thm:informal-unreg-weak} respectively.

\begin{theorem}[Extended version of \cref{thm:informal-unreg}]\label{thm:precise-unreg}
    Under \blanketapp and the additional \cref{assumption:gradient-funcs,assumption:add-geometric-simple}, with $\radius_c = \radius_c(0, 0)$ defined in \cref{eq:gen-def-radiusc}
    for any $\thres\in(0,1)$ and $\nsamples\geq 1$, if
    \begin{align}\label{eq:app-thm-concentration-wdro-unreg-radius}
        &\max \parens*{\radius_\nsamples, \frac{8192}{\sqrt \nsamples \strong (\basedualvar)^2} \parens*{
        {12\dudley}
        + \left(\fbound + \ub\bdcst(\radius_c)\right) \sqrt{{1 + \log \frac{4}{\thres}}}
}}
         \leq \radius \\
         &\text{and}\qquad
         \radius \leq \half[\radius_c] - \frac{{96\dudley[\funcs, \distfunc]}
            + 4 \sqrt{\maxcost {\log {1}/{\thres}}}}{\sqrt \nsamples}\,.
    \end{align}
    where 
\begin{align}
    \cstminradius \defeq &\min\parens*{ \frac{\basedualvar}{8(\init[\dualvar] + \smooth)},
     \frac{\strong(\basedualvar)^2}{4096 \smooth}
    }\\
    \ub \lbcst &\defeq %
        \sup_{0 < \radiusalt \leq \cstminradius} \lbcst\parens*{
    \frac{1}{\radiusalt}
      {
          \min \parens*{
              \frac{\basedualvar}{32},
              \cstminradius \afterinit[\dualvar],
              \frac{3 \radius_c^2 \cstminradius}{8 \tmplipcst}
          }
      },
         \frac{\basedualvar}{2 \radiusalt}, 0, 0}
    \\
    \radius_\nsamples \defeq& 
    \frac{117
        \parens*{
             \dudley[\funcs, \norm{\cdot}_\infty]
             + \max \parens*{\bdcstalt \parens*{\frac{\basedualvar}{32 \radius_c}}, \ub \lbcst}
            \parens*{
            1
    + \sqrt{ \log \frac{1}{\thres}}}}
    }{\sqrt \nsamples  \min \parens*{
              \frac{\basedualvar}{32},
              \cstminradius \afterinit[\dualvar],
              \frac{3 \radius_c^2 \cstminradius}{8 \tmplipcst}
          }
    }\,,
\end{align}
  then,  with probability $1 - \thres$, 
    \begin{align}\label{eq:app-thm-concentration-wdro-unreg-rob}
        \forall \obj\in\funcs, \quad    \emprisk[\radius^2] \geq \ex_{\sample \sim \probalt} \left[ \obj(\sample) \right] \qquad
        \text{for all }  \probalt \text{such that $\wass[2][2]{\prob, \probalt} \leq \radius(\radius -  \radius_\nsamples)$}\,.
    \end{align}
    In particular, with probability $1-\thres$, we have 
    \begin{align}\label{eq:app-thm-concentration-wdro-unreg-gen}
    \forall \obj\in\funcs, \quad    \emprisk[\radius^2] \geq \ex_{\sample \sim \prob} \left[ \obj(\sample) \right] .
    \end{align}
\end{theorem}

The proof of \cref{thm:precise-unreg} relies on \cref{cor:precise-conclusion-unreg} that combines the results of the previous sections, namely propositions
\ref{prop:template_general},\ref{prop:lb_dualvar_asymptotic} and, \ref{prop:lb_dualvar_unreg_crit}.

\begin{lemma}\label{cor:precise-conclusion-unreg}
    Under the blanket assumptions \blanketapp and with the additional \cref{assumption:add-geometric}, for any threshold $\thres \in (0, 1)$, define
    \begin{align}\label{eq:def-lbdualvar-precise-unreg}
        \lbdualvar(\radius)
        &=
        \begin{cases}
 \frac{\basedualvar}{32 \radius}
        \text{ if }
        \radius \leq \cstminradius = \min \parens*{\frac{\basedualvar}{32(\init[\dualvar] + \smooth)},
        \frac{\strong(\basedualvar)^2}{4096 \smooth}}\\
   \min \parens*{{\afterinit[\dualvar]},
        \frac{
           \regcrit[0,0] -  \radius^2
    }{2 \tmplipcst}} \text{ otherwise}\\
        \end{cases}\\
        \ubdualvar(\radius)
        &=
        \frac{\basedualvar}{2 \radius} \,.
    \end{align}
    Assume that
    \begin{equation}
                \radius \geq \frac{8192}{\sqrt \nsamples \strong (\basedualvar)^2} \parens*{
        {12\dudley}
        + \left(\fbound + \ub\bdcst(\radius)\right) \sqrt{{1 + \log \frac{4}{\thres}}}
        }\,,
        \label{eq:generic-concentration-cor-precise-unreg}
    \end{equation}
    and that
    \begin{align}
           \radius^2 &\leq \regcrit[0,0]  - \frac{2 \tmpbdcst(\thres)}{\sqrt \nsamples}\,.
    \end{align}   

    Then, with probability at least $1 - \half[\thres]$,
    \begin{align}
        \forall \func \in \funcs,\quad
        \emprisk[\radius^2]= \inf_{\lbdualvar(\radius) \leq \dualvar \leq \ubdualvar(\radius)} \dualvar \radius^2 + \ex_{\sample\sim\empirical} \left[ \dualfunc(\func, \sample, \dualvar, 0) \right] 
    \end{align}
    and when $\radius^2 \geq \minradius(\thres, \lbdualvar(\radius), \ubdualvar(\radius), 0)$,
    with probability $1 - \thres$, it holds,
 \begin{align}
     \emprisk[\radius^2] &\geq                                                                           \risk[\radius^2- \minradius(\thres, \lbdualvar(\radius), \ubdualvar(\radius),0)],.
    \end{align}
    Furthermore, with probability $ 1 - \thres$,
    \begin{equation}
        \forall \func \in \funcs,\quad
      \emprisk[\radius^2] \geq
      \sup \setdef*{ \ex_{\probalt}[\func]}{\probalt \in \probs\,, \wass*[2][2]{\prob, \probalt} \leq {\radius^2 - \minradius(\thres, \lbdualvar(\radius), \ubdualvar(\radius), 0)}}\,.
    \end{equation}
\end{lemma}
\begin{proof}
    This result is a consequence of \cref{prop:lb_dualvar_asymptotic,prop:lb_dualvar_unreg_crit} both applied with $\thres \gets \thres / 4$ and of \cref{prop:template_general}. Note that the upper-bound on the dual variable given by \cref{prop:lb_dualvar_asymptotic} holds for any $\radius$ since the optimal dual variable is non-increasing as a function of $\radius$.
\end{proof}

\begin{proof}[Proof of \cref{thm:precise-unreg}]
    The proof consists in simplifying both the assumptions and the result of \cref{cor:precise-conclusion-unreg}.

  We begin by showing that $\lbdualvar(\radius)$ can always be lower-bounded by a quantity proportional to $1/\radius$.
  Indeed,  by definition of $\lbdualvar(\radius)$, \cref{eq:def-lbdualvar-precise-unreg} in \cref{cor:precise-conclusion-unreg}, and using that $\radius$ is in particular less than $\half[\radius_c]$, it holds that, 
  \begin{equation}\label{eq:proof-precise-unreg-lbdualvar}
      \lbdualvar(\radius) \geq
      \frac{1}{\radius}
      {
          \min \parens*{
              \frac{\basedualvar}{32},
              \cstminradius \afterinit[\dualvar],
              \frac{3 \radius_c^2 \cstminradius}{8 \tmplipcst}
          }
      }
  \end{equation}
  Let us now turn our attention to the condition $\radius^2 \geq \minradius(\thres, \lbdualvar(\radius), \ubdualvar(\radius), 0, 0)$, whose \ac{RHS} was defined by \cref{eq:def-minradius} in \cref{sec:concentration}. We have that, by definition (\cref{assumption:funcs}), $\sup_{0 < \radius \leq \radius_c } \bdcstalt(\basedualvar/(32 \radius)) = \bdcstalt(\basedualvar/(32 \radius_c)) < +\infty$ and,
  \begin{align}
      \sup_{0 < \radius \leq \radius_c}\lbcst \parens*{\lbdualvar_\nsamples(\radius), \ubdualvar(\radius), 0, 0}
      &\leq 
\sup_{0 < \radiusalt \leq \cstminradius} \lbcst\parens*{
    \frac{1}{\radius}
      {
          \min \parens*{
              \frac{\basedualvar}{32},
              \cstminradius \afterinit[\dualvar],
              \frac{3 \radius_c^2 \cstminradius}{8 \tmplipcst}
          }
      },
         \frac{\basedualvar}{2 \radiusalt}, 0, 0}
\\
        &= \ub \lbcst
          < +\infty\,,
  \end{align}
  by definition and non-decreasingness of $\lbcst$ in its first argument (see \cref{lemma:bound-rv}) and \cref{eq:proof-precise-unreg-lbdualvar}. Hence, the following bound holds
  \begin{align}
          &\minradius(\thres, \lbdualvar(\radius), \ubdualvar(\radius), \reg,\sdev)\\
          &\leq 
        \frac{117}{\sqrt \nsamples \lbdualvar(\radius)}
        \parens*{
             \dudley[\funcs, \norm{\cdot}_\infty]
             + \max \parens*{\bdcstalt \parens*{\frac{\basedualvar}{32 \radius_c}},\ub \lbcst }
            \parens*{
            1
    + \sqrt{ \log \frac{1}{\thres}}}}\\
          &\leq \radius_\nsamples \radius\,,
   \end{align}
   where we plugged \cref{eq:proof-precise-unreg-lbdualvar}.
      
    Finally, since $\radius$ is in particular bounded by $\radius_c$, the condition \cref{eq:generic-concentration-cor-precise-unreg} is implied by 
\begin{equation}
                  \radius \geq \frac{8192}{\sqrt \nsamples \strong (\basedualvar)^2} \parens*{
        {12\dudley}        + \left(\fbound + \ub\bdcst(\radius_c)\right) \sqrt{1 + \log \frac{4}{\thres}}
        }\,,
\end{equation}
    with $\ub\bdcst(\radius_c) < + \infty$ by definition (\cref{prop:lb_dualvar_asymptotic}).
\end{proof}

\begin{theorem}[Extended version of \cref{thm:informal-unreg-weak}]\label{thm:precise-unreg-weak}
    Under \blanketapp,     for any $\thres\in(0,1)$ and $\nsamples\geq 1$, if
    \begin{align}\label{eq:thm-concentration-wdro-unreg-weak-radius}
        &\max \parens*{\radius_\nsamples, \frac{8192}{\sqrt \nsamples \strong (\basedualvar)^2} \parens*{
        {12\dudley}
        + \left(\fbound + \ub\bdcst(\radius_c)\right) \sqrt{{1 + \log \frac{2}{\thres}}}
}}
         \leq \radius \\
         &\text{and}\qquad
         \radius \leq \min \parens*{\cstminradius, \half[\radius_c] - \frac{{96\dudley[\funcs, \distfunc]}
         + 4 \sqrt{\maxcost {\log {1}/{\thres}}}}{\sqrt \nsamples}}\,.
    \end{align}
    where 
\begin{align}
    \cstminradius \defeq &\min\parens*{ \frac{\basedualvar}{8(\init[\dualvar] + \smooth)},
     \frac{\strong(\basedualvar)^2}{4096 \smooth}
    }\\
\ub \lbcst &\defeq 
        \sup_{0 < \radiusalt \leq \cstminradius} \lbcst\parens*{\frac{\basedualvar}{32 \radiusalt}, \frac{\basedualvar}{2 \radiusalt}, 0, 0}
\\
    \radius_\nsamples \defeq& 
    \frac{3744
        \parens*{
             \dudley[\funcs, \norm{\cdot}_\infty]
             + \max \parens*{\bdcstalt \parens*{\frac{\basedualvar}{32 \cstminradius}}, \ub \lbcst}
            \parens*{
            1
    + \sqrt{ \log \frac{1}{\thres}}}}
    }{\sqrt \nsamples  \basedualvar,
    }\,,
\end{align}
  then,  with probability $1 - \thres$, 
    \begin{align}\label{eq:thm-concentration-wdro-unreg-weak-rob}
        \forall \obj\in\funcs, \quad    \emprisk[\radius^2] \geq \ex_{\sample \sim \probalt} \left[ \obj(\sample) \right] \qquad
        \text{for all }  \probalt \text{such that $\wass[2][2]{\prob, \probalt} \leq \radius(\radius -  \radius_\nsamples)$}\,.
    \end{align}
    In particular, with probability $1-\thres$, we have 
    \begin{align}\label{eq:thm-concentration-wdro-unreg-weak-gen}
    \forall \obj\in\funcs, \quad    \emprisk[\radius^2] \geq \ex_{\sample \sim \prob} \left[ \obj(\sample) \right] .
    \end{align}
\end{theorem}
The proof of \cref{thm:precise-unreg-weak} leverages results from the previous sections, combined in \cref{cor:precise-conclusion-unreg-weak}. 
\begin{lemma}\label{cor:precise-conclusion-unreg-weak}
    Under the blanket assumptions \blanketapp,  for any threshold $\thres \in (0, 1)$, define
    \begin{align}
        \lbdualvar(\radius)
        =
        \frac{\basedualvar}{32 \radius}\,,
        \qquad
        \ubdualvar(\radius)
        =
        \frac{\basedualvar}{2 \radius} \,.
    \end{align}
    Assume that
    \begin{equation}
                \radius \geq \frac{8192}{\sqrt \nsamples \strong (\basedualvar)^2} \parens*{
        {12\dudley}
        + \left(\fbound + \ub\bdcst(\radius)\right) \sqrt{{1 + \log \frac{2}{\thres}}}
        }\,,
        \label{eq:generic-concentration-cor-precise-unreg-weak}
    \end{equation}
    and that
    \begin{align}
        \radius^2 &\leq \min \parens*{\regcrit[0,0]  - \frac{2 \tmpbdcst(\thres)}{\sqrt \nsamples}
            ,
        \min \parens*{\frac{\basedualvar}{32(\init[\dualvar] + \smooth)}, \frac{\strong(\basedualvar)^2}{4096 \smooth}}^2}\,.
    \end{align}   

    Then, with probability at least $1 - \half[\thres]$,
    \begin{align}
        \forall \func \in \funcs,\quad
        \emprisk[\radius^2]= \inf_{\lbdualvar(\radius) \leq \dualvar \leq \ubdualvar(\radius)} \dualvar \radius^2 + \ex_{\sample\sim\empirical} \left[ \dualfunc(\func, \sample, \dualvar, 0) \right] 
    \end{align}
    and when $\radius^2 \geq \minradius(\thres, \lbdualvar(\radius), \ubdualvar(\radius), 0)$,
    with probability $1 - \thres$, it holds,
 \begin{align}
     \emprisk[\radius^2] &\geq                                                                           \risk[\radius^2- \minradius(\thres, \lbdualvar(\radius), \ubdualvar(\radius),0)],.
    \end{align}
    Furthermore, with probability $ 1 - \thres$,
    \begin{equation}
        \forall \func \in \funcs,\quad
      \emprisk[\radius^2] \geq
      \sup \setdef*{ \ex_{\probalt}[\func]}{\probalt \in \probs\,, \wass*[2][2]{\prob, \probalt} \leq {\radius^2 - \minradius(\thres, \lbdualvar(\radius), \ubdualvar(\radius), 0)}}\,.
    \end{equation}
\end{lemma}
\begin{proof}
    This result follows directly from \cref{prop:lb_dualvar_asymptotic} that we invoke with $\thres \gets \thres / 2$ and of \cref{prop:template_general}.
\end{proof}
\begin{proof}[Proof of \cref{thm:precise-unreg-weak}]

  The proof consists in simplifying both the assumptions and the result of \cref{cor:precise-conclusion-unreg-weak} and follows the same structure as the proof of \cref{thm:precise-unreg}.

  We begin by examining the condition $\radius^2 \geq \minradius(\thres, \lbdualvar(\radius), \ubdualvar(\radius), 0, 0)$, whose \ac{RHS} was defined by \cref{eq:def-minradius} in \cref{sec:concentration}. We have that, by definition (\cref{assumption:funcs}), $\sup_{0 < \radius \leq \radius_c } \bdcstalt(\basedualvar/(32 \radius)) = \bdcstalt(\basedualvar/(32 \radius_c)) < +\infty$ and,
  \begin{equation}
      \sup_{0 < \radius \leq \cstminradius}\lbcst \parens*{\lbdualvar_\nsamples(\radius), \ubdualvar(\radius), 0, 0}
      =
\sup_{0 < \radiusalt \leq \cstminradius} \lbcst\parens*{\frac{\basedualvar}{32 \radiusalt}, \frac{\basedualvar}{2 \radiusalt}, 0, 0}
        = \ub \lbcst
          < +\infty
  \end{equation}
  by definition (see \cref{lemma:bound-rv}). Hence, we have that
  \begin{align}
          &\minradius(\thres, \lbdualvar(\radius), \ubdualvar(\radius), 0, 0)\\
          &\leq 
        \frac{117}{\sqrt \nsamples \lbdualvar(\radius)}
        \parens*{
             \dudley[\funcs, \norm{\cdot}_\infty]
             + \max \parens*{\bdcstalt \parens*{\frac{\basedualvar}{32 \radius_c}},\ub \lbcst }
            \parens*{
            1
    + \sqrt{ \log \frac{1}{\thres}}}}\\
          &\leq \radius_\nsamples \radius\,,
   \end{align}
   by definition of $\radius_\nsamples$ and with $117 \times 32  = 3744$.
      
    Finally, m \cref{eq:generic-concentration-cor-precise-unreg-weak}
    is implied by
  \begin{equation}
                \radius \geq \frac{8192}{\sqrt \nsamples \strong (\basedualvar)^2} \parens*{
        {12\dudley}
        + \left(\fbound + \ub\bdcst(\cstminradius)\right) \sqrt{{1 + \log \frac{2}{\thres}}}
        }\,,
    \end{equation}
    since $\radius \leq \cstminradius$ and $\ub \bdcst(\cstminradius) < + \infty$ by definition (\cref{prop:lb_dualvar_asymptotic}).
\end{proof}

\subsection{Regularized \ac{WDRO} case} \label{sec:clreg}

\begin{theorem}[Extended version of \cref{thm:informal-reg}]\label{thm:precise-reg}
    For $\sdev = \base[\sdev] \radius$ with $\base[\sdev] > 0$, $\reg = \base[\reg] \radius$ with $\base[\reg] > 0$ such that $\base[\reg] / \base[\sdev]^2 \leq \basedualvar / 8$, and for any $\thres\in(0,1)$ and $\nsamples\geq 1$, define,
   \begin{equation}
        \radius_c \defeq \inf \setdef*{\radius_c\left(\basereg \alt \radius,\basesdev \alt \radius\right)}{\cstminradius \leq \alt \radius \leq  \radius_c\left(\basereg \cstminradius, \basesdev \cstminradius\right)}
    \end{equation}
    and
    \begin{align}
        \cstminradius &\defeq \bdRadiusAsympt\\
        \ub \Var &\defeq \sup_{\cstminradius \leq \radiusalt \leq \radius_c} \Var(\basereg\radiusalt,\basesdev\radiusalt)\\
        \ub \Cost_\funcs &\defeq \sup_{\cstminradius \leq \radiusalt \leq \radius_c} \Cost_\funcs(\basereg \radiusalt, \basesdev \radiusalt)\\
        \ub \lbcst &\defeq \sup_{0 < \radiusalt \leq \radius_c}  \lbcst \parens*{ 
              \frac{1}{\radiusalt}
              \min \parens*{
                  \frac{\basedualvar}{32},
                  \frac{
                    \basereg \cstminradius^2 \radius_c^2
                  }
                  {
                    4 \ub \Var 
                  }
              }, 
              \max \parens*{
                  \frac{\basedualvar}{2 \radiusalt},
                  \frac{12 \basereg \radius_c \log(2 \times 6^{\dims/2})}{\GeoRad^2},
                  \frac{e^{\frac{\sup_{\func \in \funcs} \norm{\func}_\infty}{\basereg \cstminradius}} \basereg}{\cstminradius}
              }
          ,\basereg \radiusalt, \basesdev \radiusalt} \\
          \radius_\nsamples &\defeq \frac{117        \parens*{
             \dudley[\funcs, \norm{\cdot}_\infty]
             + \max \parens*{\bdcstalt \parens*{\frac{\basedualvar}{32 \radius_c}},\ub \lbcst }
            \parens*{
            1
+ \sqrt{ \log \frac{1}{\thres}}}}}
    {\sqrt \nsamples \min \parens*{
                  \frac{\basedualvar}{32},
                  \frac{
                    \basereg \cstminradius^2 \radius_c^2
                  }
                  {
                    4 \ub \Var 
                  }
          }}\,;
    \end{align}
    when
    \begin{align}
        &\max \parens*{
            \radius_\nsamples,
            \frac{8192}{\strong (\basedualvar)^2 \sqrt \nsamples} \parens*{
        {12\dudley}
        + \left(\fbound + \ub\bdcst(\radius_c)\right) \sqrt{{\log \frac{4}{\thres}}}
        },
    \frac{384\sqrt{\ub \Var} \dudley}{\basereg \radius_c^2 \sqrt{\nsamples}}
    }
    \leq \radius \\
    &
    \radius \leq  \half[\radius_c] - \frac{384\sqrt{\ub \Var} \dudley}{\basereg \radius_c^2 \sqrt{\nsamples}}\\
    \text{and}\qquad
    &\radius_c \geq \max \parens*{\parens*{\frac{192\sqrt{\ub \Var} \dudley}{\basereg \sqrt{\nsamples}}}^{1/3}, 
        2 \sqrt{\ub \Cost_\funcs}\parens*{\frac{\log \frac{4}{\thres}}{2\nsamples}}^{1/4}
    }\,,
    \end{align} 
    then, 
    with probability at least $1 - \thres$,
    \begin{align}
        \forall \obj\in\funcs, \quad    \emprisk[\radius^2][\reg] \geq \ex_{\sample \sim \probalt} \left[ \obj(\sample) \right] 
        \qquad \text{for all } \probalt \text{ such that }\wass[2,\regalt(\radius)][2]{\prob, \probalt} \leq \radius(\radius -  \radius_\nsamples)\,,
    \end{align}
    where $\regalt(\radius) \leq \frac{\reg \radius}{
              \min \parens*{
                  \frac{\basedualvar}{32},
                  \frac{
                    \basereg \cstminradius^2 \radius_c^2
                  }
                  {
                    4 \ub \Var 
                  }
              }}$.
    Furthermore, when $\basesdev \leq 1$ and $\sdev \leq \init[\sdev]$ (defined in \cref{lemma:bound_dualfunc}), with probability $1 - \thres$,
    \begin{align}
    \forall \obj\in\funcs, \quad    \emprisk[\radius^2][\reg] \geq \ex_{\sample \sim \prob} \ex_{\samplealt \sim \basecpl(\cdot | \sample)} \left[ \obj(\samplealt) \right] .
    \end{align}
\end{theorem}
The proof of \cref{thm:precise-reg} relies on \cref{cor:precise-conclusion-reg} that makes the regularized Wasserstein distance appear. It also uses \cref{lemma:reg-nbd}, to guarantee that a smoothed version of the true distribution is inside the right neighborhood.
\begin{lemma}\label{cor:precise-conclusion-reg}
    Fix a confidence threshold $\thres \in (0, 1)$, take $\reg = \basereg \radius$, $\sdev = \basesdev \radius$ with $\basereg$ and $\basesdev$ positive constants satisfying $\basereg / \basesdev^2 \leq \basedualvar / 8$ and, define  $\lbdualvar_\nsamples(\radius)$ and $\ubdualvar(\radius)$ as functions of $\radius$ by
    \begin{itemize}
        \item If
            \begin{equation}
                \assumptionRadiusAsympt\,,
            \end{equation}
            then $\lbdualvar_\nsamples(\radius) = \frac{\basedualvar}{32 \radius}$ and $\ubdualvar(\radius) = \frac{\basedualvar}{2 \radius}$,
        \item Otherwise,
            \begin{align}
                \lbdualvar_\nsamples(\radius) &=  \frac{\basereg \radius}{\Var(\basereg\radius,\basesdev\radius)}
        \parens*{
            \radius_c(\basereg\radius,\basesdev\radius)^2 - \radius^2
        - \parens*{ \frac{48\sqrt{\Var(\basereg\radius,\basesdev\radius)} \dudley}{\basereg \radius \sqrt{\nsamples}}
        + 2\Cost_\funcs(\basereg\radius,\basesdev\radius) \sqrt{\frac{\log \frac{4}{\thres}}{2\nsamples}}}
    }\\
    \ubdualvar(\radius) &=  \max \parens*{\frac{12 \basereg\radius}{\GeoRad^2} \log(2\times 6^{\dims/2}),e^{\frac{\sup_{\func \in \funcs}\norm{\func}_\infty}{\basereg\radius}} \frac{\basereg}{\radius}}\,. 
            \end{align} 
    \end{itemize}
    Assume that
    \begin{equation}
                \radius \geq \frac{8192}{\sqrt \nsamples \strong (\basedualvar)^2} \parens*{
        {12\dudley}
        + \left(\fbound + \ub\bdcst(\radius)\right) \sqrt{{1 + \log \frac{4}{\thres}}}
        }\,,
        \label{eq:generic-concentration-cor-precise-reg}
    \end{equation}
    Then, with probability at least $1 - \half[\thres]$,
    \begin{align}
        \forall \func \in \funcs,\quad
        \emprisk[\radius^2][\reg] = \inf_{\lbdualvar_\nsamples(\radius) \leq \dualvar \leq \ubdualvar(\radius)} \dualvar \radius^2 + \ex_{\sample\sim\empirical} \left[ \dualfunc(\func, \sample, \dualvar, \reg,\sdev) \right] 
    \end{align}
    and when $\radius^2 \geq \minradius(\thres, \lbdualvar_\nsamples(\radius), \ubdualvar(\radius), \reg,\sdev)$,
    with probability $1 - \thres$, it holds,
 \begin{align}
        \emprisk[\radius^2][\reg] &\geq                                                                           \risk[\radius^2- \minradius(\thres, \lbdualvar_\nsamples(\radius), \ubdualvar(\radius), \reg,\sdev)][\reg],.
    \end{align}
    Furthermore, with probability $ 1 - \thres$,
    \begin{equation}
        \forall \func \in \funcs,\quad
      \emprisk[\radius^2][\reg]
      \geq
      \sup \setdef*{ \ex_{\probalt}[\func]}{\probalt \in \probs\,, \wass*[2,\regalt(\radius)][2]{\prob, \probalt} \leq {\radius^2 - \minradius(\thres, \lbdualvar_\nsamples(\radius), \ubdualvar(\radius), \reg,\sdev)}}\,,
    \end{equation}
    with $\regalt(\radius) \defeq \frac{\basereg \radius}{\lbdualvar_\nsamples(\radius)}$.
\end{lemma}
\begin{proof}
    The first part of this result is a consequence of the combination of \cref{prop:lb_dualvar_asymptotic,prop:lb_dualvar_reg_crit}, both applied with $\thres \gets \thres / 4$, and of \cref{prop:template_general}.
    For the second part, note that \cref{rmk:refined-template_general} implies that the above argument actually gives the slightly stronger result:
    with probability $1 - \thres$, for any $\func \in \funcs$,
  \begin{align}
        \emprisk[\radius^2][\reg] \geq 
        \inf_{\lbdualvar_\nsamples(\radius) \leq \dualvar \leq \ubdualvar(\radius)} \dualvar(\radius^2 - \minradius(\thres,\lbdualvar_\nsamples(\radius),\ubdualvar(\radius),\reg,\sdev)) +\ex_{\sample \sim \prob}[\dualfunc(\func, \sample, \dualvar, \reg,\sdev)]
    \end{align}
    Next, take $\probalt \in \probs$ such that $\wass[2,\regalt(\radius)][2]{\prob, \probalt} \leq {\radius^2 - \minradius(\thres, \lbdualvar_\nsamples(\radius), \reg,\sdev)}$. With a similar argument as in the proof of \cref{prop:template_general}, we get that 
 \begin{align}
        \emprisk[\radius^2][\reg] &\geq
                                       \inf_{\lbdualvar_\nsamples(\radius) \leq \dualvar \leq \ubdualvar(\radius)} \dualvar(\radius^2 - \minradius(\thres,\lbdualvar_\nsamples(\radius),\ubdualvar(\radius),\reg,\sdev) +\ex_{\sample \sim \prob}[\dualfunc(\func, \sample, \dualvar, \reg,\sdev)]\\
                                       &= \ex_{\probalt}[\func] + 
                                       \inf_{\lbdualvar_\nsamples(\radius) \leq \dualvar \leq \ubdualvar(\radius)} \dualvar(\radius^2 - \minradius(\thres,\lbdualvar_\nsamples(\radius),\ubdualvar(\radius),\reg,\sdev) -  \braces*{\ex_{\probalt}[\obj] - \ex_{\sample \sim \prob}[\dualfunc(\func, \sample, \dualvar, \reg,\sdev)]}\\
   &= \ex_{\probalt}[\func] + 
   \inf_{\lbdualvar_\nsamples(\radius) \leq \dualvar \leq \ubdualvar(\radius)} \dualvar(\radius^2 - \minradius(\thres,\lbdualvar_\nsamples(\radius),\ubdualvar(\radius),\reg,\sdev)) -  \sup_{\alt\func \in \funcs\,}\braces*{{\ex_{\probalt}[\alt \func] - \ex_{\sample \sim \prob}[\dualfunc(\alt \func, \sample, \dualvar, \reg,\sdev)]}}\,.
    \end{align}
    We now proceed to show, and this will conclude the proof, that 
    \begin{equation}
        \sup_{\func \in \funcs}\braces*{\ex_{\probalt}[\func] - \ex_{\sample \sim \prob}[\dualfunc(\func, \sample, \dualvar, \reg,\sdev)]} \leq \dualvar \wass[2,{\regalt(\radius)}][2]{\prob, \probalt}\,,
    \end{equation}
    for $\dualvar \geq \lbdualvar_\nsamples(\radius)$.

    Indeed,
    \begin{align}
        \sup_{\func \in \funcs}\braces*{\ex_{\probalt}[\func] - \ex_{\sample \sim \prob}[\dualfunc(\func, \sample, \dualvar, \reg,\sdev)]} 
        &\leq 
        \sup_{\func \in \contfuncs[\samples]}\braces*{\ex_{\probalt}[\func] - \ex_{\sample \sim \prob}[\dualfunc(\func, \sample, \dualvar, \reg,\sdev)]}\\
        &= 
        \sup_{\func \in \contfuncs[\samples]}\braces*{\ex_{\probalt}[\func] - \ex_{\sample \sim \prob}[\log \parens*{\ex_{\samplealt \sim \basecpl(\cdot | \sample)}[e^{\frac{\func(\samplealt) - \dualvar \sqnorm{\sample -  \samplealt}/2}{\reg}}]}]}\\
        &=
       \dualvar \sup_{\func \in \contfuncs[\samples]}\braces*{\ex_{\probalt}[\func] - \ex_{\sample \sim \prob}[\log \parens*{\ex_{\samplealt \sim \basecpl(\cdot | \sample)}[e^{\frac{\func(\samplealt) - \sqnorm{\sample -  \samplealt}/2}{\reg / \dualvar }}]}]}\,.
       \label{eq:proof-conclusion-reg-bound}
    \end{align}
    where we performed the change of variable $\func \gets \func / \dualvar$.
    We now show the following equality that will allow us to rewrite the \ac{RHS} of \cref{eq:proof-conclusion-reg-bound}.
    \begin{equation}
- \ex_{\sample \sim \prob}[\log \parens*{\ex_{\samplealt \sim \basecpl(\cdot | \sample)}[e^{\frac{\func(\samplealt) - \half \sqnorm{\sample -  \samplealt}}{\reg / \dualvar }}]}]
        =
        \sup_{\funcalt \in \contfuncs[\samples]} \ex_{\prob}[\funcalt] - \frac{\reg}{\dualvar}\parens*{\ex_{(\sample, \samplealt) \sim \basecpl}[e^{\frac{\funcalt(\sample) + \func(\samplealt) - \half \sqnorm{\sample -  \samplealt}}{\reg / \dualvar}}] - 1}\,.\label{eq:proof-conclusion-reg-bound-funcpb}
    \end{equation}
    Solving the optimality condition of the concave problem of the \ac{RHS} of \cref{eq:proof-conclusion-reg-bound-funcpb}
     gives that its maximum is reached for 
    \begin{equation}
        \funcalt(\sample) = - \log \parens*{\ex_{\samplealt \sim \basecpl(\cdot | \sample)}[e^{\frac{\func(\samplealt) - \half \sqnorm{\sample -  \samplealt}}{\reg / \dualvar }}]}
    \end{equation}
    so that \cref{eq:proof-conclusion-reg-bound-funcpb} holds.
    Hence, we get that
    \begin{align}
       &\sup_{\func \in \contfuncs[\samples]}\braces*{\ex_{\probalt}[\func] - \ex_{\sample \sim \prob}[\log \parens*{\ex_{\samplealt \sim \basecpl(\cdot | \sample)}[e^{\frac{\func(\samplealt) - \half \sqnorm{\sample -  \samplealt}}{\reg / \dualvar }}]}]}\\
       &=
       \sup_{\func, \funcalt \in \contfuncs[\samples]}\braces*{\ex_{\probalt}[\func] + \ex_{\prob}[\funcalt] - \frac{\reg}{\dualvar}\parens*{\ex_{(\sample, \samplealt) \sim \basecpl}[e^{\frac{\funcalt(\sample) + \func(\samplealt) - \half \sqnorm{\sample -  \samplealt}}{\reg / \dualvar}}] - 1}}\\
       &= \wass*[2,\reg/\dualvar][2]{\prob, \probalt}\,,
    \end{align}
    by the duality formula for regularized \ac{OT} \citep{peyre2019computational}\footnote{To get this exact result for a regularization \wrt an arbitrary measure, one can readily combine \citet[Cor.~1]{patyRegularizedOptimalTransport2020} and \citet[Prop.~7]{feydy2019interpolating}. Also, note that we essentially reproved the semi-duality formula of \citet[Prop.~2.1]{genevayStochasticOptimizationLargescale2016} except that the regularization is taken \wrt a general measure.}.

    Combining this equality with the bound of \cref{eq:proof-conclusion-reg-bound} gives
    \begin{equation}
     \sup_{\func \in \funcs}\braces*{\ex_{\probalt}[\func] - \ex_{\sample \sim \prob}[\dualfunc(\func, \sample, \dualvar, \reg,\sdev)]} 
     \leq \dualvar \wass*[2,\reg/\dualvar][2]{\prob, \probalt},
    \end{equation}
    which yields the result since $\wass[2,\regalt][2]{\prob, \probalt}$ is non-decreasing in $\regalt$.
\end{proof}

\begin{lemma}\label{lemma:reg-nbd}
    In the setting of \cref{thm:precise-reg},
when $\probalt_\sdev$ denotes the second marginal of
\begin{equation}
    \prob(\dd \sample) \basecpl(\dd \samplealt | \sample)\,,
\end{equation}
and when $\sdev \leq \init[\sdev]$%
, it holds
\begin{equation}
    \wass*[2,\regalt(\radius)][2]{\prob, \probalt_\sdev} \leq \sdev^2\,.
\end{equation}
\end{lemma}
\begin{proof}
   Consider the transport plan $\coupling = \prob(\dd \sample) \basecpl(\dd \samplealt | \sample)$. To show this lemma, it suffices to prove that
   \begin{equation}
       \ex_{\coupling}[\cost] + \regalt(\radius) \dkl(\coupling | \coupling)  = \bigoh \parens*{\sdev^2}\,,
   \end{equation}
   \ie that 
    $\ex_{\coupling}[\cost] = \bigoh \parens*{\sdev^2}$.
Let us first fix $\sample \in \supp \prob$ and consider $\ex_{\samplealt \sim \basecpl(\cdot | \sample)}[\half \sqnorm{\sample -  \samplealt}]$, which is equal to
    \begin{equation}
    \ex_{\samplealt \sim \basecpl(\cdot | \sample)}[\half \sqnorm{\sample -  \samplealt}]
        =
        \frac
        {\int_\samples \half \sqnorm{\sample -  \samplealt} e^{-\frac{\norm{\sample - \samplealt}^2}{2 \sdev^2}} \dd \samplealt}
        {\int_\samples  e^{-\frac{\norm{\sample - \samplealt}^2}{2 \sdev^2}} \dd \samplealt}\,.
    \end{equation}
    The numerator can be upper-bounded as follows:
    \begin{align}
        {\int_\samples \half \sqnorm{\sample -  \samplealt} e^{-\frac{\norm{\sample - \samplealt}^2}{2 \sdev^2}} \dd \samplealt}
        \leq 
        {\int_{\R^\dims} \half \sqnorm{\sample -  \samplealt} e^{-\frac{\norm{\sample - \samplealt}^2}{2 \sdev^2}} \dd \samplealt}
        =  (2 \pi \sdev^2)^{\dims/2}\half[\sdev^2]\,.
    \end{align}
    For the denominator, we have seen in the proof of \cref{lemma:bd-normalization-cst}, and more precisely \cref{eq:proof-bd-normalizing-cst-precise}, that
    \begin{equation}
        \parens*{\int_\samples  e^{-\frac{\norm{\sample - \samplealt}^2}{2 \sdev^2}} \dd \samplealt}^{-1}
    \leq \frac{2}{(2 \pi \sdev^2)^{\dims/2}}\,,
    \end{equation}
    when $\sdev \leq \init[\sdev]$.
    Hence, we have the bound
    \begin{equation}
        \ex_{\samplealt \sim \basecpl(\cdot | \sample)}[\half \sqnorm{\sample -  \samplealt}]
        \leq 
        \sdev^2\,,
    \end{equation}
    and integrating \wrt $\sample \sim \prob$ yields the result.
\end{proof}
\begin{proof}[Proof of \cref{thm:precise-reg}]
    Since we will only consider radii in particular bounded by $\radius_c$, the condition \cref{eq:generic-concentration-cor-precise-reg} is implied by 
\begin{equation}
                  \radius \geq \frac{8192}{\strong (\basedualvar)^2 \sqrt \nsamples} \parens*{
        {12\dudley}
        + \left(\fbound + \ub\bdcst(\radius_c)\right) \sqrt{{\log \frac{4}{\thres}}}
        }\,,
\end{equation}
    with $\ub\bdcst(\radius_c) < + \infty$.

    We  now show that $\lbdualvar_\nsamples(\radius)$ can always be lower-bounded by a quantity proportional to $1/\radius$, \ie that 
    \begin{equation}\label{eq:proof-precise-reg-lbdualvar}
              \lbdualvar_\nsamples(\radius)
              \geq
              \frac{1}{\radius}
              \min \parens*{
                  \frac{\basedualvar}{32},
                  \frac{
                    \basereg \cstminradius^2 \radius_c^2
                  }
                  {
                    4 \ub \Var 
                  }
              }\,.
            \end{equation}

    Let us discuss separately the cases where $\radius \leq \cstminradius$ holds or not.
    \begin{itemize}
        \item When $\radius \leq \cstminradius$, \cref{eq:proof-precise-reg-lbdualvar} holds by definition of $\lbdualvar_\nsamples(\radius)$.
        \item When $\radius > \cstminradius$, by definition, $\lbdualvar_\nsamples(\radius)$ is lower bounded as 
            \begin{equation}
\lbdualvar_\nsamples(\radius) \geq
                 \frac{\basereg \radius}{\ub \Var}
        \parens*{
            \radius_c^2 
        - \radius^2
        - \parens*{ \frac{48\sqrt{\ub \Var} \dudley}{\basereg \radius \sqrt{\nsamples}}
        + 2 \ub \Cost_\funcs\sqrt{\frac{\log \frac{4}{\thres}}{2\nsamples}}}
    }\,.
            \end{equation}
            Applying \cref{lemma:choice-radius-reg} with $\ub \radius \gets \frac{\radius_c}{2}$ and $\const \gets \frac{48\sqrt{\ub \Var} \dudley}{\basereg \sqrt{\nsamples}}$, we obtain that, when
            \begin{equation}
                \radius_c \geq \parens*{\frac{192\sqrt{\ub \Var} \dudley}{\basereg \sqrt{\nsamples}}}^{1/3}
            \end{equation}
            and
            \begin{equation}
                \frac{384\sqrt{\ub \Var} \dudley}{\basereg \radius_c^2 \sqrt{\nsamples}}\leq \radius \leq \half[ \radius_c] - \frac{384\sqrt{\ub \Var} \dudley}{\basereg \radius_c^2 \sqrt{\nsamples}}\,,
            \end{equation} the following lower-bound holds,
            \begin{align}
                \lbdualvar_\nsamples(\radius) &\geq
                 \frac{\basereg \radius}{\ub \Var}
        \parens*{
            \frac{3\radius_c^2}{4}
        - 2 \ub \Cost_\funcs\sqrt{\frac{\log \frac{4}{\thres}}{2\nsamples}}
    }
    \geq
        \frac{\basereg \radius \radius_c^2}{4\ub \Var}
    \geq
    \frac{\basereg \cstminradius^2 \radius_c^2}{4\ub \Var \radius}\,,
            \end{align}
    where we used successively that $\frac{\radius_c^2}{2}
    \geq 2 \ub \Cost_\funcs\sqrt{\frac{\log \frac{4}{\thres}}{2\nsamples}}$ and $\radius \geq \cstminradius$. This concludes the proof of \cref{eq:proof-precise-reg-lbdualvar}. Note that it implies the bound on $\regalt(\radius)$ in the statement.
    \end{itemize}
    Let us finally turn our attention to the condition $\radius^2 \geq \minradius(\thres, \lbdualvar_\nsamples(\radius), \ubdualvar(\radius), \reg,\sdev)$. Since $\sup_{0 < \radius \leq \radius_c } \bdcstalt(\basedualvar/(32 \radius)) = \bdcstalt(\basedualvar/(32 \radius_c)) < +\infty$ by definition (\cref{assumption:funcs}) and
    \begin{align}
        &\sup_{0 < \radiusalt \leq \radius_c}  \lbcst \parens*{ \lbdualvar_\nsamples(\radiusalt), \ubdualvar(\radiusalt),\basereg \radiusalt, \basereg \sdev}\\
        &\leq 
\sup_{0 < \radiusalt \leq \radius_c}  \lbcst \parens*{ 
              \frac{1}{\radiusalt}
              \min \parens*{
                  \frac{\basedualvar}{32},
                  \frac{
                    \basereg \cstminradius^2 \radius_c^2
                  }
                  {
                    4 \ub \Var 
                  }
              }, 
              \max \parens*{
                  \frac{\basedualvar}{2 \radiusalt},
                  \frac{12 \basereg \radius_c \log(2 \times 6^{\dims/2})}{\GeoRad^2},
              \frac{e^{\frac{\sup_{\func \in \funcs} \norm{\func}_\infty}{\basereg \cstminradius}} \basereg}{\cstminradius}
              }
          ,\basereg \radiusalt, \basesdev \radiusalt} \\
          &= \ub \lbcst < + \infty\\
    \end{align}
    where we used the monotonicity properties of $\lbcst$ (\cref{lemma:bound-rv}) and \cref{eq:proof-precise-reg-lbdualvar}.

    In conclusion, along with \cref{eq:proof-precise-reg-lbdualvar}, we obtain that
  \begin{align}
          &\minradius(\thres, \lbdualvar(\radius), \ubdualvar(\radius), \reg, \sdev)\\
          &\leq 
        \frac{117}{\sqrt \nsamples \lbdualvar(\radius)}
        \parens*{
             \dudley[\funcs, \norm{\cdot}_\infty]
             + \max \parens*{\bdcstalt \parens*{\frac{\basedualvar}{32 \radius_c}},\ub \lbcst }
            \parens*{
            1
    + \sqrt{ \log \frac{1}{\thres}}}}\\
          &\leq \radius_\nsamples \radius\,,
   \end{align}
   by definition of $\radius_\nsamples$.
   The last part of the statement then follows from \cref{lemma:reg-nbd}.
\end{proof}

\section{Upper-bound on the empirical robust risk}
\label{sec:upperbounds}

In this section we prove \cref{thm:sandwich} that complements the main results by providing both a lwoer and an upper bound on the empirical robus risk. In view of the previous section, the missing part is ther upper-bound, that we establish in this section.

The proof of the upper-bound is similar to the proof of our main results, yet simpler. Indeed, the bounds on the dual variable are required for the true distribution $\prob$, which is fixed, instead of the empirical distribution $\empirical$.
We slightly modify our main concentration result (\cref{prop:template_general}) in \cref{prop:template_general_alt}. We simplify our bounds on the dual multiplier when the radius is close to the critical radius (\cref{prop:lb_dualvar_unreg_crit,prop:lb_dualvar_reg_crit}) in  \cref{prop:lb_dualvar_unreg_crit_alt,prop:lb_dualvar_reg_crit_alt}.

\subsection{From empirical to true risk}

\begin{proposition}\label{prop:template_general_alt}
    For $\radius > 0$, $\reg \geq 0$, $\sdev > 0$ and $\thres \in (0, 1)$, assume that there is some $0<\lbdualvar\leq\ubdualvar < +\infty$ such that, 
    \begin{align}
        \forall \func \in \funcs,\quad
        \risk[\radius^2][\reg] &= \inf_{\lbdualvar \leq \dualvar \leq \ubdualvar} \dualvar \radius^2 + \ex_{\sample\sim\prob} \left[ \dualfunc(\func, \sample, \dualvar, \reg,\sdev) \right] \,.
    \end{align}
    Then, when $\radius^2 \geq \minradius(\thres, \lbdualvar, \ubdualvar, \reg,\sdev)$,  with probability $1 - \thres$, %
 \begin{align}
    \forall \func \in \funcs,\quad \emprisk[\radius^2 - \minradius(\thres, \lbdualvar, \ubdualvar, \reg,\sdev)][\reg] \leq \risk[\radius^2][\reg] \,.
    \end{align}
\end{proposition}
\begin{proof}
    This proof closely mimics the one of \cref{prop:template_general} but switches the roles of $\prob$ and $\empirical$.
    First, note that by following the proof of \cref{lemma:template_general_lemma} with and replacing $\prob$ by $\empirical$ and \emph{vice versa} (and using statement $(a)$ of \cref{lemma:concentration_basic} instead of $(b)$) yields the following
            \begin{align}
                \sup_{(\func, \dualvar) \in \funcs \times [\lbdualvar, \ubdualvar]} \braces*{\frac{\ex_{\sample \sim \empirical}[\dualfunc(\func, \sample, \dualvar, \reg,\sdev)] - \ex_{\sample \sim \prob}[\dualfunc(\func, \sample, \dualvar, \reg,\sdev)]}{\dualvar}}
            \leq 
            \minradius(\thres, \lbdualvar, \ubdualvar, \reg,\sdev)\,.
        \end{align}
    We can now follow the last part of the proof of \cref{prop:template_general}. On the event above, for any $\func \in \funcs$,
        \begin{align}
            \risk[\radius^2][\reg] &= \inf_{\lbdualvar \leq \dualvar \leq \ubdualvar} \left\{ \dualvar \radius^2 + \ex_{\sample\sim\prob} \left[ \dualfunc(\func, \sample, \dualvar, \reg,\sdev) \right] \right\}\\
                                            &= \inf_{\lbdualvar \leq \dualvar \leq \ubdualvar}\left\{ \dualvar \radius^2 + \ex_{\sample\sim\emp} \left[ \dualfunc(\func, \sample, \dualvar, \reg,\sdev) \right] - \dualvar \frac{\ex_{\sample \sim \emp}[\dualfunc(\func, \sample, \dualvar, \reg,\sdev)] - \ex_{\sample \sim \prob}[\dualfunc(\func, \sample, \dualvar, \reg,\sdev)]}{\dualvar} \right\}\\
                                            &\geq \inf_{\lbdualvar \leq \dualvar \leq \ubdualvar}\left\{ \dualvar \radius^2 + \ex_{\sample\sim\emp} \left[ \dualfunc(\func, \sample, \dualvar, \reg,\sdev) \right] - \dualvar \sup_{\lbdualvar \leq \alt\dualvar \leq \ubdualvar}\frac{\ex_{\sample \sim \emp}[\dualfunc(\func, \sample, \alt\dualvar, \reg,\sdev)] - \ex_{\sample \sim \prob}[\dualfunc(\func, \sample, \alt \dualvar, \reg,\sdev)]}{\alt\dualvar} \right\}\\
                                            &\geq \inf_{\lbdualvar \leq \dualvar \leq \ubdualvar}\left\{ \dualvar \radius^2 + \ex_{\sample\sim\empirical} \left[ \dualfunc(\func, \sample, \dualvar, \reg,\sdev) \right] - \dualvar \minradius(\thres, \lbdualvar, \ubdualvar, \reg,\sdev)\right\} \\
                                            &\geq \emprisk[\radius^2- \minradius(\thres, \lbdualvar, \ubdualvar, \reg,\sdev)][\reg]\,.
        \end{align}

\end{proof}

\subsection{Standard \ac{WDRO} case}

\begin{proposition}\label{prop:lb_dualvar_unreg_crit_alt}
    Let \cref{assumption:add-geometric} hold and fix a threshold $\thres \in (0, 1)$. Assume that $\radius^2 \leq \regcrit[0,0]$.
    Then, we have,   
    \begin{align}
        \forall \func \in \funcs,\quad
        \risk[\radius^2] = \inf_{\lbdualvar \leq \dualvar} \dualvar \radius^2 + \ex_{\sample\sim\prob} \left[ \dualfunc(\func, \sample, \dualvar, 0) \right] 
    \end{align}
    where the dual bound $\lbdualvar$ is defined as
    \begin{align}
        \lbdualvar \defeq \min \parens*{{\afterinit[\dualvar]},
        \frac{
           \regcrit[0,0] -  \radius^2
    }{\tmplipcst}}\,,
    \end{align}
    and $\tmplipcst$ is defined in \cref{prop:lb_dualvar_unreg_crit}.
\end{proposition}

\begin{proof}
    Let $0 \leq \dualvar\leq \lbdualvar$. 
    By  \cref{lemma:lipschitz-unreg-one-sample} and the dominated convergence theorem, one has that,
        \begin{align}
            \partial_\dualvar\dualfuncalt(\dualvar) &=
            \radius^2 - \ex_{\sample \sim \prob}[\min \setdef*{\half \sqnorm{\sample -  \samplealt}}{\samplealt \in \argmax_{\samples} \func - \half[\dualvar]\sqnorm{\sample -  \cdot}}] \\
            &\leq \radius^2 - \parens*{1 - \frac{16 \dualvar}{\strongOpt}} \ex_{\sample \sim \prob}[ \min \setdef*{\half \sqnorm{\sample -  \samplealt}}{\samplealt \in \argmax_{\samples} \func} ] \\
            &\leq   \radius^2 - \regcrit[0,0] + \tmplipcst \dualvar\,,
            \end{align}
            which is non-negative by definition of $\lbdualvar$ and thus concludes the proof.
\end{proof}

We can now state analogues of \cref{thm:precise-unreg,thm:precise-unreg-weak}.
     Note that the bounds $\lbdualvar(\radius)$ that we obtained in this section are better than the ones we got in the main proof. For the sake of simplicity, we give up this additional precision and use the same bounds as in \cref{thm:precise-unreg,thm:precise-unreg-weak}.

\begin{corollary}\label{cor:precise-conclusion-unreg-alt}
    In the same setting as \cref{thm:precise-unreg},   with probability $1 - \thres$, it holds,
 \begin{align}
     \forall \func \in \funcs,\quad
     \risk[\radius^2] &\geq                                                                           \emprisk[\radius(\radius - \radius_\nsamples)]\,.
    \end{align}
\end{corollary}
\begin{proof}
    This result is obtained as a combination of \cref{prop:template_general_alt,prop:lb_dualvar_unreg_crit_alt,prop:lb_dualvar_asymptotic}, which gives the desired result with probability at least $1 - \half[\thres]$ and \emph{a fortiori} $1 - \thres$.%
\end{proof}

\begin{corollary}\label{cor:precise-conclusion-unreg-weak-alt}
    In the same setting as \cref{thm:precise-unreg-weak},  with probability $1 - \thres$, it holds,
 \begin{align}
     \forall \func \in \funcs,\quad
     \risk[\radius^2] &\geq                                                                           \emprisk[\radius(\radius - \radius_\nsamples)]\,.
    \end{align}
\end{corollary}
\begin{proof}
    This result follows by combining  \cref{prop:template_general_alt,prop:lb_dualvar_asymptotic}, which gives the desired result with probability at least $1 - \half[\thres]$ and \emph{a fortiori} $1 - \thres$. 
\end{proof}

To conclude, in the context of \cref{thm:precise-unreg} (resp.~ \cref{thm:precise-unreg-weak}),  \cref{cor:precise-conclusion-unreg-alt} (resp.~\cref{cor:precise-conclusion-unreg-weak-alt}) with $\radius \gets \radius + \radius_\nsamples$ yields, with probability at least $1 - \thres$,
\begin{equation}
     \forall \func \in \funcs,\quad
  \emprisk[\radius(\radius + \radius_\nsamples)]
        \leq 
        \risk[(\radius + \radius_\nsamples)^2]\,,
\end{equation}
so that, since $\radius \geq \radius_\nsamples$,
\begin{equation}
    \forall \func \in \funcs,\quad
  \emprisk[\radius^2]
        \leq 
        \risk[\radius(\radius + 3\radius_\nsamples)]\,,
\end{equation}

Combining this bound with \cref{thm:precise-unreg} (resp.~\cref{thm:precise-unreg-weak}) completes the bound of \cref{thm:sandwich}. 

\subsection{Regularized case}
In the regularized case, the bound simplifies as well compared to \cref{prop:lb_dualvar_reg_crit}.

\begin{proposition}\label{prop:lb_dualvar_reg_crit_alt}
    Fix a threshold $\thres \in (0, 1)$. When $\radius^2 \leq \regcrit$, we have,
    \begin{align}
        \forall \func \in \funcs,\quad
        \risk[\radius^2][\reg] = \inf_{\lbdualvar_\nsamples \leq \dualvar \leq \ub\dualvar} \dualvar \radius^2 + \ex_{\sample\sim\prob} \left[ \dualfunc(\func, \sample, \dualvar, \reg,\sdev) \right]
    \end{align}
    where the dual bounds are defined by
    \begin{align}
        \lbdualvar &\defeq \frac{\reg}{\Var(\reg,\sdev)}
        \parens*{
        \regcrit  - \radius^2
        }  \\
        \text{and } ~~~  \ub\dualvar &\defeq \max \parens*{\frac{12 \reg}{\GeoRad^2} \log(2\times 6^{\dims/2}),e^{\frac{\sup_{\func \in \funcs}\norm{\func}_\infty}{\reg}} \frac{\basereg}{\radius}} \, ,
    \end{align}
    and $\basedualvar$, $\strong$ were defined in \cref{asm:rho_small}.
\end{proposition}

\begin{proof}
    The proof of the upper-bound is exactly the same as in \cref{prop:lb_dualvar_reg_crit} so we focus on the lower-bound.
    Following the same reasoning as the one to get \eqref{eq:beforeConcentration} in \cref{prop:lb_dualvar_reg_crit} but with $\prob$ instead of $\empirical$ we get that
    \begin{align}
        \partial_\dualvar\left\{ \dualvar \radius^2 + \ex_{\sample\sim\prob} \left[ \dualfunc(\func, \sample, \dualvar, \reg,\sdev) \right] \right\} 
        &\leq   \radius^2 - \ex_{\sample\sim\prob} \left[ \ex_{\samplealt \sim \modbase[\func/\reg](\cdot | \sample)}[ \half \sqnorm{\sample -  \samplealt}] \right] +  \frac{\dualvar}{\reg}\Var(\reg,\sdev) \\
        &=   \radius^2 - \regcrit +  \frac{\dualvar}{\reg}\Var(\reg,\sdev)\,,
     \end{align}
     which is non-positive when $\dualvar \leq \lbdualvar$.
     \end{proof}

\begin{corollary}\label{cor:precise-conclusion-reg-alt}
    In the same setting as \cref{thm:precise-reg},  with probability $1 - \thres$, it holds,
 \begin{align}
     \forall \func \in \funcs,\quad
     \risk[\radius^2][\reg] &\geq                                                                           \emprisk[\radius(\radius - \radius_\nsamples)][\reg]\,.
    \end{align}
\end{corollary}
\begin{proof}
    This result follows by combining  \cref{prop:template_general_alt,prop:lb_dualvar_asymptotic} and \cref{prop:lb_dualvar_reg_crit_alt}, which gives the desired result with probability at least $1 - \half[\thres]$ and \emph{a fortiori} $1 - \thres$. 
\end{proof}

To conclude, in the context of \cref{thm:precise-reg},  \cref{cor:precise-conclusion-reg-alt} with $\radius \gets \radius + \radius_\nsamples$  and $\basereg \gets \frac{\basereg \radius}{\radius + \radius_\nsamples}$, \ie $\reg \gets \frac{\basereg \radius}{\radius + \radius_\nsamples} \times (\radius + \radius_\nsamples)$, yields,\footnote{
    Though $\basereg$ now formally depends on $\radiusalt$, the same bounds still hold and do not become degenerate since $\basereg$ lies $[\basereg/2, \basereg]$ that avoids zero.
} with probability at least $1 - \thres$,
\begin{equation}
     \forall \func \in \funcs,\quad
\emprisk[\radius(\radius + \radius_\nsamples)][\basereg \radius]
        \leq 
        \risk[(\radius + \radius_\nsamples)^2][\basereg \radius]\,,
        \quad\text{and, in particular,}\quad
  \emprisk[\radius^2][\basereg \radius]
        \leq 
        \risk[\radius(\radius + 3\radius_\nsamples)][\basereg \radius]\,.
\end{equation}
Combining this bound with \cref{thm:precise-reg} completes the bound of \cref{thm:sandwich}.

\section{Technical lemmas}
\label{app:lemmas}

In this section, we recall and adapt known results, as well as establish technical facts, all useful in our developments. They are presented in self-contained lemmas and are arranged in four thematic subsections.

\subsection{Laplace approximation}
\begin{lemma}[Restriction to $\samples$]\label{lemma:approx_laplace}
    Consider $\samples \subset \R^\dims$, $\init[\reg], \init[\regalt] > 0$ and a map $\sol[\samplealt] : [0, \init[\regalt]] \to \samples$ defined by $\sol[\samplealt](\regalt) = \sample + \regalt \gvec$ with $\sample \in \samples$, $\gvec \in \R^\dims$ and assume that there is a positive radius $\GeoRad$ such that,
    \begin{enumerate}
    \item The closed ball $\ball(\sol[\samplealt](0), \GeoRad)$ is included in $\samples$.
    \item $\GeoRad$, $\init[\regalt]$ and $\norm{\gvec}$ satisfy
        $\frac{\GeoRad^2}{6} \geq \init[\regalt]^2 \norm{\gvec}^2$.
    \end{enumerate}
    Then, for $(\reg, \regalt) \in [0, \init[\reg]] \times [0, \init[\regalt]]$,
\begin{align}
       \left|(2 \pi \reg \regalt)^{-\half[\dims]}\int_{\samples} \exp \parens*{-\frac{\norm{\samplealt - \sol[\samplealt](\regalt)}_2^2}{2\reg \regalt}} \dd \samplealt - 1\right| \leq 
    6^{\dims / 2}e^{-\frac{\GeoRad^2}{12 \reg \regalt}}\,,
\end{align}
\end{lemma}
\begin{proof}
    The quantity to bound rewrites
    \begin{align}
\left|(2 \pi \reg \regalt)^{-\half[\dims]}\int_{\samples} \exp \parens*{-\frac{\norm{\samplealt - \sol[\samplealt](\regalt)}_2^2}{2\reg \regalt}} \dd \samplealt - 1\right|
    =
(2 \pi \reg \regalt)^{-\half[\dims]}\int_{\R^\dims \setminus \samples} \exp \parens*{-\frac{\norm{\samplealt - \sol[\samplealt](\regalt)}_2^2}{2\reg \regalt}} \dd \samplealt\,,
    \end{align}
    so let us bound this integral.
    Since $\ball(\sol[\samplealt](0), \GeoRad)$ is inside $\samples$, this means that, for any $\samplealt  \notin \samples$, $\norm{\samplealt - \sample}$ is at least equal to $\GeoRad$.
    Hence, for any $\samplealt \notin \samples$, one has that
    \begin{align}
        \norm{\samplealt - \sol[\samplealt](\regalt)}^2 
        &\geq \half \norm{\samplealt - \sample}^2 - \regalt^2 \norm{\gvec}^2\\
        &\geq \frac{1}{6} \norm{\samplealt - \sample}^2 + \frac{1}{6} \GeoRad^2 + \frac{1}{6} \GeoRad^2 - \regalt^2 \norm{\gvec}^2\\
        &\geq \frac{1}{6} \norm{\samplealt - \sample}^2 + \frac{1}{6} \GeoRad^2\,,
    \end{align}
    so that we get the bound
    \begin{align}
        (2 \pi \reg \regalt)^{-\half[\dims]}\int_{\R^\dims \setminus \samples} \exp \parens*{-\frac{\norm{\samplealt - \sol[\samplealt](\regalt)}_2^2}{2\reg \regalt}} \dd \samplealt
        &\leq e^{-\frac{\GeoRad^2}{12 \reg \regalt}} 
    \times (2 \pi \reg \regalt)^{-\half[\dims]}\int_{\R^\dims \setminus \samples} \exp \parens*{-\frac{\norm{\samplealt - \samplealt}_2^2}{12\reg \regalt}} \dd \samplealt \\
    &= 6^{\dims / 2}e^{-\frac{\GeoRad^2}{12 \reg \regalt}}\,.
    \end{align}
\end{proof}

\subsection{Concentration}

We rely on standard concentration tools that we encapsulate in the following lemma for convenience.

\begin{lemma}\label{lemma:concentration_basic}
    Let $(\pspace, \dist)$ be a (totally bounded) separable metric space, $\prob$ a probability distribution on a probability space $\samples$ and $\emp = \empex \dirac{\sample_\ind}$ with $\sample_1,\dots,\sample_\nsamples \sim \prob$ \ac{iid}.
    Consider a mapping $\rv : \pspace \times \samples \to \R$ and assume that,
    \begin{enumerate}
        \item For each $\point \in \pspace$, $\sample \mapsto \rv(\point, \sample)$ is measurable;
        \item There is a constant $\lipcst > 0$ such that, for each $\sample \in \samples$, $\point \mapsto \rv(\point, \sample)$ is $\lipcst$-Lipschitz;
        \item $\rv$ almost surely belongs to $[\lbcst, \ubcst]$.
    \end{enumerate} 
    Then, for any $\thres \in (0, 1)$, 
    \begin{enumerate}[label=(\alph*)]
        \item With probability at least $1 - \thres$,
    \begin{align}
        \forall \point \in \pspace,\quad \ex_{\sample \sim \emp}[\rv(\point, \sample)] - \ex_{\sample \sim \prob}[\rv(\point, \sample)] \leq  \frac{48\lipcst\dudley[\pspace, \dist]}{\sqrt{\nsamples}}
        + 2(\ubcst - \lbcst) \sqrt{\frac{\log \frac{1}{\thres}}{2\nsamples}}\,.
    \end{align}
        \item With probability at least $1 - \thres$,
    \begin{align}
        \forall \point \in \pspace,\quad \ex_{\sample \sim \prob}[\rv(\point, \sample)] - \ex_{\sample \sim \emp}[\rv(\point, \sample)] \leq  \frac{48\lipcst\dudley[\pspace, \dist]}{\sqrt{\nsamples}}
        + 2(\ubcst - \lbcst) \sqrt{\frac{\log \frac{1}{\thres}}{2\nsamples}}\,.
    \end{align}
    \end{enumerate}
\end{lemma}
\begin{proof}
    First, let us note that we can assume that $\ex_{\sample \sim \prob}[\rv(\point, \sample)] = 0$ provided that we prove the bound above with the left-hand side divided by a factor two. Indeed, %
    considering the random variables $\altrv(\point, \sample) \defeq \rv(\point, \sample) - \ex_{\samplealt \sim \prob}[\rv(\point, \samplealt)]$, we see that $\altrv$ satisfy the assumptions of the lemma, albeit with the constants $\lipcst \gets 2\lipcst$, $\lbcst \gets \lbcst - \ubcst$ and $\ubcst \gets \ubcst - \lbcst$.
    Moreover, we only prove the assertion $(a)$ since the $(b)$ follows from $(a)$ with $\rv \gets - \rv$.
    
    \noindent \underline{Step 1: Bound on the expectation.}
    First, we focus on bounding the expectation of the quantity
    \begin{align}
        \sup_{\point \in \pspace} \{\ex_{\sample \sim \emp} \rv(\point, \sample) \}\,.
    \end{align}
    By the symmetrization principle (\eg \cite[Lem.~11.4]{boucheronconcentrationinequalitiesnonasymptotic2013}), with $\rad_1,\dots,\rad_\nsamples$ i.i.d.~Rademacher random variables,
    \begin{align}
        \ex_{}[\sup_{\point \in \pspace}\left\{ \ex_{\sample \sim \emp} \rv(\point, \sample) \right\}] 
        &\le 
        2\ex_{}[\sup_{\point \in \pspace} \empex \rad_\ind \rv(\point, \sample_\ind)]\,.
    \end{align}
    Take $\point, \pointalt \in \pspace$. By the Lipschitz property of $\rv$, with $\sample \sim \prob$, for any $\ind = 1, \dots, \nsamples$, the random variable
    \begin{align}\label{eq:aux_symmetrization}
        \frac{\rad_\ind \parens*{\rv(\point, \sample) - \rv(\pointalt, \sample)}}{\sqrt{\nsamples} \lipcst}
    \end{align}
    is bounded, in absolute value, by $\frac{\dist(\point, \pointalt)}{\sqrt{ \nsamples} }$ and as such it is sub-Gaussian with parameter $\frac{\dist(\point, \pointalt)^2}{{ \nsamples} }$ by Hoeffding's lemma (\eg \cite[Lem.~2.2]{boucheronconcentrationinequalitiesnonasymptotic2013}).
    As a consequence, by independence, the random variable
\begin{align}
       \sum_{\ind = 1}^\nsamples \frac{\rad_\ind \parens*{\rv(\point, \sample) - \rv(\pointalt, \sample)}}{\sqrt{\nsamples} \lipcst}
\end{align}
    is sub-Gaussian with parameter $\dist(\point, \pointalt)^2$. Since, in addition, it is zero-mean, we can invoke Dudley's bound (\eg \cite[Cor.~13.2]{boucheronconcentrationinequalitiesnonasymptotic2013}) to get that,
    \begin{align}
\ex_{}[\sup_{\point \in \pspace}\inv{\sqrt \nsamples \lipcst} \sum_{\ind = 1}^\nsamples \rad_\ind \rv(\point, \sample_\ind)] \leq 
    12 \dudley[\pspace, \dist]\,,
    \end{align}
    or, in other words, by \eqref{eq:aux_symmetrization}, that
    \begin{align}\label{eq:concentration_first_step}
        \ex_{}[\sup_{\point \in \pspace} \{\ex_{\sample \sim \emp} \rv(\point, \sample) \}]
      \leq  \frac{24 \lipcst \dudley[\pspace, \dist]}{\sqrt{\nsamples} }\,.
    \end{align}
    \noindent \underline{Step 2: Concentration inequality.}
    Since the functions $\rv$ are uniformly bounded,\\
    $\sup_{ \point \in \pspace }\left\{ \ex_{\sample \sim \emp} \rv(\point, \sample)  \right\}$, seen as a function of $(\sample_1,\ldots,\sample_\nsamples)$, satisfies the bounded difference property with constant $\ubcst - \lbcst$.
Therefore, the bounded difference inequality (\eg \cite[Thm.~6.2]{boucheronconcentrationinequalitiesnonasymptotic2013}) readily yields that, with probability at least $1 - \thres$,
    \begin{align}
        \sup_{\point \in \pspace}\left\{ \ex_{\sample \sim \emp} \rv(\point, \sample) \right\}
        &\leq 
        \ex_{}{\sup_{\point \in \pspace}\left\{ \ex_{\sample \sim \emp} \rv(\point, \sample) \right\}}
        + (\ubcst - \lbcst) \sqrt{\frac{\log \frac{1}{\thres}}{2\nsamples}}\\
        &\leq
        \frac{24\lipcst\dudley[\pspace, \dist]}{\sqrt{\nsamples}}
        + (\ubcst - \lbcst)\sqrt{\frac{\log \frac{1}{\thres}}{2\nsamples}}
    \end{align}
    where we plugged in \cref{eq:concentration_first_step}, the bound on the expectation from the first step.
\end{proof}

\subsection{Dudley's integral bounds}

\begin{lemma}\label{lemma:dudley-product-space}
    Let $(\pspace_1, \dist_1)$ and $(\pspace_2, \dist_2)$ be two metric spaces, and consider $\pspace \defeq \pspace_1 \times \pspace_2$ equipped with the distance $\dist \defeq \const_1 \dist_1 + \const_2 \dist_2$ with $\const_1, \const_2 > 0$.
    Then
    \begin{align}
        \dudley[\pspace, \dist] \leq \const_{1} \dudley[\pspace_1, \dist_1] + \const_2 \dudley[\pspace_2, \dist_2]\,.
    \end{align}
\end{lemma}
\begin{proof}
    Note that, for any $t > 0$, the inequality $N(t, \pspace, \dist) \leq N(t, \pspace_1, \const_1 \dist_1) \times N(t, \pspace_2, \const_2 \dist_2)$ holds, so that, by subdadditivity of the square root,
    \begin{align}
        \dudley[\pspace, \dist] &= \int_{0}^{+\infty} \sqrt{ \log N(t, \pspace, \dist)} \dd t\\
                                &\leq 
        \int_{0}^{+\infty} \sqrt{ \log N(t, \pspace_1, \const_1 \dist_1)} \dd t
        +
        \int_{0}^{+\infty} \sqrt{ \log N(t, \pspace_2, \const_2 \dist_2)} \dd t\\
        &=
        \int_{0}^{+\infty} \sqrt{ \log N(t / \const_1, \pspace_1, \dist_1)} \dd t
        +
        \int_{0}^{+\infty} \sqrt{ \log N(t / \const_2, \pspace_2, \dist_2)} \dd t\\
        &=  \const_1 \dudley[\pspace_1, \dist_1] + \const_2 \dudley[\pspace_2, \dist_2]\,,
    \end{align}
    where we performed changes of variable in to obtain the last equality.
\end{proof}

\begin{lemma}\label{lemma:dudley-segment}
    For $\const >0$,
        \begin{align}
            \dudley[{[0, \const], \abs{\cdot}}] \leq \half[\const]\parens*{1 + 2 \log 2}\,.
        \end{align}
\end{lemma}

\begin{proof}
    Noticing that $N(t, [0, \const], \abs{\cdot}) = 1$ whenever $t \geq \const$, we get that
    \begin{align}
        \dudley[{[0, \const], \abs{\cdot}}] &= \int_{0}^{\const} \sqrt{ \log N(t,[0, \const], \abs{\cdot})} \dd t\\
                                            &\leq  \int_{0}^{\const} \parens*{1 +\log N(t,[0, \const], \abs{\cdot})} \dd t\,.
    \end{align}
    Now, a rough bound on $N(t, [0, \const], \abs{\cdot})$ is $1 + \frac{\const}{t}$ which fits our purpose and yields
    \begin{align}
    \dudley[{[0, \const], \abs{\cdot}}] &\leq  \const + \int_{0}^{\const} \log \parens*{ 1 + \frac{\const}{t}} \dd t = \const \parens*{1 + 2 \log 2}\,. \end{align}
\end{proof}

\subsection{Auxiliary results}

We conclude these sections with auxiliary technical results.

The following lemma recalls basic inequalities with the logarithm function.
\begin{lemma}\label{lemma:log_ineq}
    For $0 \leq x \leq \half$, the following inequalities hold,
        \begin{align}
            \log(1 -x) \geq -2x \qquad \text{and} \qquad
            \log(1 + x) \leq x\,.
        \end{align}
\end{lemma}

\begin{lemma}\label{lemma:non-incr-log}
    For $\alpha > 0$, the function $x \mapsto \frac{\log(\alpha + x)}{x}$ is non-increasing on $(\pospart{e^{W(1)} - \alpha}, +\infty)$.
\end{lemma}
\begin{proof}
    Denote by $\func : x \mapsto \frac{\log(\alpha + x)}{x}$ this function, defined on $(0, +\infty)$.
    Its derivative is $\func' : x \mapsto \inv{x}\parens*{\inv{x + \alpha} - \log(x + \alpha)}$.
    But the function $x \mapsto \inv{x + \alpha} - \log(x +\alpha)$ is non-increasing, goes to $-\infty$ at infinity and its only potential zero is $e^{W(1)} -\alpha$ if it is positive,\footnote{$W$ denotes the Lambert function, \ie the inverse of the map $x \mapsto x e^x$.} which yields the result.
\end{proof}

\begin{lemma}\label{lemma:softmax_ineq}
    For $\samples \subset \R^\dims$ a compact set, $\objalt \in \contfuncs[\samples]$ and, $\probalt \in \probs$,
    \begin{align}
        \log \ex_{\sample \sim \probalt}[e^{\objalt(\sample)}] \leq \frac{\ex_{\sample \sim \probalt}[\objalt(\sample) e^{\objalt(\sample)}]}{\ex_{\sample \sim \probalt} [e^{\objalt(\sample)}]}\,.
    \end{align}
\end{lemma}
\begin{proof}
    Define $\phi : t \mapsto  \log \ex_{\sample \sim \probalt}[e^{t \objalt(\sample)}]$ which is convex and differentiable, since $\objalt$ is continuous on the compact set $\samples$.
    Hence,
    \begin{align}
        0 = \phi(0) \geq \phi(1) + \phi'(1)(0 - 1)\,,
    \end{align}
    so that $\phi'(1) \geq \phi(1)$ which is the desired inequality.
\end{proof}

\begin{lemma}\label{lemma:study_model}
    For $a,b,c,r >0$ fixed, consider the function defined on $\R_+$ by
    \begin{align}
        \dualfunc(\dualvar) = a \dualvar + \frac{b}{\dualvar + r} - c \log(\dualvar + r)\,.
    \end{align}
    Then, for any $\overline{\dualvar} > 0$, $\dualfunc$ is strongly convex on $\left[0, \overline{\dualvar}\right]$ with strong convexity constant
    \begin{align}
        \strong \defeq \frac{2 b}{ (\overline{\dualvar}+r)^3} + \frac{c}{(\overline{\dualvar}+r)^2}\,.
    \end{align}
    and the unique solution to the minimization problem
    \begin{align}
        \min_{\dualvar \geq 0} \dualfunc(\dualvar) 
    \end{align}
    is given by,
    \begin{align}
        \sol[\dualvar] =
        \pospart*{\frac{c + \sqrt{c^2 + 4ab}}{2a} - r}\,.
    \end{align}
\end{lemma}

\begin{proof}
$\dualfunc$ is twice differentiable and its derivatives are, for $\dualvar \geq 0$,
\begin{align}
    \dualfunc'(\dualvar) &= a - \frac{b}{(\dualvar + r)^2} - \frac{c}{\dualvar + r}\\
    \dualfunc''(\dualvar) &= \frac{2b}{(\dualvar + r)^3} + \frac{c}{(\dualvar + r)^2}\,,
\end{align}
which shows that $\dualfunc$ is strictly convex on $\R_+$ and yields its strong convexity on compact intervals. 
Then, the first order optimality condition $ \dualfunc'(\dualvar) = 0$ gives us that  
\begin{align}\label{eq:study_model_quadr}
    a(\dualvar + r)^2 - c(\dualvar + r) -{b} = 0\,,
\end{align}
which has an unique solution satisfying $\dualvar + r \geq 0$ which is given by,
\begin{align}
    \sol[\dualvar]  = \frac{c + \sqrt{c^2 + 4ab}}{2a} - r\,.
\end{align}
If $\sol[\dualvar] \geq 0$, then this is the solution we are looking for. If is not, this means that both roots of \eqref{eq:study_model_quadr} are non-positive  and therefore $\dualfunc'(0) \geq 0$ which means that $0$ is the solution to the minimization problem.
\end{proof}

\begin{lemma}\label{lemma:choice-radius-reg}
    For $\const > 0$, $\ub \radius > 0$ such that $\ub \radius \geq (4 \const)^{1/3}$,
    the inequality  
    \begin{equation}
        \ub \radius ^2 - \radius ^2 - \frac{\const}{\radius} \geq 0
    \end{equation}
    holds in particular when
    \begin{equation}
        \frac{2\const}{\ub \radius^2}
        \leq \radius  \leq 
        \ub \radius - \frac{2\const}{\ub \radius^2}
    \end{equation}
\end{lemma}
\begin{proof}
    When $0 < \radius \leq  \ub \radius$, the inequation $\ub \radius ^2 - \radius ^2 - \frac{\const}{\radius} \geq 0$ is implied by
    \begin{equation}
        \ub \radius ^2 \radius - \ub \radius \radius ^2 - {\const} \geq 0\,.
    \end{equation}
    Solving the latter yields the interval
    \begin{equation}
        \bracks*{
            \frac{\ub \radius ^2 - \sqrt{\ub \radius \parens*{\ub \radius ^3 - 4 \const}}}{2 \ub \radius},
            \frac{\ub \radius ^2 + \sqrt{\ub \radius \parens*{\ub \radius ^3 - 4 \const}}}{2 \ub \radius}
        }
    \end{equation}
    and the inequality $1 - \scalar \leq \sqrt{1 - \scalar}$ for $\scalar \in [0, 1]$
    yields the result.
\end{proof}

\WAedit{
    \section{Numerical illustrations}\label{app:num}
We present numerical experiments supporting our theoretical results. On logistic and linear regression models, we illustrate that, provided the radius is large enough, the robust loss on the training distribution is indeed an upper-bound on the true loss. For $\func(\param, \sample)$ as defined in \cref{ex:logistic,ex:l2}, we estimate the following probability, as in \citet[\S7.2.A]{esfahani2018data},

\begin{equation} \prob^{\otimes^\nsamples}\left( \emprisk[\radius^2][\reg] [\func(\widehat{\param}_\nsamples, \cdot)] \geq \ex_{\sample \sim \prob}[\func(\widehat{\param}_\nsamples, \sample)]\right) \quad \text{where} \quad \widehat{\param}_\nsamples = \argmin_{\params} \emprisk[\radius^2][\reg][\func(\param, \cdot)]\,,\end{equation}
and $\prob^{\otimes^\nsamples}$ denotes the distribution of the training set $(\sample_i)_{1 \leq i \leq \nsamples}$ with $\sample_i \sim \prob$ \ac{iid}.

We observe on the plots that, for $\radius$ large enough, the above probability is close to 1, for both models and for both standard and regularized cases (as guaranteed by \cref{thm:informal-unreg,thm:informal-reg}).

\begin{figure}
    \centering
    \includegraphics[clip, trim=2.5cm 0cm 0cm 0cm, width=1\linewidth]{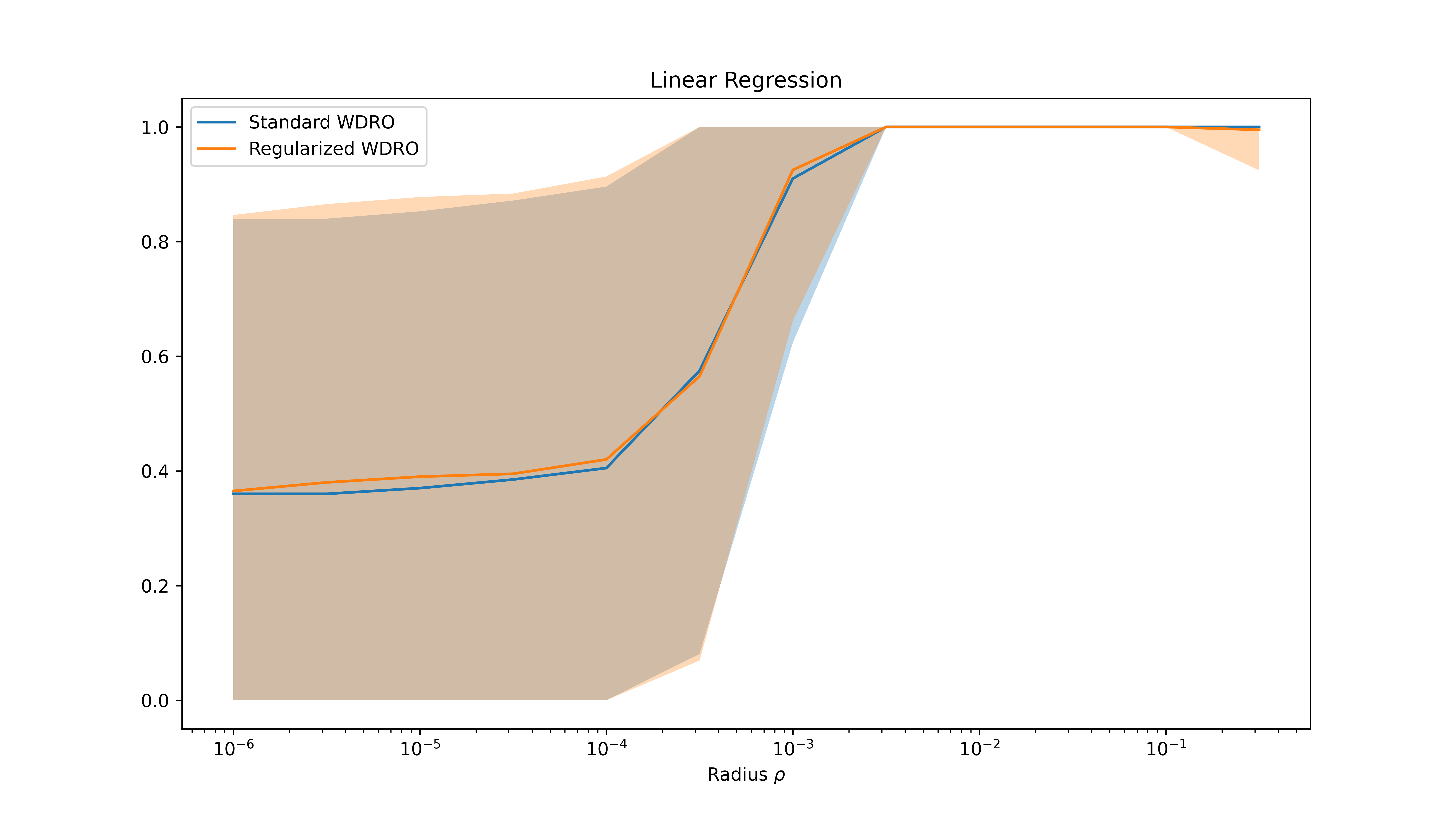}
    \caption*{\scriptsize
        Estimates of the probability $P^{\otimes^\nsamples}\left( \widehat{\mathcal R}^\epsilon_{\rho^2} (f(\hat\theta_n, \cdot)) \geq \mathbb{E}_P[f(\hat \theta_n, \xi)]\right)$ where $\hat\theta_n$ is the robust model with radius $\rho$ for the linear regression model (Example 3.7).
        $\theta$ has dimension $d = 10$, $n = 1000$ synthetic training samples are used, $\sigma$ and $\epsilon$ are chosen proportional to $\rho$ following \cref{thm:informal-reg}. For each value of $\rho$, we sample 200 training datasets and solve the WDRO problem on each of them, to obtain an estimate of the probability above. The solid line is the average over these 200 results, while the shaded area represents the standard deviation.
        As predicted by \cref{thm:informal-unreg,thm:informal-reg}, we observe that for $\rho$ large enough, the probability that the robust loss on the training set upper bounds the true risk is almost 1. We also observe that standard and regularized WDRO have almost identical generalization behaviours.
        The WDRO problems are solved by LBFGS-B combined with the formulas of Example 2 of \citet{wang2021sinkhorn}.
    }
\end{figure}
\begin{figure}
    \centering
    \includegraphics[clip, trim=2.5cm 0cm 0cm 0cm, width=1\linewidth]{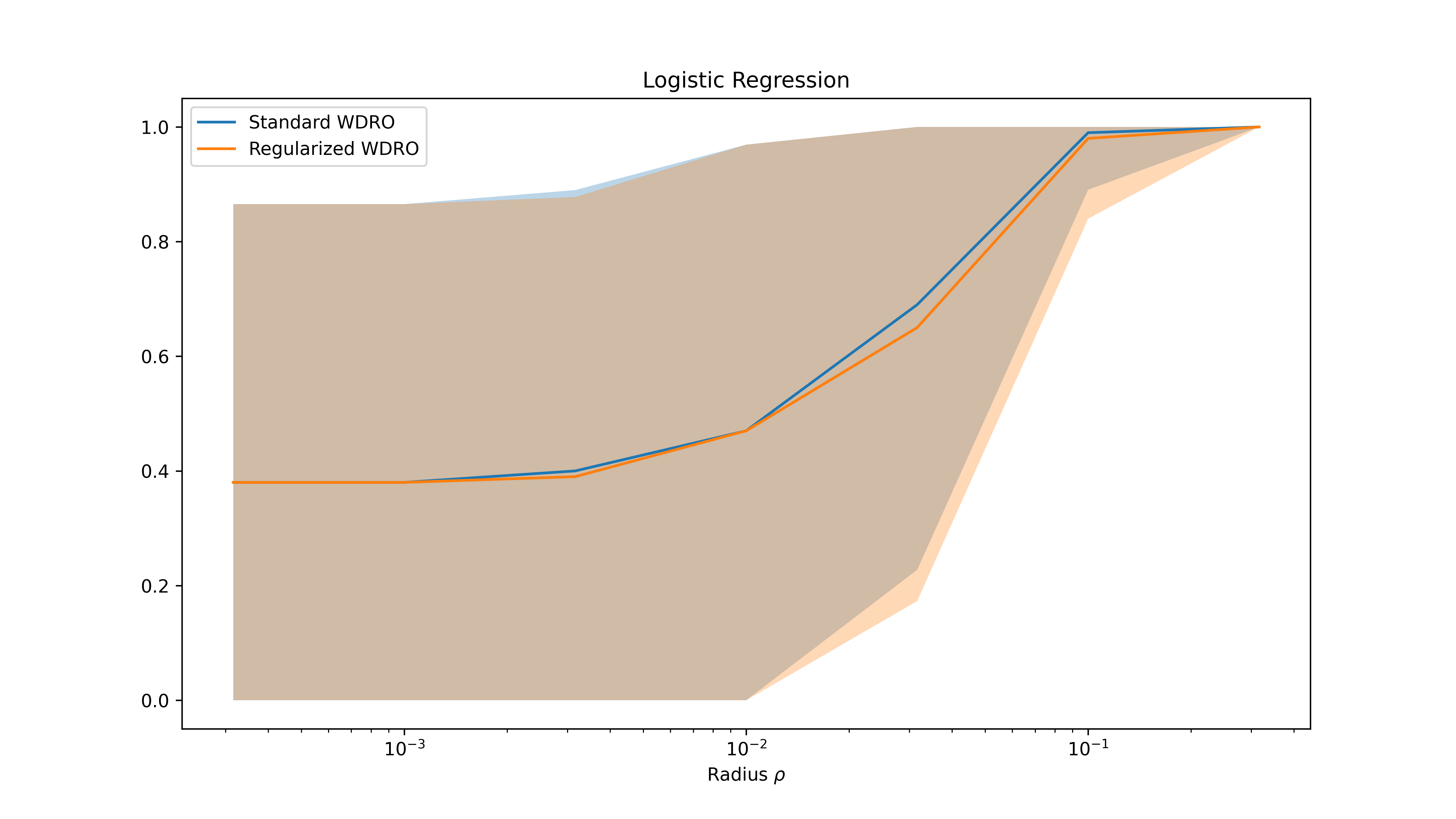}
\caption*{\scriptsize
    Estimates of the probability $P^{\otimes^\nsamples}\left( \widehat{\mathcal R}^\epsilon_{\rho^2} (f(\hat\theta_n, \cdot)) \geq \mathbb{E}_P[f(\hat \theta_n, \xi)]\right)$ where $\hat\theta_n$ is the robust model with radius $\rho$  for the logistic regression model (Example 3.6).
$\theta$ has dimension $d = 5$, $n = 500$ synthetic training samples are used, $\sigma$ and $\epsilon$ are chosen proportional to $\rho$ following \cref{thm:informal-reg}. For each value of $\rho$, we sample 100 training datasets and solve the WDRO problem on each of them, to obtain an estimate of the probability above. The solid line is the average over these 100 results, while the shaded area represents the standard deviation.
As for the previous plot and as predicted by \cref{thm:informal-unreg,thm:informal-reg}, we observe that for $\rho$ large enough, the probability that the robust loss on the training set upper bounds the true risk is almost 1. We also observe that standard and regularized WDRO have almost identical generalization behaviours. The standard WDRO problem is solved using the algorithm of \citet{blanchet2022optimal} while the regularized problem is solved using LBFGS-B combined with an explicit expression of the inner integral in the robust loss.
}

\end{figure}
}